\documentclass{article} 
\usepackage{iclr2027_conference,times}

\usepackage{amsmath,amsfonts,bm}

\def\eqref#1{equation~\ref{#1}}

\def\1{\bm{1}}

\DeclareMathAlphabet{\mathsfit}{\encodingdefault}{\sfdefault}{m}{sl}
\SetMathAlphabet{\mathsfit}{bold}{\encodingdefault}{\sfdefault}{bx}{n}

\newcommand{\E}{\mathbb{E}}

\newcommand{\R}{\mathbb{R}}

\newcommand{\Dn}{\mathscr{D}_{n}}
\newcommand{\bX}{\textbf{X}}

\renewcommand{\P}{\mathds{P}}
\newcommand{\V}{\mathbb{V}}

\newcommand{\indep}{\perp \!\!\! \perp}

\usepackage{hyperref}
\usepackage{url}

\usepackage[utf8]{inputenc} 
\usepackage[T1]{fontenc}    
\usepackage{booktabs}       
\usepackage{amsfonts}       
\usepackage{nicefrac}       
\usepackage{microtype}      
\usepackage{xcolor}         
\usepackage{dsfont}
\usepackage{empheq}
\usepackage{tabularx}
\usepackage{comment}
\usepackage{algorithm}
\usepackage{algorithmic}
\usepackage{color}
\usepackage{mathrsfs}
\usepackage{float}
\usepackage{amssymb}
\usepackage{amsfonts}
\usepackage{amsthm}
\usepackage{stmaryrd}
\usepackage{stackrel}
\usepackage{placeins}
\usepackage{tikz}
\usepackage{makecell}
\usepackage{multirow}
\usepackage{colortbl}
\usepackage{wrapfig}

\theoremstyle{plain}
\newtheorem{theorem}{Theorem}[section]

\newtheorem{lemma}[theorem]{Lemma}

\theoremstyle{definition}

\theoremstyle{remark}

\title{Hoeffding Concept Bottleneck Models with Applications to Overhead Images}

\author{%
  Clément Bénard \\
  Thales cortAIx-Labs, Palaiseau, France\\
  \texttt{clement-l.benard@thalesgroup.com} \\
  \And
  Manon Arfib \\
  Université Paris-Saclay, CentraleSupélec, \\ Gif-sur-Yvette, France \\
  \And
  Christophe Labreuche \\
  Thales cortAIx-Labs, Palaiseau, France\\
  \texttt{christophe.labreuche@thalesgroup.com} \\
  \And
  Victor Quétu \\
  Thales cortAIx-Labs, Palaiseau, France\\
  \texttt{victor.quetu@thalesgroup.com} \\
}

\iclrfinalcopy 
\begin{document}

\maketitle
\lhead{Preprint}

\begin{abstract}
  Explainability of deep learning algorithms is critical for computer-vision applications with high-stakes decisions. Concept bottleneck models (CBM) have recently shown promising performance to provide explainable and accurate predictions for classification problems, based on a bottleneck of high-level concepts. Existing CBM methods rely on a linear aggregation of the concept scores to compute predictions. However, a large number of concepts is often used in this linear approach, which undermines explainability and favors information leakage.
  In general, the underlying relation between concepts and output logits is not linear. Therefore, we introduce Hoeffding Concept Bottleneck Models (HCBM), which build on the Hoeffding functional decomposition of gradient-boosted trees to provide non-linear and sparse aggregations of concept scores, and generate compact predictions using prime implicants. HCBM have strong theoretical properties of robustness to interconcept leakage, and outperform sparse linear CBM in practice, as shown in extensive experiments.
  Beyond classification, HCBM can be adapted to object detection, and we focus on a challenging case with overhead images to show the high performance of HCBM in these settings.
\end{abstract}

\section{Introduction}

Deep learning algorithms have demonstrated impressive accuracy performance in computer vision for classification and object-detection problems. However, the lack of explainability of their prediction mechanisms remains a strong limitation, especially for domains with critical decisions at stake, such as healthcare or industry sectors. Therefore, the field of explainable AI (XAI) has attracted high interest in the machine learning community to develop new algorithms explaining the inner workings of black-box neural networks. 
The first successful XAI methods in computer vision are probably attribution methods. Their outputs take the form of heatmaps to highlight the pixels of an input image, which have a strong impact on the variations of the network output logits \citep{simonyan2013deep, selvaraju2017grad, fel2021look}. These approaches are very effective to locate the relevant image areas used by the neural network to build its predictions. However, attribution methods face a main limitation, since they only tell ``where the network looks, but not what it sees'' \citep{adebayo2018sanity, fel2023craft}. In other words, the logits of different classes can be highly sensitive to the same image areas, and the prediction cannot be explained by attribution methods in these cases.
A new type of methods based on high-level concepts provides more efficient and precise explanations \citep{poeta2023concept}, and has thus gained strong momentum in the community, in particular Concept Bottleneck Models (CBM).

\paragraph{Concept Bottleneck Models.} CBM were originally introduced by \citet{koh2020concept} to explain image classification using textual concepts, which are typically class attributes or abstract characteristics of images. For example, in animal classification, relevant concepts for the class ``bird'' can be ``legs'' or ``beak'', or ``an animal on a tree branch''.
The principle of CBM is to introduce a new layer in the neural network, positioned just before the output layer, where each neuron represents the activation strength of a given concept. Hence, the neural network is trained to predict the concept scores for an input image, and then to predict the output logits from an aggregation of these scores, which is constrained to take a simple form to provide explainable predictions. 
CBM were initially learned from datasets with annotated concepts. However, the collection of datasets with such annotations is a tedious and time-consuming task, and limits the use of CBM for large datasets and concept sets. More recently, multimodal models have enabled the automatic generation of concepts from image datasets, and have thus led to a new type of more efficient CBM based on vision language models (VLM), which do not require manual annotations.
Therefore, several VLM-CBM algorithms have been developed in the past few years, to take advantage of this automatic concept generation, especially using CLIP \citep{radford2021learning}, as summarized in Table \ref{table:sota_cbm} in Appendix \ref{app:biblio}. While PCBM is the pioneering work \citep{yuksekgonul2022post}, several algorithms have later been introduced to refine the concept generation and selection to improve accuracy \citep{oikarinen2023label, yang2023language, shang2024incremental, rao2024discover, liu2025hybrid, zhao2025partially, enouen2026debugging, santis2026learning}.
However, almost all these methods aggregate concept scores with a linear layer, which limits model performance. 
Most VLM-CBM exhibit a high accuracy on standard datasets using a large number of concepts, which undermines explainability. For example, \citet{midavaine2024re} show that PCBM has about $500$ non-null coefficients in its final layer for CIFAR-10 dataset, with only $10$ classes. Consequently, the average number of concepts used to model each class logit is about $50$, and the CBM is hardly explainable. This metric is defined by \citet{srivastava2024vlg} as the Number of Effective Concepts (NEC), who recommend NEC values around $5$ to obtain CBM providing efficient explanations. 

\paragraph{Information leakage.}
Beyond the complexity of the final concept aggregation model, CBM suffer from a more critical limitation, known as information leakage \citep{havasi2022addressing}. This phenomenon occurs when additional information is encoded in the concept scores and is not explicitly represented in the concept list, or when a specific concept stores information from other ones. These two types of information leakage are respectively called concept-task leakage and interconcept leakage in \citet{parisini2025leakage}.
When information leakage is strong, the concept bottleneck becomes similar to a standard layer in the neural network, maximizing classification accuracy, but explainability is lost, since the concepts are not faithfully represented. In particular, using a large number of concepts in the bottleneck clearly favors information leakage, as extensively discussed in previous articles \citep{srivastava2024vlg}. It is therefore recommended to strongly limit the number of concepts involved, by controlling the NEC value, for example.
However, there is another major source of information leakage.
Most CBM use a sparse linear model to aggregate the concepts in the final layer, as already mentioned. In fact, the output logits of the targeted classes are often not linear functions of the concept scores, and this approximation favors interconcept leakage, since correlated concepts compensate for each other to recover a final aggregation with a high accuracy, leading to the selection of irrelevant concepts in the bottleneck, as illustrated in Appendices \ref{app:linear_leakage} and \ref{app:xp_results}. A promising direction to overcome this problem is to drop linearity and use non-linear models instead. Since the core principle of CBM is to have an explainable concept aggregation, it is clearly not possible to simply use standard black-box learning algorithms (e.g., MLP, tree ensembles...). Hence, we will show that the Hoeffding functional decomposition is the relevant approach to model logits as non-linear and explainable functions of concept scores.

\paragraph{Hoeffding functional decomposition and prime implicants.}
The Hoeffding functional decomposition breaks down any function of a set of input variables into a sum of functional components, which all have a subset of variables as argument. 
Importantly, this decomposition enjoys several attractive properties, under mild assumptions, that are not shared by other additive models. The Hoeffding decomposition is unique, and often sparse, with mainly functional components of one or two variables, which are intrinsically explainable, since they can be displayed. The decomposition also performs causal variable selection, when there are no hidden confounders. This implies that only the components of input variables with unique information regarding the output function take non-null values, even with hidden confounders. The Hoeffding decomposition was originally introduced by \citet{hoeffding1948}, and later extended by \citet{stone1994use}, \citet{hooker2007generalized}, and \citet{chastaing2012hoeffding} to handle dependent input variables. 
The practical estimation of the Hoeffding decomposition is notoriously challenging, but \citet{benard2025tree} recently introduced the TreeHFD algorithm to estimate the decomposition from an XGBoost model  \citep{chen2016xgboost}, which is a highly efficient algorithm to build ensembles of gradient-boosted trees \citep{friedman2001greedy}, known to be a state-of-the-art method in this context with numeric concept scores as inputs \citep{grinsztajn2022tree}. See Appendix \ref{app:treehfd} for additional details about the Hoeffding decomposition and TreeHFD.
Besides, in CBM applications, only a list of the most important concepts is often displayed to explain pointwise predictions, and the length of this list is somewhat arbitrary. Prime implicants provide a more objective approach to select sufficient subsets of concepts, based on logic---see \citet{shchdarwIJCAI18}, \citet{ignamarAAAI19}, and \citet{darwhiECAI20}. The core principle is that a set of concepts is not important for a given prediction, if changing their scores to any other possible value does not change the final prediction.

\paragraph{Contributions.}
The goal of the article is to introduce Hoeffding Concept Bottleneck Models (HCBM), where the concept aggregation is based on the Hoeffding functional decomposition of gradient-boosted tree models, to increase the sparsity of CBM and their robustness to information leakage, while preserving a high accuracy. In particular, HCBM automatically select relevant concepts, enforce monotone constraints in the concept aggregation, and predict compact subsets of concepts using prime implicants.
We show the high empirical performance of the proposed HCBM algorithm with extensive experiments.
Finally, we assess HCBM beyond standard classification, through a case study of object detection in overhead images with xView dataset. In this case, object resolutions are often small, and the design of relevant concepts for CBM is a challenging problem. However, we show that HCBM successfully adapt to these difficult settings and outperform sparse linear CBM.

\section{Hoeffding Concept Bottleneck Models} \label{sec:hcbm}

We first describe HCBM for classification, and then show the robustness to information leakage.

\subsection{HCBM algorithm}
An input image is represented by a random vector $\bX \in \R^d$, associated with a label $Y \in \{1,\hdots,L\}$, where $L \in \mathbb{N}^{\star}$ is the number of labels. A dataset of labeled images $\Dn = \{(\bX_i, Y_i)\}_{i=1,\hdots,n}$ of size $n$ with all observations distributed as $(\bX, Y)$ is available.
The list of predefined textual concepts of length $K \in \mathbb{N}^{\star}$ is defined by $C_K = \{c_1, \hdots, c_K \}$. In the original definition of CBM \citep{koh2020concept}, the model is written as $f \circ g$, where $g$ is the mapping of images in the concept space, and $f$ is the aggregation function of the concept scores into the output label logits. We define below $g$ and $f$ for HCBM, and their empirical counterparts $g_n$ and $f_n$ estimated from the dataset $\Dn$ in practice. We finally show how predictions are explained using prime implicants.

\paragraph{Concept mapping.}
The image and text embeddings of the VLM (e.g., CLIP or SigLIP \citep{zhai2023sigmoid}) are respectively denoted by $\Phi_I$ and $\Phi_T$. 
For $j \in \{1,\hdots,K\}$, we define $\smash{Z^{(j)}}$ as the score of concept $c_j$ for image $\bX$ by the usual cosine similarity \citep{radford2021learning}, that is \vspace*{-3mm}
\begin{align*}
    Z^{(j)} = \frac{\langle \Phi_I(\bX), \Phi_T(c_j) \rangle}{\| \Phi_I(\bX) \|_2 \| \Phi_T(c_j) \|_2},
\end{align*}
where $\langle \cdot, \cdot \rangle$ is the standard scalar product in the Euclidean space $\R^{d_L}$, with $d_L$ the dimension of the VLM latent space, and $\|\cdot\|_2$ is the induced $2$-norm.
Then, the HCBM mapping $g$ of images in the concept space is given for each component by $\smash{g^{(j)}(\bX) = F^{(j)}(Z^{(j)})}$, where $\smash{F^{(j)}}$ is the min-max normalization to map the support of $\smash{Z^{(j)}}$ to $[0, 1]$, which always exists since the cosine similarity is bounded. This normalization enforces $g(\bX)$ to take values in the unit cube $[0,1]^K$, while preserving dependence between concepts. The cosine similarity of CLIP may typically vary between $0.1$ and $0.3$ for a given dataset and concept, and it is therefore hard to understand if a given score is high or low for a specific image. With the min-max normalization, such interpretation becomes clear with respect to the training distribution. In practice, the concept mapping $g_n$ is estimated using $\Dn$, through the affine transforms $\smash{F_n^{(j)}}$, directly computed from the minimum and maximum values of $\smash{\{Z_i^{(j)}\}_{i=1,\hdots,n}}$.

\paragraph{Concept aggregation.}
For a label $\ell \in \{1,\hdots,L\}$, the associated logit $f_{\ell}$ of CBM takes the general form $f_{\ell}(g(\bX)) = \mathrm{logit}(\P(Y = \ell \mid g(\bX)))$, where $\mathrm{logit}$ is the standard logit function, combined with a saturation of the input probability to constant values in arbitrarily small neighborhoods of $0$ and $1$. Such saturation has no practical impact, and enforces that $f_{\ell} \circ g$ is always well-defined.
While most CBM assume that $f_{\ell}$ is linear to provide explainable estimates, the main principle of HCBM is to use the Hoeffding functional decomposition of $f_{\ell}$ to obtain a highly accurate and explainable concept aggregation, without using a strong approximation of the logit functions. Such decomposition takes the form of a sum of low-order functions of essentially one or two variables, as explained in the introduction. Under the mild assumptions of bounded density and support of $g(\bX)$, this decomposition is unique and exactly adds up to the target logit $f_{\ell}(g(\bX))$, as formalized in the following theorem, proved in Appendix \ref{app:proof}.
\begin{theorem} \label{thm:hcbm}
    If the functions $f_{1}, \hdots, f_{L}$ are square integrable, each concept score takes values in a compact set, and the density of $g(\bX)$ is bounded away from $0$ and infinity on the product support, then there exists a unique set of functions $\smash{\{f_{\ell}^{(J)}\}_{J \subset \{1, \hdots, K \}}}$ for each label $\ell \in \{1,\hdots,L\}$, such that 
    \begin{align} \label{eq:logit_hfd}
        f_{\ell}(g(\bX)) = \sum_{J \subset \{1, \hdots, K \}} f_{\ell}^{(J)}(g^{(J)}(\bX)), 
    \end{align} 
    where $g^{(J)}(\bX)$ is the subvector of $g(\bX)$ with only the components in $J$, and for all $I \subsetneq J \subset \{1, \hdots, K \}$, we have $\smash{\E[f_{\ell}^{(J)}(g^{(J)}(\bX)) \mid g^{(I)}(\bX)] = 0}$. 
\end{theorem}
The last equation of the above result states that the decomposition components are hierarchically orthogonal. In other words, a given component cannot contain lower-order effects, which implies that a component is null if it is possible to break down the target logit using only lower-order terms. Hence, these orthogonality constraints enforce uniqueness and sparsity of the decomposition. It is even often sufficient to only consider main effects in the decomposition. Indeed, if several concepts have an interaction, a new concept can be added to approximate this interaction.

The core of HCBM is to estimate the above theoretical decompositions using $\Dn$, to provide the logit estimate $f_{n, \ell}$ of each label $\ell$ in an explainable form. Therefore, we combine gradient-boosted models from XGBoost \citep{chen2016xgboost} and the TreeHFD algorithm \citep{benard2025tree}, which computes the Hoeffding functional decomposition of XGBoost models, defined by $\smash{\{f^{(J)}_{n, \ell}\}_{J, \ell}}$, such that $\smash{f_{n, \ell}(g_n(\bX)) = \sum_{J \subset \{1, \hdots, K \}} f_{n, \ell}^{(J)}(g_n^{(J)}(\bX))}$,
and the orthogonality constraints are satisfied---see Appendix \ref{app:treehfd}.
Since the concept mapping transforms the problem into a tabular learning task with numeric inputs, XGBoost is known to be a state-of-the-art method in this context, as already mentioned. Additionally, TreeHFD estimates the Hoeffding decomposition of  XGBoost models with strong convergence guarantees. Therefore, TreeHFD provides highly accurate estimates of the Hoeffding decomposition of the logits with respect to the concept scores, formalized in Equation (\ref{eq:logit_hfd}), and remains explainable through its additive form of low-order functions.

\paragraph{Concept selection and monotone constraints.}
Only a small fraction of the components $\smash{f_{\ell}^{(J)}}$ have a non-negligible impact on the output logits in most cases. To enforce sparsity of the concept aggregation, components are discarded when their relative variance is smaller than a threshold $\alpha > 0$, since the relative variance is the standard approach to quantify a component importance \citep[Equation $3$]{chastaing2012hoeffding}. Consequently, we only select the components satisfying 
\begin{align} \label{eq:cimp}
    \V[f_{\ell}^{(J)}(g^{(J)}(\bX))] > \alpha \times \V[f_{\ell}(g(\bX))].
\end{align}
The parameter $\alpha$ directly controls the NEC value, which decreases as $\alpha$ increases, and $\alpha$ is thus chosen to satisfy the NEC input value.
Furthermore, the impact of the concept score $g^{(J)}$ on the logit of a given label is monotone for most concepts, since each concept is supposed to be associated positively or negatively with the label, and $g^{(J)}$ measures the activation strength of a concept for a given image. 
For example, the concept of ``beak'' is positively associated with a bird, and negatively with a dog.
Therefore, we shift $\smash{f_{\ell}^{(J)}}$ functions to enforce $\smash{f_{\ell}^{(J)}(\mathbf{0}) = 0}$, to increase model explainability, and update the intercept $\smash{f_{\ell}^{(\emptyset)}}$ accordingly to keep the aggregation untouched. Consequently, all increasing components $\smash{f_{\ell}^{(J)}}$ take only positive values for concepts that are positively associated with label $\ell$, and only negative values otherwise for decreasing components of negative concepts.

In practice, we first fit XGBoost and TreeHFD from $\Dn$, with all concepts involved and without monotone constraints. Then, we compute $\alpha$ to meet the input NEC value through Equation (\ref{eq:cimp}) and the fitted decompositions $\smash{\{f^{(J)}_{n, \ell}\}_{J, \ell}}$, and track the selected concepts for each label. We also compute the relative cumulated increase of each main effect $\smash{f_{n, \ell}^{(j)}}$ for $\smash{j \in \{1, \hdots, K\}}$, to detect increasing and decreasing components, even if the variations are noisy---see Appendix \ref{app:hcbm} for all details about monotone constraints.
Once all monotone constraints are derived, we can move to the final step: we refit XGBoost and TreeHFD from $\Dn$, with only the selected concepts of each label, and with the selected monotone constraints enforced in each XGBoost model.
Finally, we shift the outputs of TreeHFD to enforce $\smash{f_{n, \ell}^{(J)}(\mathbf{0}) = 0}$, to improve model explainability, as explained before. 

Overall, the learning procedure of HCBM is summarized in Algorithm \ref{algo_hcbm}, where we only consider main effects in TreeHFD decompositions to improve clarity.
The computational complexity of HCBM is quasilinear with respect to   the sample size $n$, and linear with respect to the number of initial concepts, the number of trees in XGBoost models, and the number of output labels $L$. Finally, we highlight that it is possible to define a specific concept list for each label, in the same spirit as \citet{zhao2025partially}, and the required adaptation of Step $4$ in Algorithm \ref{algo_hcbm} is straightforward.

\begin{algorithm}
\caption{HCBM Learning}
\label{algo_hcbm}
\begin{algorithmic}[1]
    \REQUIRE A dataset $\Dn = \{(\bX_i, Y_i)\}_{i=1}^n$ of labeled images, a list of concepts $C_K = \{c_1, \hdots, c_K \}$, and the NEC value.
    \STATE Compute the image and text embeddings for all images of $\Dn$ and concepts of $C_K$, and deduce the cosine similarity $\smash{Z_i^{(j)}}$ of all pairs of images $i \in \{1,\hdots,n\}$ and concepts $j \in \{1,\hdots,K\}$.
    \STATE Estimate the empirical affine transform $F_n^{(j)}$ of each concept score $Z^{(j)}$ from $\smash{\{Z_i^{(j)}\}_{i=1,\hdots,n}}$, to obtain the concept mapping $\smash{g_n^{(j)}(\bX_i) = F_n^{(j)}(Z_i^{(j)})}$.
    \FOR{$\ell \in \{1, \hdots, L\}$}
        \STATE Fit XGBoost with all concept scores $g_n^{(1)}, \hdots, g_n^{(K)}$ as inputs and the binary output indicating label $\ell$.
        \STATE Fit TreeHFD for the XGBoost logit of label $\ell$.
        \STATE Compute the importance of each TreeHFD component using Equation (\ref{eq:cimp}).
    \ENDFOR
    \STATE Set $\alpha$ to select the total number of TreeHFD components given by the input NEC, using Equation (\ref{eq:cimp}) and the importance values computed at Step $6$.
    \STATE Compute the monotone constraints of the selected TreeHFD components.
    \STATE Retrain XGBoost and TreeHFD with the selected concepts and monotone constraints. 
    \STATE Shift the decomposition components to enforce $\smash{f_{n, \ell}^{(J)}(\mathbf{0}) = 0}$.
\end{algorithmic}
\end{algorithm}

\paragraph{Local explanations.}
Prime implicants provide sufficient sets of concepts to explain each prediction, as introduced before.
A prime implicant for a specific prediction is thus a minimal subset of concepts such that the prediction remains unchanged, regardless of how the scores of the other concepts are modified.
Formally, there exists a permutation $\pi$ of $\{1, \hdots, L \}$ to order the label logits such that $f_{n, \pi(1)}(g_n(\bX)) > f_{n, \pi(2)}(g_n(\bX)) > \hdots > f_{n, \pi(L)}(g_n(\bX))$, for a new image $\bX$, and the final prediction is the label $\pi(1)$.
For the sake of clarity, we focus on the case where only main effects are involved in the Hoeffding decomposition, but the approach can be extended to the general case.
By definition, an implicant $S(\bX) \subset \{1, \hdots, K \}$ for prediction at $\bX$ satisfies for all $\ell \in \{1, \hdots, L\} \setminus \pi(1)$, \vspace*{-1mm}
\begin{align*}
    \delta_{\ell}^{(\emptyset)} + \sum_{j \in S(\bX)} \delta_{\ell}^{(j)}(g_n^{(j)}(\bX))
    + \sum_{j \notin S(\bX)} \min_{z \in [0,1]} \delta_{\ell}^{(j)}(z)
    > 0, \quad \mathrm{where} \quad \delta_{\ell}^{(j)} = f_{n, \pi(1)}^{(j)} - f_{n, \ell}^{(j)}. \\[-2em]
\end{align*}
We focus on prime implicants $S(\bX)$, which can be found easily through the following procedure with a quasilinear complexity with respect to $K \times L$. Indeed, we simply rank the concepts by decreasing values of the quantity $\smash{h_{2}^{(j)}(\bX) = \delta_{\pi(2)}^{(j)}(g_n^{(j)}(\bX)) - \min_{z \in [0,1]} \delta_{\pi(2)}^{(j)}(z)}$, where this ranking is defined by the permutation $\sigma_{2}$. Then, we find the smallest integer $s_2$ such that $\smash{\sum_{j = 1}^{s_2} h_{2}^{(\sigma_2(j))}(\bX) > - \delta_{\pi(2)}^{(\emptyset)} - \sum_{j = 1}^K \min_{z \in [0,1]} \delta_{\pi(2)}^{(j)}(z)}$, and the concepts $S_2(\bX) = \{\sigma_2(1), \hdots, \sigma_2(s_2)\}$ are sufficient to explain that label $\pi(1)$ is predicted over $\pi(2)$. Finally, this approach is sequentially applied to all logits $\pi(3), \hdots, \pi(L)$ to obtain $\smash{\cup_{\ell=2}^L S_{\ell}(\bX)}$, which is finally pruned to form a prime implicant $S(\bX)$ for HCBM prediction---see Appendix \ref{app:hcbm} for details. Besides, $\smash{S(\bX)}$ can be displayed for each label using a waterfall plot, as illustrated in the experimental section.

\subsection{Robustness to information leakage}

We show that HCBM automatically discard irrelevant concepts, even when they contain useful information about target labels through interconcept leakage with other relevant concepts of the initial list $C_K$.
Since information leakage within concepts is a well-known and critical limitation of CBM, this property gives strong guarantees that HCBM only include relevant concepts containing unique information regarding the considered label in the final aggregation $f_{\ell}$, and enforces a high sparsity, which favors model explainability.
Importantly, this property does not hold for standard CBM with a linear concept aggregation---see Appendix \ref{app:linear_leakage}, or when other non-linear additive models are used in HCBM instead of TreeHFD.
Hence, we show in the following theorem that the Hoeffding decomposition of $f_{\ell}$ discards irrelevant concepts by construction, as a consequence of the uniqueness and sparsity properties enforced by the orthogonality constraints.
\begin{theorem} \label{thm:hfd_irrelevant}
    Let $\smash{J_{\ell}^{\star} \subset \{1, \hdots, K \}}$ be the indices of the relevant concepts for label $\ell \in \{1, \hdots, L \}$, defined as the smallest subset of concepts such that $\smash{\P(Y = \ell \mid g(\bX)) = \P(Y = \ell \mid g^{(J_{\ell}^{\star})}(\bX))}$, and let the assumptions of Theorem \ref{thm:hcbm} be satisfied.
    Then, $\smash{J_{\ell}^{\star}}$ is unique, and the functional components including irrelevant concepts in the Hoeffding decomposition of $f_{\ell}(g(\bX))$ are null, that is, for all $\smash{J \not \subset J_{\ell}^{\star}}$, $\smash{f_{\ell}^{(J)}(g^{(J)}(\bX)) = 0}$ a.s.
\end{theorem}
In particular, the score of an irrelevant concept $c_j$ with $\smash{j \in \{1, \hdots, K \} \setminus J_{\ell}^{\star}}$, which suffers from interconcept leakage, can be formalized as $\smash{g^{(j)}(\bX) = \nu^{(j)}(\bX, \varepsilon_{\nu}) + \mu^{(j)}(g^{(J_{\ell}^{\star})}(\bX), \varepsilon_{\mu})}$ where $\smash{\nu^{(j)}}$ and $\smash{\mu^{(j)}}$ are two functions taking values in $[0, 1/2]$ with independent noise $\varepsilon_{\nu}$ and $\varepsilon_{\mu}$, such that $\smash{\nu^{(j)}(\bX, \varepsilon_{\nu}) \indep Y \mid g^{(J_{\ell}^{\star})}(\bX)}$. This means that $\smash{\nu^{(j)}
}$ does not contain additional information about the target labels, and $\smash{\mu^{(j)}}$ models the interconcept leakage. Hence, Theorem \ref{thm:hfd_irrelevant} provides strong guarantees that HCBM automatically remove such concepts with interconcept leakage---see also Appendix \ref{app:proof} for theoretical guarantees of the concept selection performed by HCBM fitted with $\Dn$ in practice.
In the following experimental section, we will show that these theoretical properties have strong practical consequences to outperform linear CBM.

\section{Empirical assessment} \label{sec:xp}

The performance of CBM algorithms strongly relies on both the initial list of concepts and the concept aggregation model. The core of HCBM is to improve the concept aggregation using the explainable and non-linear Hoeffding decomposition of XGBoost models, which can obviously be combined with any initial list of concepts.
Therefore, we run comparisons between the standard sparse linear aggregation of CBM and the TreeHFD aggregation used in HCBM for a given dataset and initial concept list in various cases---see Table \ref{table:xp_hcbm}.
In fact, linear CBM are the state-of-the-art competitors of HCBM, since the sparse linear aggregation is used by almost all existing CBM, with the notable exceptions of SCBM \citep{vandenhirtz2024stochastic} and CB2 \citep{atienza2024cutting}, which both report lower accuracy than the linear approach. For example for CIFAR-10, linear CBM and HCBM both achieve a classification accuracy of $0.96$ with NEC at $10$, while SCBM and CB2 respectively achieve $0.72$ and $0.91$ without sparsity constraints. The concept aggregation of CB2 is also entangled with the alignment of an additional black-box vision classifier, which is a valuable feature, but out of the article scope.
On the other hand, we compare HCBM when TreeHFD is replaced by EBM \citep{nori2019interpretml}, the most widely used algorithm to build non-linear additive models.
Importantly, post-hoc XAI methods for XGBoost such as TreeSHAP \citep{lundberg2020local} cannot be used in HCBM instead of TreeHFD, since they cannot handle monotone constraints, 
and do not have the property of robustness to information leakage.

\begin{table}
\caption{Using CLIP, classification accuracy of HCBM with TreeHFD versus sparse linear CBM and HCBM with EBM (with standard deviations $\leq 10^{-3}$), in the upper part.
Below, proportion ($\%$) of irrelevant concepts with injected leakage selected by HCBM versus sparse linear CBM and EBM.}
\centering
\setlength{\tabcolsep}{1pt}
\resizebox{\textwidth}{!}{
\begin{tabular}{cccccccccccccccccc}
  \hline
  \hline
  NEC & \multicolumn{3}{c}{Waterbirds-2} & \multicolumn{3}{c}{xView-3} & \multicolumn{3}{c}{EuroSAT-10} & \multicolumn{3}{c}{RESISC-45} & \multicolumn{3}{c}{CUB-200} & \multicolumn{2}{c}{ImageNet-1000}\\
   & Linear & \multicolumn{2}{c}{HCBM} & Linear & \multicolumn{2}{c}{HCBM} & Linear & \multicolumn{2}{c}{HCBM} & Linear & \multicolumn{2}{c}{HCBM} & Linear & \multicolumn{2}{c}{HCBM} & Linear & HCBM \\
   & & \small{EBM} & \small{TreeHFD} & & \small{EBM} & \small{TreeHFD} & & \small{EBM} & \small{TreeHFD} & & \small{EBM} & \small{TreeHFD} & & \small{EBM} & \small{TreeHFD} & & \small{TreeHFD} \\
   \hline
   $3$ & $0.542$ & $0.691$ & $0.692$ & $0.689$ & $0.925$ & $0.871$ & $0.742$ & $0.783$ & $0.812$ & $0.673$ & $0.123$ & $0.816$ & $0.294$ & $0.077$ & $0.318$ & $0.275$ & $0.554$ \\
   $5$ & $0.575$ & $0.737$ & $0.677$ & $0.906$ & $0.950$ & $0.942$ & $0.830$ & $0.843$ & $0.864$ & $0.835$ & $0.172$ & $0.845$ & $0.441$ & $0.116$ & $0.404$ & $0.412$ & $0.605$ \\
   $10$ & $0.678$ & $0.719$ & $0.704$ & $0.959$ & $0.967$ & $0.964$ & $0.897$ & $0.901$ & $0.900$ & $0.883$ & $0.234$ & $0.875$ & $0.645$ & $0.155$ & $0.626$ & $0.621$ & $0.669$ \\
   \hline 
   $5$ & $78\%$ & $40\%$ & $0.0\%$ & $59\%$ & $3.2\%$ & $0.0\%$ & $50\%$ & $1.2\%$ & $0.0\%$ & $42\%$ & $64\%$ & $3.5\%$ & $72\%$ & $65\%$ & $24\%$ & $72\%$ & $3.8\%$ \\
   $10$ & $86\%$ & $34\%$ & $6.0\%$ & $57\%$ & $5.6\%$ & $1.7\%$ & $53\%$ & $2.4\%$ & $0.7\%$ & $48\%$ & $62\%$ & $11\%$ & $74\%$ & $63\%$ & $33\%$ & $69\%$ & $15\%$ \\
   \hline
   \hline
\end{tabular}}
\label{table:xp_hcbm}
\end{table}

\paragraph{Experiment settings.}
Experiments are conducted with standard image datasets, Waterbirds \citep{Sagawa2020Distributionally}, CIFAR \citep{krizhevsky2009learning}, CUB \citep{wah2011caltech}, and ImageNet \citep{russakovsky2015imagenet}, and several datasets of overhead images, that is EuroSAT \citep{helber2019eurosat} and RESISC \citep{cheng2017remote}, both used in previous CBM articles, and also xView \citep{lam2018xview}---see Appendix \ref{app:xp_settings} for details about datasets and concepts.
We use existing lists of concepts introduced in other CBM articles to focus on the comparisons of the concept aggregation models, except for xView dataset. In this case, there is no existing work about CBM, and no available list of concepts, to our best knowledge.
Therefore, we carefully craft a list of relevant concepts, following the standard approach of combining LLM generation and expert knowledge, as described in the next section dedicated to xView. A list of labeled bounding boxes is provided in the training data, and we extract crops with identified objects of three parent classes to build the image set.  
Regarding the settings of HCBM, we fit XGBoost, TreeHFD, and EBM with all default parameters, except that TreeHFD and EBM use only main effects, since interactions are excluded in any case because of the NEC constraints---see Appendix \ref{app:xp_settings} for additional details about experimental settings.
For the sparse linear aggregation, we use the loss function originally introduced in PCBM \citep{yuksekgonul2022post}, which is the cross-entropy loss with an elastic-net penalization ($L1$-ratio of $0.99$), and a penalization strength set to meet the NEC input, exactly as in \citet{srivastava2024vlg}.
Finally, we use the balanced classification accuracy as performance metric, where each class has the same weight. 

\begin{figure*}
	\begin{center}
		\includegraphics[scale=0.335]{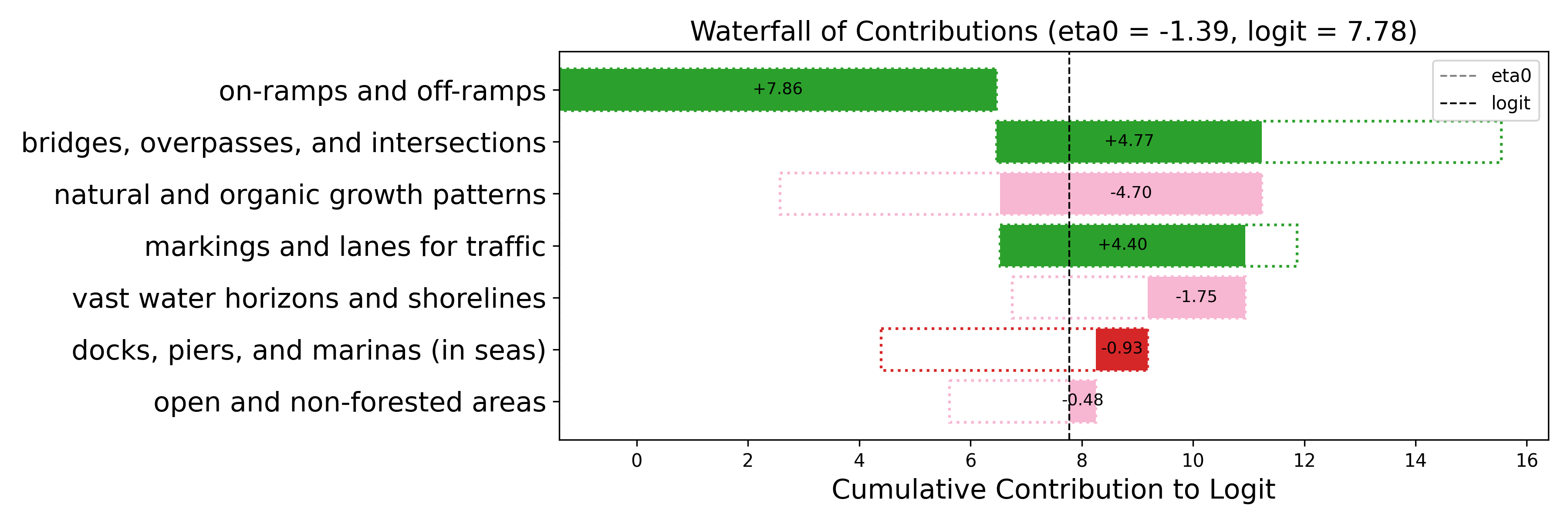}
		\hspace*{5mm}
		\includegraphics[scale=1.35]{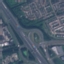}
		\caption{HCBM waterfall of concept contributions to prediction of ``Highway\_1009'' in EuroSAT.} \label{fig:waterfall_highway}
	\end{center}
\end{figure*}

\paragraph{Results.}
Experimental results are displayed in Table \ref{table:xp_hcbm} using CLIP, where metrics are averaged over five repetitions to make standard deviations negligible. Table \ref{table:xp_hcbm} shows a large accuracy improvement of HCBM over linear CBM, in particular for NEC values of $3$ and $5$---see also Tables \ref{table:xp_clip}--\ref{table:xp_siglip_large_classes} in Appendix \ref{app:xp_results} for CIFAR and SigLIP \citep{zhai2023sigmoid}.
For NEC at $10$, linear CBM and HCBM have close accuracy performance, except for Waterbirds, CUB, and ImageNet datasets. For Waterbirds, the higher accuracy of HCBM shows its better generalization, since the test set has out-of-distribution data. We also take advantage of this case to show the higher performance of HCBM with interventions---see Appendix \ref{app:interventions}.
For datasets with a large number of classes, HCBM outperform linear CBM for ImageNet, but not for CUB, a dataset with a small number of images per class ($30$), a setting where sparse linear models are highly efficient.
Notice that HCBM can be easily scaled to ImageNet through the downsampling of the negative class at Step $4$ of Algorithm \ref{algo_hcbm}.
Importantly, linear CBM often exhibit a competitive accuracy for NEC at $10$, but rely on strong information leakage to achieve such performance. An inspection of linear CBM models confirms this problem.
For CIFAR-10 dataset with NEC at $10$, for example, linear CBM select many irrelevant positive concepts for the ``frog'' class, such as ``ratline, uropygial gland, mammal, horse’s foot, bird’s foot, feline'', but only one for HCBM.
\begin{wrapfigure}{r}{0.55 \textwidth}
	\begin{center}
        \includegraphics[scale=0.3]{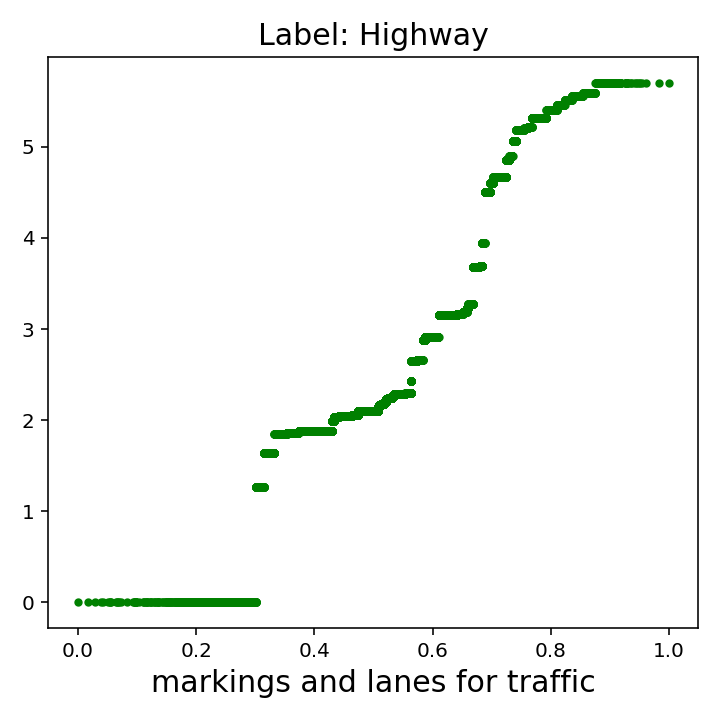} \hspace*{-0.2cm}
		\includegraphics[scale=0.3]{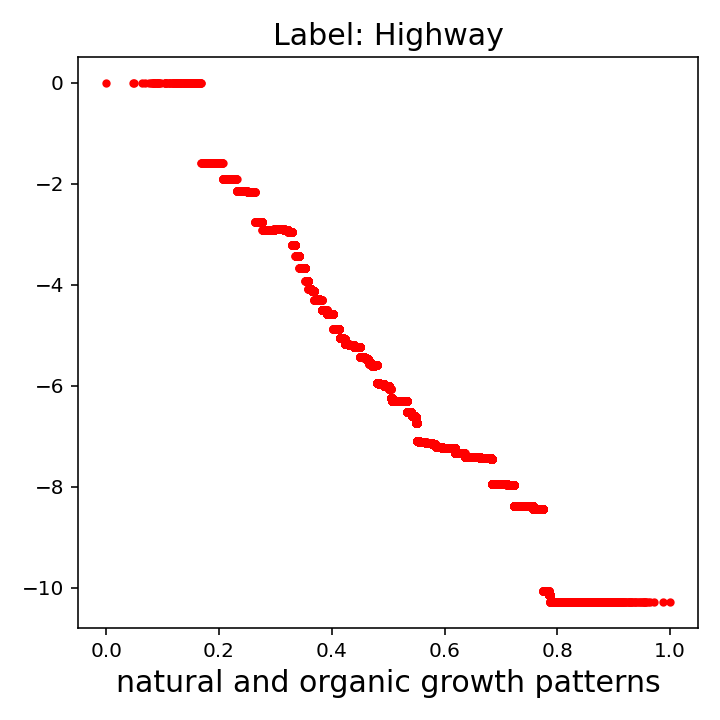}
		\caption{For EuroSAT dataset, main effects of concepts in the logit decompositions of ``Highway'' in HCBM.} \label{fig:eurosat_components}
	\end{center}
\end{wrapfigure}
For EuroSAT dataset, linear CBM also learn an increasing effect of the concepts ``streetlights and road signs'' and ``regularly maintained and paved surfaces'' on ``Sea \& Lake'' probability, even with NEC at $5$,  but not HCBM---see Appendix \ref{app:xp_results} for additional details.
Besides, the HCBM variant with EBM can have a slightly better accuracy than HCBM with TreeHFD for Waterbirds and xView datasets, which both have few classes ($2$ and $3$), but at the price of the very costly default number of boosting rounds of $50 000$ for EBM, while XGBoost uses the default number of $100$ rounds for TreeHFD. 
For the other datasets with larger numbers of classes, 
HCBM with TreeHFD strongly outperform EBM for accuracy, and are more robust to interconcept leakage, as explained in the next paragraph. In particular, EBM struggles with RESISC-45 and CUB-200, because of the high number of concepts and classes, and computations are intractable for ImageNet---see Appendix \ref{app:xp_settings}.

For EuroSAT dataset, we also display the main effects of several concepts in the decomposition provided by HCBM for the logit of highway class in Figure \ref{fig:eurosat_components}, and also in Figures \ref{fig:eurosat_components_full}--\ref{fig:xview_HCBM_5} in Appendices \ref{app:xp_results}--\ref{app:xview_concepts} (with NEC $= 5$).
We observe that the identification of ``markings and lanes for traffic'' or ``on-ramps and off-ramps'' strongly increases the probability to classify the image as a highway. Such patterns are quite expected, since these elements are characteristics of highways, and can be seen from overhead images. The concept ``natural and organic growth patterns'' has a quite strong negative impact on the highway logit, since the concept is associated with most other classes, such as crops, forests, pasture, vegetation, river...
Overall, it is clear from Figures \ref{fig:eurosat_components} and \ref{fig:eurosat_components_full}--\ref{fig:xview_HCBM_5} that the impact of concept scores on output logits is often not linear. The common behavior is a quite sharp increase or decrease when the concept score reaches a given threshold, corresponding to the identification of the concept. 
For low and high values of concept scores, the impact on logits is quite constant, which ensures a good behavior when the model is queried for images with scores outside of the training distribution support. HCBM can efficiently handle this typical behavior of concept scores, in particular in the transition area where the logit contribution strongly varies, while linear CBM or hard CBM \citep{havasi2022addressing} provide a rough approximation, with a line or a step function, respectively.

\paragraph{Information leakage.}
We show the high robustness of HCBM based on TreeHFD to interconcept leakage with the following experiments. For each dataset, we add irrelevant concepts with injected leakage, built from the scores of the original concepts and noise---see Appendix \ref{app:xp_settings} for details. Then, we run linear CBM and HCBM, and report in Table \ref{table:xp_hcbm} the proportion of concepts selected by each method that contain injected leakage. linear CBM select about half of irrelevant concepts with injected leakage, whereas HCBM with TreeHFD select few irrelevant concepts. HCBM with EBM also outperform linear CBM, but select more irrelevant concepts than TreeHFD, especially in the cases with large numbers of classes and concepts. This is quite expected since the property of robustness to interconcept leakage does not hold for EBM---see Appendix \ref{app:xp_results} for additional experiments.

\paragraph{Local explanation of predictions.}
Prime implicants efficiently explain predictions using about half of the concepts initially selected in HCBM learning, and thus provide compact explanations---see Tables \ref{table:xp_clip_implicants}-\ref{table:xp_clip_implicants_large_classes} in Appendix \ref{app:xp_results}.
The numbers of implicant concepts are close for linear CBM and HCBM.
In Figure \ref{fig:waterfall_highway}, we show the prediction breakdown with a waterfall plot for the example of image ``Highway\_1009'' in EuroSAT dataset (NEC $= 10$). The concepts are ranked by decreasing order of absolute contribution for the highest logit, that is the highway class in this case. The dashed rectangles show the lowest and highest possible values that the concept contribution can take across the training data. For example, the concept of ``natural and organic growth patterns'' takes an intermediate value of $-4.70$ for our target image because of surrounding elements, but can take much lower values around $-9$ for images of the other classes that are positively associated with this concept, as opposed to the highway class. The implicant concepts are highlighted in the waterfall plot with dark-green color for positive concepts and dark red for negative ones: ``on-ramps and off-ramps'', ``bridges, overpasses, and intersections'', ``markings and lanes for traffic'', and ``docks, piers, and marinas''. 
The highway logit involves three other concepts (displayed with lighter colors), but even if they are assigned to their lowest possible value, the highway logit remains larger than all other updated logits. Therefore, the implicant concepts are sufficient to explain the prediction.

\section{Applications to object detection in overhead images} \label{sec:detection}

We show the high performance of HCBM with the challenging problem of object detection in overhead images with xView dataset. We first describe the adaptation of HCBM to object detection, and then analyze experimental results.

\paragraph{HCBM for object detection.}
Convolutional neural networks (CNN) have been successfully extended to handle object detection, in particular with the famous YOLO \citep{redmon2016you} and RTMDet \citep{lyu2022rtmdet} architectures. These methods perform two tasks simultaneously with the detection of bounding boxes containing targeted objects, and the classification of each detected bounding box to one of the target labels. 
We can combine HCBM with CNN detectors to explain the classification of the detected objects. In practice, the training data of HCBM is generated by a trained CNN detector, which outputs a set of crops and predicted logits for the CNN detections in raw images. Then, the HCBM model is fit with this generated dataset and an initial concept list, as explained in Section \ref{sec:hcbm}, with the extension of HCBM to handle continuous outputs. More precisely, the logits predicted by the detector replace the binary outputs in Step $4$ of Algorithm \ref{algo_hcbm} in HCBM training. Then, both XGBoost and TreeHFD natively handle regression problems, and the final components of the Hoeffding decomposition take the same form as in the classification case.
The final inference pipeline consists of the combination of the CNN detector to predict the bounding boxes for a new input image, and then the classification of the associated crops is performed by the fitted HCBM. 

\paragraph{Experiments with xView dataset.}
The xView dataset is especially challenging, because many objects have very small resolutions. For example, some cars are represented by a few pixels---as illustrated in Figure \ref{fig:waterfall_car}, which makes the design of relevant concepts a difficult problem. Hence, we focus on three parent classes: vehicles, aircraft, and maritime vessels. 
The concept list is built using LLM and expert knowledge, as mentioned in the previous section.
\begin{wraptable}{r}{0.48 \textwidth}
\caption{\small{For HCBM and linear CBM, explained variance of each logit and classification accuracy, with respect to predictions of an RTMDet detector for xView.}}
\centering
\setlength{\tabcolsep}{1pt}
\begin{tabular}{cccccc}
  \hline
  \hline
  \small{NEC} & \small{Algorithm} & \multicolumn{3}{c}{\small{Logit explained variance}} &  \small{Accuracy} \\
   & & \small{Vehicles} & \small{Aircraft} & \small{Boats} &   \\
  \hline
   \multirow{2}{*}{$3$} & \small{Linear} & $0.27$ & $0.41$ & $0.21$ & $0.48$ \\
    & \small{HCBM} & $0.68$ & $0.67$ & $0.63$ & $0.88$ \\
   \hline
   \multirow{2}{*}{$5$}  & \small{Linear} & $0.41$ & $0.52$ & $0.40$ & $0.67$ \\
    & \small{HCBM} & $0.76$ & $0.72$ & $0.65$ & $0.90$ \\
   \hline
   \multirow{2}{*}{$10$} & \small{Linear} & $0.64$ & $0.66$ & $0.61$ & $0.88$ \\
    & \small{HCBM} & $0.80$ & $0.76$ & $0.73$ & $0.93$ \\
  \hline 
  \hline 
\end{tabular}
\label{table:xp_hcbm_detection}
\end{wraptable}
We extract two crops for each detected object, where the enlarged crop has a width and height ten times larger than the original bounding box, to also consider background elements, e.g., the sea or the road. Consequently, we design two types of concepts: those characterizing object attributes for the original bounding boxes, which essentially describe object geometry and colors, and those for background elements in enlarged crops. Overall, the final list contains $73$ concepts, and is provided in Appendix \ref{app:xview_concepts} for reproducibility.
Then, we run experiments for object detection with xView dataset, following the methodology defined above, and also using linear CBM instead of HCBM---see Appendix \ref{app:xview_concepts} for all setting details.
The results are displayed in Table \ref{table:xp_hcbm_detection}, with the proportion of explained variance obtained by the surrogate CBM for each logit, and the balanced classification accuracy, estimated by $5$-fold cross-validation.
We observe that HCBM strongly outperform linear CBM to reproduce the predictions of the RTMDet detector with a high accuracy. In particular, the accuracy gaps between linear CBM and HCBM are larger in this case, where we fit the continuous logits of the detector, than in the classification case of the previous section. Indeed, the non-linear aggregation of HCBM has a higher flexibility than a linear model to handle continuous outputs.
Finally, we illustrate the HCBM model with the prediction of the detected car of Figure \ref{fig:waterfall_car}, displayed as a waterfall plot. We see that two positive concepts have a strong contribution to the logit of the ``vehicle'' class, corresponding to the identification of the road surrounding the car, and a ``white blurry shape on darker background'', and four other concepts have a negative contribution, since they are associated with aircraft or maritime vessels.
\begin{figure*}
	\begin{center}
		\includegraphics[scale=0.34]{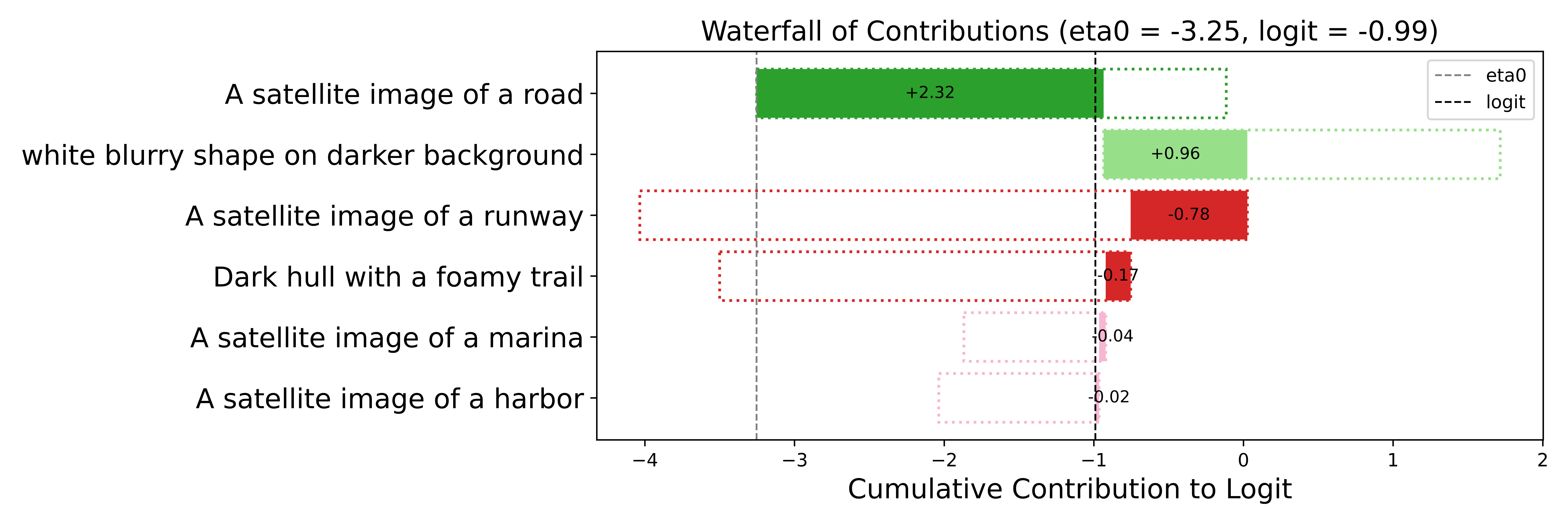}
		\includegraphics[scale=0.36]{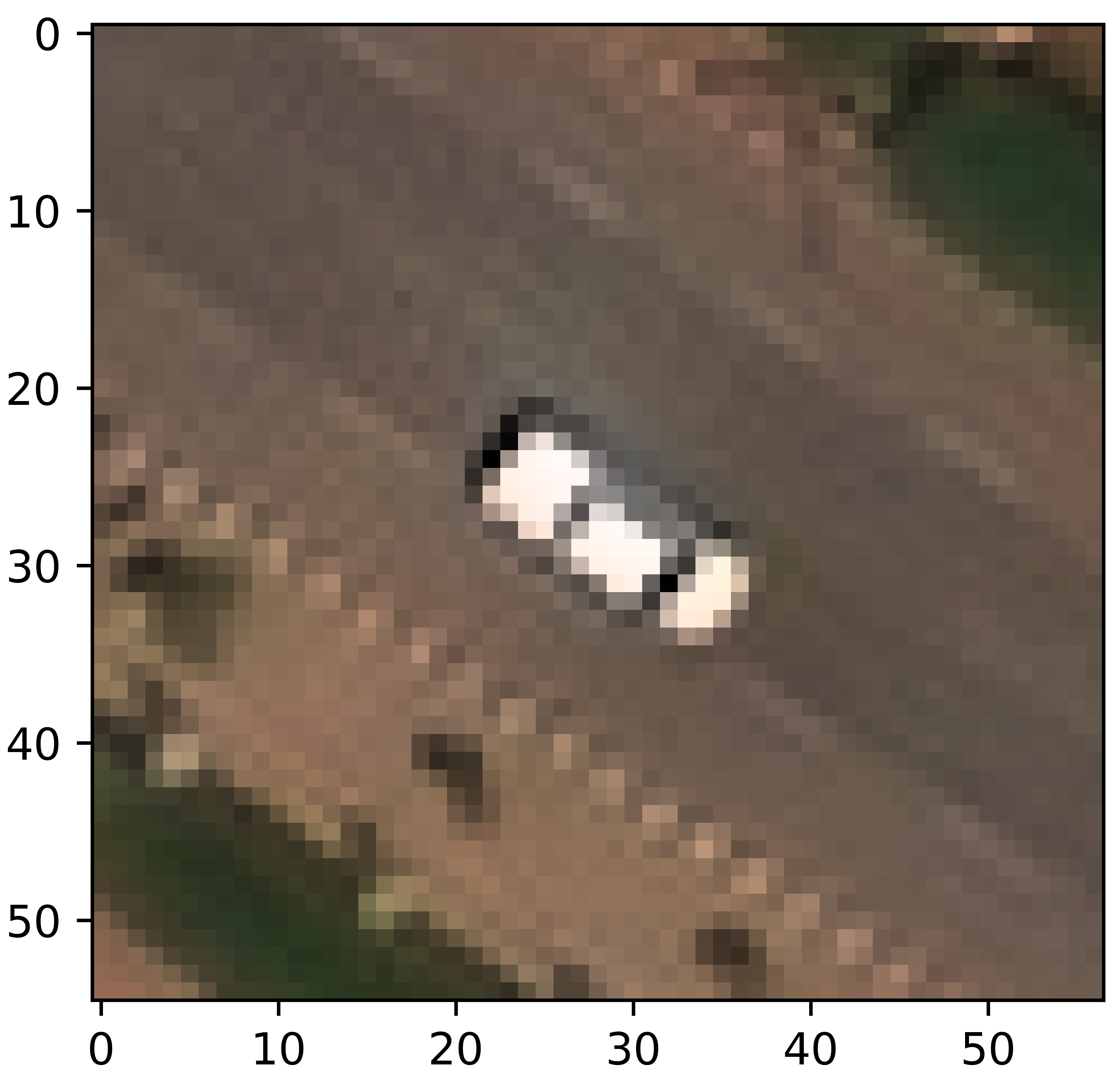}
		\caption{HCBM waterfall of concept contributions to prediction of a car in xView dataset.} \label{fig:waterfall_car}
	\end{center}
\end{figure*}

\section{Conclusion and limitations}

We have introduced HCBM to improve sparsity and robustness to information leakage over standard CBM with sparse linear concept aggregations, while preserving a high accuracy. In particular, HCBM ability to remove interconcept leakage has strong theoretical groundings.
When the number of classes in the data is large, HCBM are still robust to interconcept leakage, but the accuracy can be slightly lower than linear CBM for cases with a small number of images per class.
Additionally, HCBM can be efficiently adapted to object detection to explain object classification of black-box detectors, but HCBM do not explain the object localization of detectors. 
Finally, HCBM share a current limitation of CBM with concept-task leakage, which may still occur for HCBM, as illustrated in Appendix \ref{app:xp_results}. Therefore, the development of procedures to generate and select concepts to eliminate concept-task leakage seems an important and promising research direction for future work.

\subsection*{AI use statement}

In this work, we used generative AI tools to generate synthetic data sets. More precisely, generative AI tools were used to generate concept lists in the experiments.

We have not used generative AI tools to help develop theoretical models or conceptual frameworks, formulate mathematical claims, provide critical ingredients for proving mathematical claims, assist in the writing of proofs, propose or refine hypotheses, design or provide feedback on research  methodology or experiments, implement methods, assist with translation, clean and reformat dataset, support qualitative and thematic data analysis, or interpret results.

Additionally, generative AI tools were used to identify a few typos in the final version of the article, which were manually corrected.

We have reviewed all AI-assisted work. Most concept lists were taken from previously published articles about CBM, and we carefully checked and edited the concept list for xView dataset.
We take responsibility for the final content of this work, including text, claims or artifacts produced with the aid of generative AI.

\subsection*{Ethics statement}

The article states that black-box deep learning algorithms may have negative societal impacts in critical domains, such as healthcare or industry sectors. The article is expected to have a positive societal impact, since it introduces a new algorithm to improve CBM, which aims at explaining deep learning models for safer use in critical applications.

\subsection*{Reproducibility statement}

All experiment details are provided in Sections \ref{sec:xp}--\ref{sec:detection} and Appendix \ref{app:xp_settings}, and all datasets are available online.



\bibliography{biblio}

\begin{thebibliography}{58}
\providecommand{\natexlab}[1]{#1}
\providecommand{\url}[1]{\texttt{#1}}
\expandafter\ifx\csname urlstyle\endcsname\relax
  \providecommand{\doi}[1]{doi: #1}\else
  \providecommand{\doi}{doi: \begingroup \urlstyle{rm}\Url}\fi

\bibitem[Adebayo et~al.(2018)Adebayo, Gilmer, Muelly, Goodfellow, Hardt, and
  Kim]{adebayo2018sanity}
Julius Adebayo, Justin Gilmer, Michael Muelly, Ian Goodfellow, Moritz Hardt,
  and Been Kim.
\newblock Sanity checks for saliency maps.
\newblock \emph{Advances in neural information processing systems}, 31, 2018.

\bibitem[Atienza et~al.(2024)Atienza, Bresson, Rousselot, Caillou, Cohen,
  Labreuche, and Sebag]{atienza2024cutting}
Nicolas Atienza, Roman Bresson, Cyriaque Rousselot, Philippe Caillou, Johanne
  Cohen, Christophe Labreuche, and Mich{\`e}le Sebag.
\newblock Cutting the black box: Conceptual interpretation of a deep neural net
  with multi-modal embeddings and multi-criteria decision aid.
\newblock In \emph{Thirty-Third International Joint Conference on Artificial
  Intelligence IJCAI-24}, pp.\  3669--3678. International Joint Conferences on
  Artificial Intelligence Organization, 2024.

\bibitem[Benard(2025)]{benard2025tree}
Clement Benard.
\newblock {Tree Ensemble Explainability through the Hoeffding Functional
  Decomposition and Tree{HFD} Algorithm}.
\newblock In \emph{The Thirty-ninth Annual Conference on Neural Information
  Processing Systems}, 2025.

\bibitem[Chastaing et~al.(2012)Chastaing, Gamboa, and
  Prieur]{chastaing2012hoeffding}
Gaelle Chastaing, Fabrice Gamboa, and Cl{\'e}mentine Prieur.
\newblock {Generalized Hoeffding-Sobol decomposition for dependent variables -
  application to sensitivity analysis}.
\newblock \emph{Electronic Journal of Statistics}, 6:\penalty0 2420 -- 2448,
  2012.

\bibitem[Chen \& Guestrin(2016)Chen and Guestrin]{chen2016xgboost}
Tianqi Chen and Carlos Guestrin.
\newblock {XGBoost: A scalable tree boosting system}.
\newblock In \emph{Proceedings of the 22nd ACM SIGKDD International Conference
  on Knowledge Discovery and Data Mining}, pp.\  785--794, New York, 2016. ACM.

\bibitem[Cheng et~al.(2017)Cheng, Han, and Lu]{cheng2017remote}
Gong Cheng, Junwei Han, and Xiaoqiang Lu.
\newblock Remote sensing image scene classification: Benchmark and state of the
  art.
\newblock \emph{Proceedings of the IEEE}, 105:\penalty0 1865--1883, 2017.

\bibitem[Darwiche \& Hirth(2020)Darwiche and Hirth]{darwhiECAI20}
Adnan Darwiche and Auguste Hirth.
\newblock On the reasons behind decisions.
\newblock In \emph{Proceedings of the European Conference on Artificial
  Intelligence (ECAI 2020)}, pp.\  712--720, Santiago, Spain, August 2020.

\bibitem[{De Santis} et~al.(2026){De Santis}, Tong, Brambilla, and
  Kagal]{santis2026learning}
Antonio {De Santis}, Schrasing Tong, Marco Brambilla, and Lalana Kagal.
\newblock Learning concept bottleneck models from mechanistic explanations.
\newblock In \emph{The Fourteenth International Conference on Learning
  Representations}, 2026.

\bibitem[Enouen \& Galhotra(2026)Enouen and Galhotra]{enouen2026debugging}
Eric Enouen and Sainyam Galhotra.
\newblock Debugging concept bottleneck models through removal and retraining.
\newblock In \emph{The Fourteenth International Conference on Learning
  Representations}, 2026.
\newblock URL \url{https://openreview.net/forum?id=zZNYUkBS77}.

\bibitem[Fel et~al.(2021)Fel, Cad{\`e}ne, Chalvidal, Cord, Vigouroux, and
  Serre]{fel2021look}
Thomas Fel, R{\'e}mi Cad{\`e}ne, Mathieu Chalvidal, Matthieu Cord, David
  Vigouroux, and Thomas Serre.
\newblock Look at the variance! efficient black-box explanations with
  sobol-based sensitivity analysis.
\newblock \emph{Advances in neural information processing systems},
  34:\penalty0 26005--26014, 2021.

\bibitem[Fel et~al.(2023)Fel, Picard, Bethune, Boissin, Vigouroux, Colin,
  Cad{\`e}ne, and Serre]{fel2023craft}
Thomas Fel, Agustin Picard, Louis Bethune, Thibaut Boissin, David Vigouroux,
  Julien Colin, R{\'e}mi Cad{\`e}ne, and Thomas Serre.
\newblock Craft: Concept recursive activation factorization for explainability.
\newblock In \emph{Proceedings of the IEEE/CVF Conference on Computer Vision
  and Pattern Recognition}, pp.\  2711--2721, 2023.

\bibitem[Fong \& Vedaldi(2018)Fong and Vedaldi]{fong2018net2vec}
Ruth Fong and Andrea Vedaldi.
\newblock Net2vec: Quantifying and explaining how concepts are encoded by
  filters in deep neural networks.
\newblock In \emph{2018 IEEE/CVF Conference on Computer Vision and Pattern
  Recognition}, pp.\  8730--8738. IEEE, 2018.

\bibitem[Friedman(2001)]{friedman2001greedy}
Jerome~H. Friedman.
\newblock Greedy function approximation: {A} gradient boosting machine.
\newblock \emph{Annals of Statistics}, pp.\  1189--1232, 2001.

\bibitem[Galliamov et~al.(2026)Galliamov, Kazmi, Khan, and
  Ram{\'\i}rez~Rivera]{galliamov2026concepts}
Karim Galliamov, Syed Muhammad Ahsan~Raza Kazmi, Adil Khan, and Ad{\'\i}n
  Ram{\'\i}rez~Rivera.
\newblock Concepts' information bottleneck models.
\newblock In \emph{International Conference on Learning Representations},
  volume 2026, pp.\  82405--82430, 2026.

\bibitem[Grinsztajn et~al.(2022)Grinsztajn, Oyallon, and
  Varoquaux]{grinsztajn2022tree}
L{\'e}o Grinsztajn, Edouard Oyallon, and Ga{\"e}l Varoquaux.
\newblock Why do tree-based models still outperform deep learning on typical
  tabular data?
\newblock \emph{Advances in Neural Information Processing Systems},
  35:\penalty0 507--520, 2022.

\bibitem[Havasi et~al.(2022)Havasi, Parbhoo, and
  Doshi-Velez]{havasi2022addressing}
Marton Havasi, Sonali Parbhoo, and Finale Doshi-Velez.
\newblock Addressing leakage in concept bottleneck models.
\newblock \emph{Advances in Neural Information Processing Systems},
  35:\penalty0 23386--23397, 2022.

\bibitem[Helber et~al.(2019)Helber, Bischke, Dengel, and
  Borth]{helber2019eurosat}
Patrick Helber, Benjamin Bischke, Andreas Dengel, and Damian Borth.
\newblock {EuroSAT: A novel dataset and deep learning benchmark for land use
  and land cover classification}.
\newblock \emph{IEEE Journal of Selected Topics in Applied Earth Observations
  and Remote Sensing}, 12:\penalty0 2217--2226, 2019.

\bibitem[Hoeffding(1948)]{hoeffding1948}
Wassily Hoeffding.
\newblock {A Class of Statistics with Asymptotically Normal Distribution}.
\newblock \emph{The Annals of Mathematical Statistics}, 19:\penalty0 293 --
  325, 1948.

\bibitem[Hooker(2007)]{hooker2007generalized}
Giles Hooker.
\newblock Generalized functional anova diagnostics for high-dimensional
  functions of dependent variables.
\newblock \emph{Journal of Computational and Graphical Statistics},
  16:\penalty0 709--732, 2007.

\bibitem[Ignatiev et~al.(2019)Ignatiev, Narodytska, and
  Marques-Silva]{ignamarAAAI19}
Alexey Ignatiev, Nina Narodytska, and Joao Marques-Silva.
\newblock Abduction-based explanations for machine learning models.
\newblock In \emph{AAAI}, pp.\  1511--1519, Honolulu, Hawai, January 2019.

\bibitem[Koh et~al.(2020)Koh, Nguyen, Tang, Mussmann, Pierson, Kim, and
  Liang]{koh2020concept}
Pang~Wei Koh, Thao Nguyen, Yew~Siang Tang, Stephen Mussmann, Emma Pierson, Been
  Kim, and Percy Liang.
\newblock Concept bottleneck models.
\newblock In \emph{International conference on machine learning}, pp.\
  5338--5348. PMLR, 2020.

\bibitem[Krishna et~al.(2017)Krishna, Zhu, Groth, Johnson, Hata, Kravitz, Chen,
  Kalantidis, Li, Shamma, et~al.]{krishna2017visual}
Ranjay Krishna, Yuke Zhu, Oliver Groth, Justin Johnson, Kenji Hata, Joshua
  Kravitz, Stephanie Chen, Yannis Kalantidis, Li-Jia Li, David~A Shamma, et~al.
\newblock Visual genome: Connecting language and vision using crowdsourced
  dense image annotations.
\newblock \emph{International journal of computer vision}, 123\penalty0
  (1):\penalty0 32--73, 2017.

\bibitem[Krizhevsky et~al.(2009)Krizhevsky, Hinton,
  et~al.]{krizhevsky2009learning}
Alex Krizhevsky, Geoffrey Hinton, et~al.
\newblock Learning multiple layers of features from tiny images.
\newblock Technical report, Toronto, ON, Canada, 2009.

\bibitem[Lam et~al.(2018)Lam, Kuzma, McGee, Dooley, Laielli, Klaric, Bulatov,
  and McCord]{lam2018xview}
Darius Lam, Richard Kuzma, Kevin McGee, Samuel Dooley, Michael Laielli, Matthew
  Klaric, Yaroslav Bulatov, and Brendan McCord.
\newblock xview: Objects in context in overhead imagery.
\newblock \emph{arXiv preprint arXiv:1802.07856}, 2018.

\bibitem[Lengerich et~al.(2020)Lengerich, Tan, Chang, Hooker, and
  Caruana]{lengerich2020purifying}
Benjamin Lengerich, Sarah Tan, Chun-Hao Chang, Giles Hooker, and Rich Caruana.
\newblock Purifying interaction effects with the functional anova: An efficient
  algorithm for recovering identifiable additive models.
\newblock In \emph{International Conference on Artificial Intelligence and
  Statistics}, pp.\  2402--2412. PMLR, 2020.

\bibitem[Liang \& Zou(2022)Liang and Zou]{liang2022metashift}
Weixin Liang and James Zou.
\newblock Metashift: A dataset of datasets for evaluating contextual
  distribution shifts and training conflicts.
\newblock In \emph{International Conference on Learning Representations}, 2022.
\newblock URL \url{https://openreview.net/pdf?id=MTex8qKavoS}.

\bibitem[Liu et~al.(2025)Liu, Zhang, and Gu]{liu2025hybrid}
Yang Liu, Tianwei Zhang, and Shi Gu.
\newblock Hybrid concept bottleneck models.
\newblock In \emph{Proceedings of the Computer Vision and Pattern Recognition
  Conference}, pp.\  20179--20189, 2025.

\bibitem[Lundberg et~al.(2020)Lundberg, Erion, Chen, DeGrave, Prutkin, Nair,
  Katz, Himmelfarb, Bansal, and Lee]{lundberg2020local}
Scott~M Lundberg, Gabriel Erion, Hugh Chen, Alex DeGrave, Jordan~M Prutkin,
  Bala Nair, Ronit Katz, Jonathan Himmelfarb, Nisha Bansal, and Su-In Lee.
\newblock From local explanations to global understanding with explainable ai
  for trees.
\newblock \emph{Nature Machine Intelligence}, 2:\penalty0 56--67, 2020.

\bibitem[Lyu et~al.(2022)Lyu, Zhang, Huang, Zhou, Wang, Liu, Zhang, and
  Chen]{lyu2022rtmdet}
Chengqi Lyu, Wenwei Zhang, Haian Huang, Yue Zhou, Yudong Wang, Yanyi Liu,
  Shilong Zhang, and Kai Chen.
\newblock Rtmdet: An empirical study of designing real-time object detectors.
\newblock \emph{arXiv preprint arXiv:2212.07784}, 2022.

\bibitem[Midavaine et~al.(2024)Midavaine, Go, Canez, Simion, and
  Chatterji]{midavaine2024re}
Nesta Midavaine, Gregory Hok~Tjoan Go, Diego Canez, Ioana Simion, and Satchit
  Chatterji.
\newblock On the reproducibility of post-hoc concept bottleneck models.
\newblock \emph{Transactions on Machine Learning Research}, 2024.
\newblock ISSN 2835-8856.

\bibitem[Nguyen et~al.(2025)Nguyen, Li, and Peng]{nguyen2025interpretable}
Kien~X Nguyen, Tang Li, and Xi~Peng.
\newblock Interpretable failure detection with human-level concepts.
\newblock In \emph{Proceedings of the AAAI Conference on Artificial
  Intelligence}, volume~39, pp.\  26326--26334, 2025.

\bibitem[Nori et~al.(2019)Nori, Jenkins, Koch, and
  Caruana]{nori2019interpretml}
Harsha Nori, Samuel Jenkins, Paul Koch, and Rich Caruana.
\newblock Interpretml: A unified framework for machine learning
  interpretability.
\newblock \emph{arXiv preprint arXiv:1909.09223}, 2019.

\bibitem[Oikarinen et~al.(2023)Oikarinen, Das, Nguyen, and
  Weng]{oikarinen2023label}
Tuomas Oikarinen, Subhro Das, Lam~M. Nguyen, and Tsui-Wei Weng.
\newblock Label-free concept bottleneck models.
\newblock In \emph{The Eleventh International Conference on Learning
  Representations}, 2023.

\bibitem[Panousis et~al.(2024)Panousis, Ienco, and Marcos]{panousis2024coarse}
Konstantinos Panousis, Dino Ienco, and Diego Marcos.
\newblock Coarse-to-fine concept bottleneck models.
\newblock \emph{Advances in Neural Information Processing Systems},
  37:\penalty0 105171--105199, 2024.

\bibitem[Parisini et~al.(2025)Parisini, Chakraborti, Harbron, MacArthur, and
  Banerji]{parisini2025leakage}
Enrico Parisini, Tapabrata Chakraborti, Chris Harbron, Ben~D MacArthur, and
  Christopher~RS Banerji.
\newblock Leakage and interpretability in concept-based models.
\newblock \emph{arXiv preprint arXiv:2504.14094}, 2025.

\bibitem[Pedregosa et~al.(2011)Pedregosa, Varoquaux, Gramfort, Michel, Thirion,
  Grisel, Blondel, Prettenhofer, Weiss, Dubourg, et~al.]{pedregosa2011scikit}
Fabian Pedregosa, Ga{\"e}l Varoquaux, Alexandre Gramfort, Vincent Michel,
  Bertrand Thirion, Olivier Grisel, Mathieu Blondel, Peter Prettenhofer, Ron
  Weiss, Vincent Dubourg, et~al.
\newblock Scikit-learn: Machine learning in python.
\newblock \emph{Journal of Machine Learning Research}, 12:\penalty0 2825--2830,
  2011.

\bibitem[Poeta et~al.(2025)Poeta, Ciravegna, Pastor, Cerquitelli, and
  Baralis]{poeta2023concept}
Eleonora Poeta, Gabriele Ciravegna, Eliana Pastor, Tania Cerquitelli, and Elena
  Baralis.
\newblock Concept-based explainable artificial intelligence: A survey.
\newblock \emph{ACM Comput. Surv.}, 2025.

\bibitem[Radford et~al.(2021)Radford, Kim, Hallacy, Ramesh, Goh, Agarwal,
  Sastry, Askell, Mishkin, Clark, et~al.]{radford2021learning}
Alec Radford, Jong~Wook Kim, Chris Hallacy, Aditya Ramesh, Gabriel Goh,
  Sandhini Agarwal, Girish Sastry, Amanda Askell, Pamela Mishkin, Jack Clark,
  et~al.
\newblock Learning transferable visual models from natural language
  supervision.
\newblock In \emph{International Conference on Machine Learning}, pp.\
  8748--8763. PMLR, 2021.

\bibitem[Ragkousis \& Parbhoo(2024)Ragkousis and Parbhoo]{ragkousis2024tree}
Angelos Ragkousis and Sonali Parbhoo.
\newblock Tree-based leakage inspection and control in concept bottleneck
  models.
\newblock \emph{arXiv preprint arXiv:2410.06352}, 2024.

\bibitem[Rao et~al.(2024)Rao, Mahajan, B{\"o}hle, and Schiele]{rao2024discover}
Sukrut Rao, Sweta Mahajan, Moritz B{\"o}hle, and Bernt Schiele.
\newblock Discover-then-name: Task-agnostic concept bottlenecks via automated
  concept discovery.
\newblock In \emph{European Conference on Computer Vision}, pp.\  444--461.
  Springer, 2024.

\bibitem[Redmon et~al.(2016)Redmon, Divvala, Girshick, and
  Farhadi]{redmon2016you}
Joseph Redmon, Santosh Divvala, Ross Girshick, and Ali Farhadi.
\newblock You only look once: Unified, real-time object detection.
\newblock In \emph{Proceedings of the IEEE conference on computer vision and
  pattern recognition}, pp.\  779--788, 2016.

\bibitem[Russakovsky et~al.(2015)Russakovsky, Deng, Su, Krause, Satheesh, Ma,
  Huang, Karpathy, Khosla, Bernstein, et~al.]{russakovsky2015imagenet}
Olga Russakovsky, Jia Deng, Hao Su, Jonathan Krause, Sanjeev Satheesh, Sean Ma,
  Zhiheng Huang, Andrej Karpathy, Aditya Khosla, Michael Bernstein, et~al.
\newblock Imagenet large scale visual recognition challenge.
\newblock \emph{International journal of computer vision}, 115\penalty0
  (3):\penalty0 211--252, 2015.

\bibitem[Sagawa et~al.(2020)Sagawa, Koh, Hashimoto, and
  Liang]{Sagawa2020Distributionally}
Shiori Sagawa, Pang~Wei Koh, Tatsunori~B. Hashimoto, and Percy Liang.
\newblock Distributionally robust neural networks.
\newblock In \emph{International Conference on Learning Representations}, 2020.
\newblock URL \url{https://openreview.net/forum?id=ryxGuJrFvS}.

\bibitem[Selvaraju et~al.(2017)Selvaraju, Cogswell, Das, Vedantam, Parikh, and
  Batra]{selvaraju2017grad}
Ramprasaath~R Selvaraju, Michael Cogswell, Abhishek Das, Ramakrishna Vedantam,
  Devi Parikh, and Dhruv Batra.
\newblock Grad-cam: Visual explanations from deep networks via gradient-based
  localization.
\newblock In \emph{Proceedings of the IEEE international conference on computer
  vision}, pp.\  618--626, 2017.

\bibitem[Shang et~al.(2024)Shang, Zhou, Zhang, Ni, Yang, and
  Wang]{shang2024incremental}
Chenming Shang, Shiji Zhou, Hengyuan Zhang, Xinzhe Ni, Yujiu Yang, and Yuwang
  Wang.
\newblock Incremental residual concept bottleneck models.
\newblock In \emph{Proceedings of the IEEE/CVF Conference on Computer Vision
  and Pattern Recognition}, pp.\  11030--11040, 2024.

\bibitem[Shih et~al.(2018)Shih, Choi, and Darwiche]{shchdarwIJCAI18}
Andy Shih, Arthur Choi, and Adnan Darwiche.
\newblock A symbolic approach to explaining bayesian network classifiers.
\newblock In \emph{Proceedings of the Twenty-Seventh International Joint
  Conference on Artificial Intelligence (IJCAI 2018)}, pp.\  5103--5111,
  Stockholm, Sweden, July 2018.

\bibitem[Simonyan et~al.(2013)Simonyan, Vedaldi, and
  Zisserman]{simonyan2013deep}
Karen Simonyan, Andrea Vedaldi, and Andrew Zisserman.
\newblock Deep inside convolutional networks: Visualising image classification
  models and saliency maps.
\newblock \emph{arXiv preprint arXiv:1312.6034}, 2013.

\bibitem[Srivastava et~al.(2024)Srivastava, Yan, and Weng]{srivastava2024vlg}
Divyansh Srivastava, Ge~Yan, and Lily Weng.
\newblock {VLG-CBM}: Training concept bottleneck models with vision-language
  guidance.
\newblock \emph{Advances in Neural Information Processing Systems},
  37:\penalty0 79057--79094, 2024.

\bibitem[Stone(1994)]{stone1994use}
Charles~J Stone.
\newblock The use of polynomial splines and their tensor products in
  multivariate function estimation.
\newblock \emph{The Annals of Statistics}, 22:\penalty0 118--171, 1994.

\bibitem[Vandenhirtz et~al.(2024)Vandenhirtz, Laguna, Marcinkevi{\v{c}}s, and
  Vogt]{vandenhirtz2024stochastic}
Moritz Vandenhirtz, Sonia Laguna, Ri{\v{c}}ards Marcinkevi{\v{c}}s, and Julia
  Vogt.
\newblock Stochastic concept bottleneck models.
\newblock \emph{Advances in Neural Information Processing Systems},
  37:\penalty0 51787--51810, 2024.

\bibitem[Wah et~al.(2011)Wah, Branson, Welinder, Perona, and
  Belongie]{wah2011caltech}
Catherine Wah, Steve Branson, Peter Welinder, Pietro Perona, and Serge
  Belongie.
\newblock {The Caltech-UCSD Birds-200-2011 dataset}.
\newblock Technical report, {California Institute of Technology}, 2011.

\bibitem[Wu et~al.(2023)Wu, Yuksekgonul, Zhang, and Zou]{wu2023discover}
Shirley Wu, Mert Yuksekgonul, Linjun Zhang, and James Zou.
\newblock Discover and cure: Concept-aware mitigation of spurious correlation.
\newblock In \emph{International Conference on Machine Learning}, pp.\
  37765--37786. PMLR, 2023.

\bibitem[Yan et~al.(2023)Yan, Wang, Zhong, Dong, He, Lu, Wang, Shang, and
  McAuley]{yan2023learning}
An~Yan, Yu~Wang, Yiwu Zhong, Chengyu Dong, Zexue He, Yujie Lu, William~Yang
  Wang, Jingbo Shang, and Julian McAuley.
\newblock Learning concise and descriptive attributes for visual recognition.
\newblock In \emph{Proceedings of the IEEE/CVF International Conference on
  Computer Vision}, pp.\  3090--3100, 2023.

\bibitem[Yang et~al.(2023)Yang, Panagopoulou, Zhou, Jin, Callison-Burch, and
  Yatskar]{yang2023language}
Yue Yang, Artemis Panagopoulou, Shenghao Zhou, Daniel Jin, Chris
  Callison-Burch, and Mark Yatskar.
\newblock Language in a bottle: Language model guided concept bottlenecks for
  interpretable image classification.
\newblock In \emph{Proceedings of the IEEE/CVF conference on computer vision
  and pattern recognition}, pp.\  19187--19197, 2023.

\bibitem[Yuksekgonul et~al.(2023)Yuksekgonul, Wang, and
  Zou]{yuksekgonul2022post}
Mert Yuksekgonul, Maggie Wang, and James Zou.
\newblock Post-hoc concept bottleneck models.
\newblock In \emph{The Eleventh International Conference on Learning
  Representations}, 2023.

\bibitem[Zhai et~al.(2023)Zhai, Mustafa, Kolesnikov, and
  Beyer]{zhai2023sigmoid}
Xiaohua Zhai, Basil Mustafa, Alexander Kolesnikov, and Lucas Beyer.
\newblock Sigmoid loss for language image pre-training.
\newblock In \emph{2023 IEEE/CVF International Conference on Computer Vision
  (ICCV)}, pp.\  11941--11952. IEEE, 2023.

\bibitem[Zhao et~al.(2026)Zhao, Huang, Yan, Sun, and Yu]{zhao2025partially}
Delong Zhao, Qiang Huang, Di~Yan, Yiqun Sun, and Jun Yu.
\newblock {Partially Shared Concept Bottleneck Models}.
\newblock In \emph{Proceedings of the 40th AAAI Conference on Artificial
  Intelligence}, 2026.

\bibitem[Zhou et~al.(2017)Zhou, Lapedriza, Khosla, Oliva, and
  Torralba]{zhou2017places}
Bolei Zhou, Agata Lapedriza, Aditya Khosla, Aude Oliva, and Antonio Torralba.
\newblock Places: A 10 million image database for scene recognition.
\newblock \emph{IEEE transactions on pattern analysis and machine
  intelligence}, 40\penalty0 (6):\penalty0 1452--1464, 2017.

\end{thebibliography}
\bibliographystyle{iclr2027_conference}

\newpage
\appendix

\section{State-of-the-art algorithms} \label{app:biblio}

\paragraph{VLM-CBM.}
Table \ref{table:sota_cbm} provides a list of main VLM-CBM variants.

\paragraph{Information leakage.}
When concepts are annotated in the data, several articles introduce methods to mitigate information leakage, such as \citet{ragkousis2024tree} and \citet{galliamov2026concepts}.
However, annotated concepts are very often not available in the data, as explained in the introduction, and most CBM rely on VLM to quantify concept activations. Without the concept ground truth, mitigating information leakage is much more challenging, and HCBM tackle this more general problem. HCBM mitigate interconcept leakage with the Hoeffding decomposition, which enjoys strong concept selection properties, as stated in Theorems \ref{thm:hcbm}, \ref{thm:hfd_irrelevant}, and \ref{thm:concept_selection}.

\begin{table*}
\caption{Main VLM-CBM variants.}
\centering
\setlength{\tabcolsep}{2pt}
\begin{tabular}{ccccc}
  \hline
  \hline
  Name & Article & \thead{Concept \\ Generation} & \thead{Concept \\ Aggregation} & \thead{Black-box \\ Alignment}  \\
  \hline
  PCBM & \citet{yuksekgonul2022post} & ConceptNet & Linear & No \\
  LF-CBM & \citet{oikarinen2023label} & LLM & Linear & Yes \\
  LaBo & \citet{yang2023language} & LLM \& Selection & Linear & No \\
  LM4CV & \citet{yan2023learning} & LLM \& Selection & Linear & No  \\
  CB2 & \citet{atienza2024cutting} & LLM & NeurHCI \& Auto-Encoder & Yes \\
  SCBM & \citet{vandenhirtz2024stochastic} & LLM & Gaussian distribution & No  \\
  CF-CBM & \citet{panousis2024coarse} & LLM \& Selection & Linear & No \\
  ResCBM & \citet{shang2024incremental} & LLM \& Selection & Linear & No \\
  DN-CBM & \citet{rao2024discover} & SAE & Linear & No \\
  HybridCBM & \citet{liu2025hybrid} & LLM \& Selection & Linear & No \\
  PS-CBM & \citet{zhao2025partially} & LLM \& Selection & Linear & No \\
  CBDebug & \citet{enouen2026debugging} & LLM \& Selection & Linear & No \\
  MCBM & \citet{santis2026learning} & SAE & Linear & Partial \\
  \textbf{HCBM} & (ours)  & User input & TreeHFD & No \\
  \hline 
  \hline
\end{tabular}
\label{table:sota_cbm}
\end{table*}

\section{Illustration of concept selection} \label{app:linear_leakage}

We provide a short simulated example to show how linear aggregations of concept scores can favor interconcept leakage, while HCBM are robust to this problem.

We assume that a concept score $g^{(1)}$ is uniformly distributed over $[0, 1]$, and that the output logit $f_1$ of a given class only depends on this concept, with a typical form illustrated in the left panel of Figure \ref{fig:simulated_case}, and defined by
\begin{align*}
    f_1 = 10 \times (g^{(1)} - 0.5) \mathds{1}_{g^{(1)} > 0.5} + \varepsilon,
\end{align*}
where $\varepsilon \sim \mathcal{N}(0, 0.01)$ is an independent Gaussian noise.
We assume that a second concept score $g^{(2)}$ is a function of $g^{(1)}$ through interconcept leakage, and is defined as
\begin{align*}
    g^{(2)} = (1 - 2g^{(1)}) \mathds{1}_{g^{(1)} < 0.5} + \varepsilon_2,
\end{align*}
where $\varepsilon_2 \sim \mathcal{N}(0, 0.01)$ is also an independent Gaussian noise.
From a causal perspective, we say that $g^{(1)}$ has a causal effect on both $f_1$ and $g^{(2)}$, but that $g^{(2)}$ does not have a causal effect on $f_1$, and $f_1$ and $g^{(2)}$ are correlated because of the confounder $g^{(1)}$.

We draw a sample of size $n = 1000$ from this data generating process, and fit a linear model of $f_1$ with respect to only $g^{(1)}$. Next, we fit $f_1$ linearly using both $g^{(1)}$ and $g^{(2)}$.
In the first case, the proportion of explained variance of the output is about $0.80$, since $f_1$ cannot be represented by a linear function of $g^{(1)}$ with a high accuracy, as illustrated in the left panel of Figure \ref{fig:simulated_case}. In the second case, the proportion of explained variance increases to $0.94$, since $f_1$ is in fact a linear combination of $g^{(1)}$ and $g^{(2)}$ when the noise variables are removed. 
This toy example shows how linear aggregations can take advantage of interconcept leakage to recover an aggregation of high accuracy.

Finally, we fit TreeHFD with all default parameters to model $f_1$ with respect to $g^{(1)}$ and $g^{(2)}$, and obtain a proportion of explained variance of $0.99$. The two functional components are displayed in the right panel of Figure \ref{fig:simulated_case}. We see that the component for $g^{(1)}$ recovers the logit $f_1$ up to the intercept, and $g^{(2)}$ is very close to the null function.
Therefore, TreeHFD automatically removes the concept score $g^{(2)}$ from the aggregation, despite its strong correlation with $f_1$ and $g^{(1)}$, and HCBM are thus robust to interconcept leakage. This powerful property is a consequence of the uniqueness of the Hoeffding decomposition, enforced by the orthogonality constraints.
This example also illustrates that HCBM perform causal variable selection, when there are no hidden confounders, which is the case when the initial concept list contains all relevant concepts.

\begin{figure}
	\begin{center}
        \includegraphics[scale=0.27]{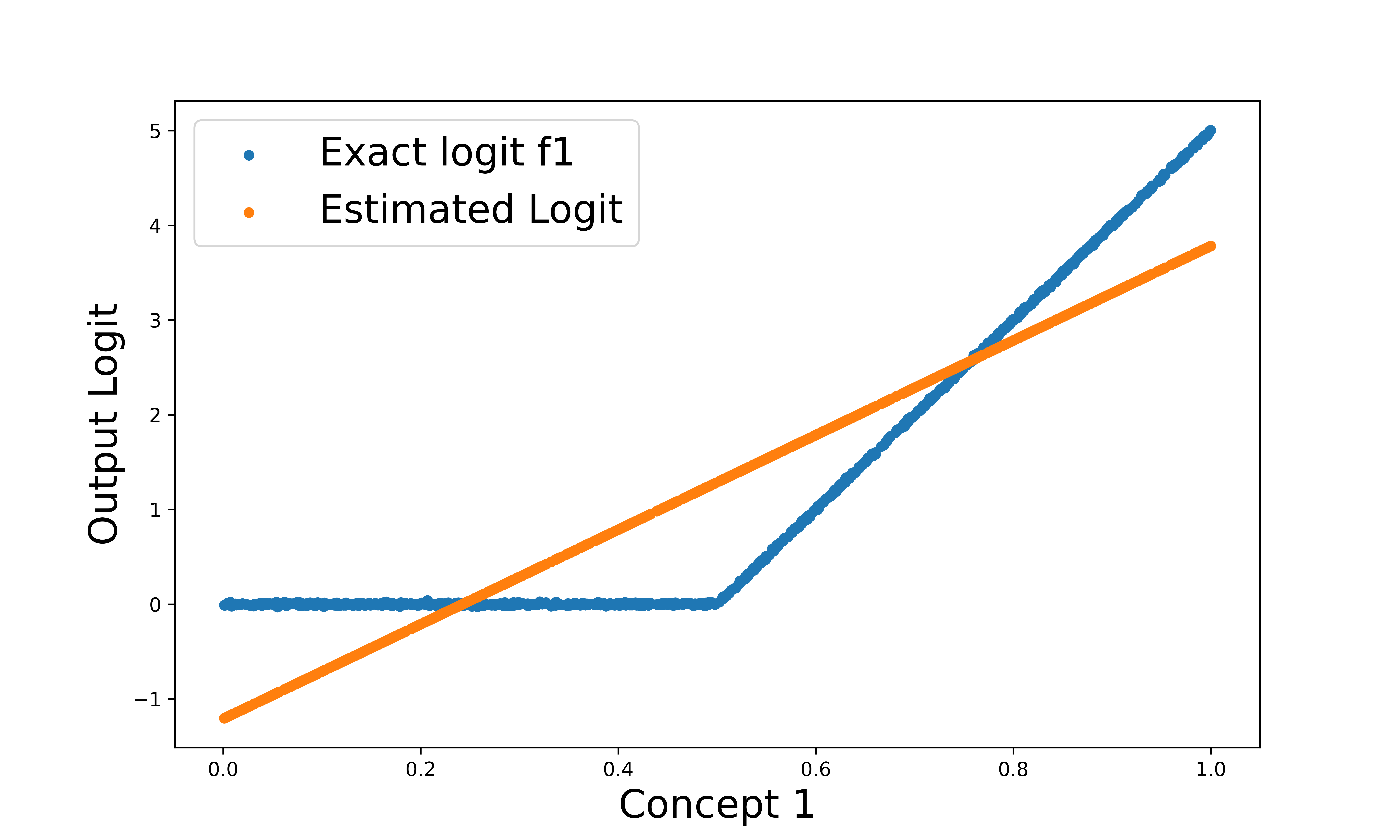} \hspace*{-0.2cm}
		\includegraphics[scale=0.27]{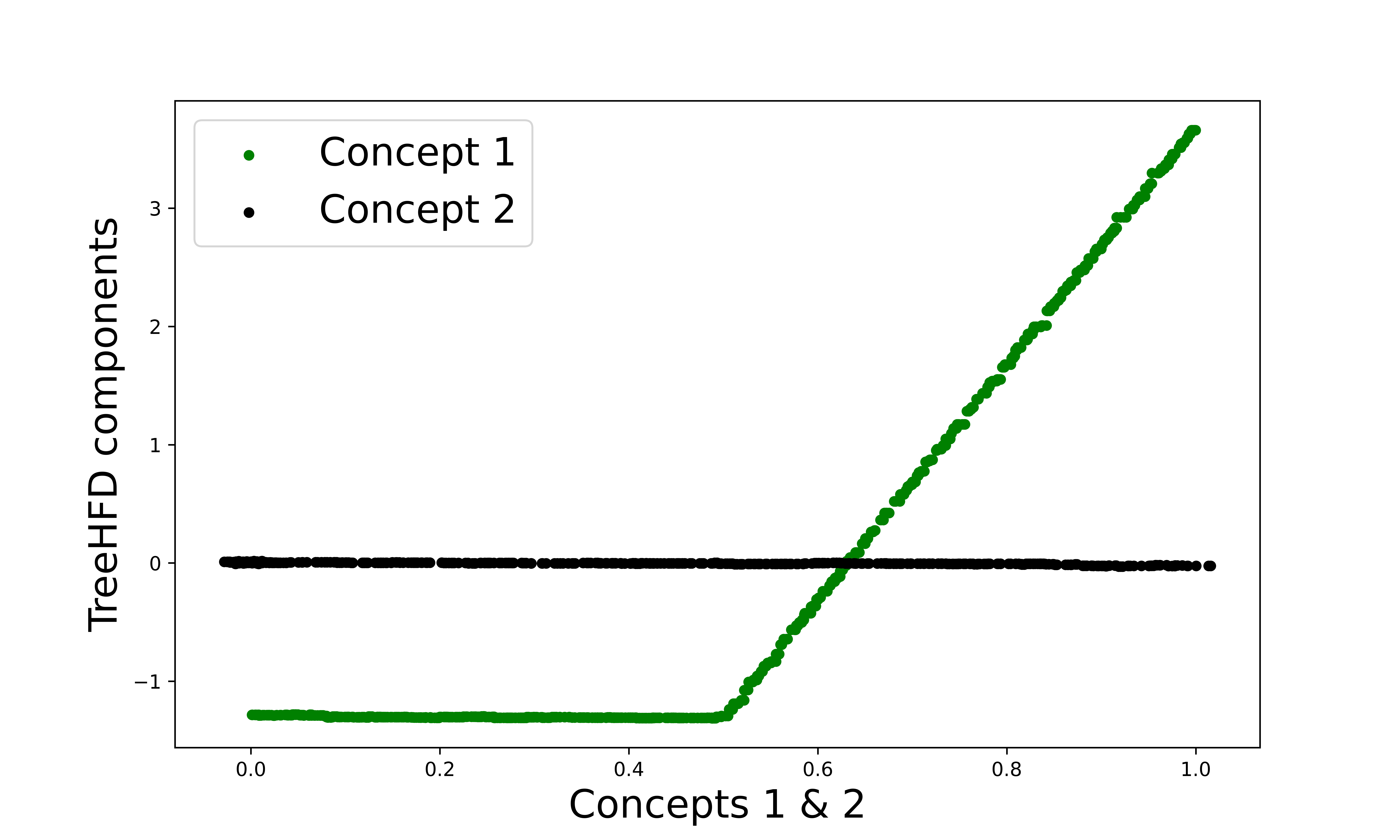}
		\caption{On the left panel, linear approximation of the logit $f_1$ with the concept score $g^{(1)}$. On the right panel, components generated by TreeHFD to fit $f_1$ with $g^{(1)}$ and $g^{(2)}$.} \label{fig:simulated_case}
	\end{center}
\end{figure}

\section{Additional details for HCBM algorithm} \label{app:hcbm}

\paragraph{Monotone constraints.}
We provide below the detailed procedure to detect increasing and decreasing components in HCBM.
Once the initial XGBoost and TreeHFD algorithms are fit (without monotone constraints), we compute the relative cumulated increase of each main effect $\smash{f_{n, \ell}^{(j)}}$ for $\smash{j \in \{1, \hdots, K\}}$ of the obtained decompositions.
More precisely, for a given component $j$, the values $\smash{\{f^{(j)}_{n, \ell}(g_n^{(j)}(\bX_i))}\}_{i=1}^n$ are sorted with respect to increasing values of the concept scores $\smash{\{g_n^{(j)}(\bX_i)\}_{i=1}^n}$, and the difference between two consecutive values is computed. Then, the positive differences are added together, and divided by the sum of all absolute differences, to obtain the relative increase of the considered component. When this ratio is higher than a threshold $r \in (1/2, 1)$, the component is set as an increasing function. On the other hand, when the ratio is below $1 - r$, the component is considered as a decreasing one. If the ratio is between $1 - r$ and $r$, no constraint is enforced.

The default value is $r = 0.8$, and HCBM are quite insensitive to $r$ values.
Indeed, $r$ could be set to $0.7$ or $0.9$ with almost no difference in the output results. 
When $r$ is set to $0.5$, all components of the Hoeffding decomposition are forced to be monotone (decreasing or increasing), which can hurt performance, since there are often several concepts with clear non-monotone behavior. On the other hand, when $r = 1$, no monotone constraints are enforced, since a component already needs to be monotone in the initial unconstrained XGBoost to be detected as monotone. Overall, setting $r$ away from $0.5$ and $1$ is critical, but the exact intermediate value has a small impact on HCBM outputs.

\paragraph{HCBM interactions.}
Including interaction terms does not change HCBM accuracy and implicant sizes for all datasets of the experimental section, because of the small NEC values, which systematically discard interaction terms. It is quite expected to have 10 main effects with a higher impact than interactions when there are hundreds of input concepts.

\paragraph{Normalization of VLM scores.}
The min-max normalization of VLM scores (CLIP or SigLIP) has no impact on XGBoost and TreeHFD involved in HCBM, since tree ensembles are invariant with respect to monotone transformations of the input variables. Indeed, node splits are orthogonal to each axis, and the same splits will be optimal regardless of the variable values, provided that the data points have the same order with respect to each variable.
The goal of the min-max normalization is to provide an easy interpretation of the VLM scores with respect to the training data.

\paragraph{Out-of-distribution inference.}
By definition of the min-max normalization, if a new query image $\bX$ has concept scores with smaller or larger values than all images of the training data of HCBM, these normalized scores will be outside of $[0, 1]$, that is $-0.1$ or $1.05$ for example. Tree ensembles are constant in all directions outside the training support, since all tree leaves contain at least one training observation and predictions are constant within each leaf. Therefore, HCBM prediction at $\bX$ will be exactly identical to the case where the concept scores outside $[0, 1]$ are clipped to $0$ and $1$.
Consequently, the algorithm for prime implicants described in the following paragraph optimizes component values over $[0, 1]$.

\paragraph{Prime implicants.}
We provide additional details about the procedure used to find prime implicants to explain HCBM predictions, which has a quasilinear computational complexity with respect to $K \times L$.

Recall that $\pi$ is the permutation  of $\{1, \hdots, L \}$ to order the label logits such that $f_{n, \pi(1)}(g_n(\bX)) > f_{n, \pi(2)}(g_n(\bX)) > \hdots > f_{n, \pi(L)}(g_n(\bX))$, for a new image $\bX$, with $\pi(1)$ the final prediction, and that an implicant $S(\bX) \subset \{1, \hdots, K \}$ for prediction at $\bX$ satisfies for all $\ell \in \{1, \hdots, L\} \setminus \pi(1)$,
\begin{align*}
    \delta_{\ell}^{(\emptyset)} + \sum_{j \in S(\bX)} \delta_{\ell}^{(j)}(g_n^{(j)}(\bX))
    + \sum_{j \notin S(\bX)} \min_{z \in [0,1]} \delta_{\ell}^{(j)}(z)
    > 0, \quad \mathrm{where} \quad \delta_{\ell}^{(j)} = f_{n, \pi(1)}^{(j)} - f_{n, \ell}^{(j)}.
\end{align*}

To find a prime implicant, we first rank the concepts by decreasing values of the quantity 
\begin{align*}
    h_{2}^{(j)}(\bX) = \delta_{\pi(2)}^{(j)}(g_n^{(j)}(\bX)) - \min_{z \in [0,1]} \delta_{\pi(2)}^{(j)}(z),
\end{align*}
where this ranking is defined by the permutation $\sigma_{2}$. Then, we find the smallest integer $s_2$ such that
\begin{align*}
    \sum_{j = 1}^{s_2} h_{2}^{(\sigma_2(j))}(\bX) > - \delta_{\pi(2)}^{(\emptyset)} - \sum_{j = 1}^K \min_{z \in [0,1]} \delta_{\pi(2)}^{(j)}(z),
\end{align*}
and the concepts $S_2(\bX) = \{\sigma_2(1), \hdots, \sigma_2(s_2)\}$ are sufficient to explain that label $\pi(1)$ is predicted over $\pi(2)$. 
Next, for $\pi(3)$, $h_{3}^{(j)}(\bX)$ is defined similarly to $h_{2}^{(j)}(\bX)$, and the values are ranked for all concepts that are not in $S_2$ by decreasing values, defined with the mapping $\sigma_3$. Then, we find the smallest integer $s_3$, such that
\begin{align*}
    \sum_{j = 1}^{s_3} h_{3}^{(\sigma_3(j))}(\bX) + \sum_{j \in S_2(\bX)} h_{3}^{(j)}(\bX)
    > - \delta_{\pi(3)}^{(\emptyset)} - \sum_{j = 1}^K \min_{z \in [0,1]} \delta_{\pi(3)}^{(j)}(z),
\end{align*}
and similarly, $S_3(\bX) = \{\sigma_3(1), \hdots, \sigma_3(s_3)\}$.
Finally, the same approach is applied sequentially for $\ell \in \{4, \hdots, L\}$, where
$h_{\ell}^{(j)}(\bX)$ are ranked for concepts that are not in $\cup_{k=2}^{\ell-1} S_{k}(\bX)$, and $s_{\ell}$ is the smallest integer satisfying 
\begin{align*}
    \sum_{j = 1}^{s_{\ell}} h_{{\ell}}^{(\sigma_{\ell}(j))}(\bX) + \sum_{j \in \cup_{k=2}^{\ell-1} S_{k}(\bX)} h_{{\ell}}^{(j)}(\bX)
    > - \delta_{\pi(\ell)}^{(\emptyset)} - \sum_{j = 1}^K \min_{z \in [0,1]} \delta_{\pi(\ell)}^{(j)}(z),
\end{align*}
and we get $S_{\ell}(\bX) = \{\sigma_{\ell}(1), \hdots, \sigma_{\ell}(s_{\ell})\}$, and $S(\bX) = \cup_{\ell=2}^{L} S_{\ell}(\bX)$.

At this stage $S(\bX)$ is an implicant, but may not be a prime implicant, since a strict subset of $S(\bX)$ may still be an implicant, by construction of the above greedy procedure. We turn $S(\bX)$ into a prime implicant, by trying to remove one concept at a time from $S(\bX)$, and checking whether $S(\bX)$ remains an implicant without the concept by comparing updated logits. If not, the concept is required, otherwise we remove it and proceed to the next concept. The concepts are tested one by one in order of decreasing logits and increasing $\smash{h_{\ell}^{(j)}(\bX)}$ values for each logit.


\section{Introduction to the Hoeffding decomposition and TreeHFD} \label{app:treehfd}

This section provides an introduction to the Hoeffding decomposition and the TreeHFD algorithm.

\subsection{Hoeffding decomposition}

The Hoeffding decomposition was originally introduced by \citet{hoeffding1948} for independent input variables, as mentioned in the introduction. In this case, all the components of the decomposition are orthogonal.
The extension to the dependent case was later achieved by \citet{stone1994use} and \citet{hooker2007generalized}, using hierarchical orthogonality constraints. These constraints imply that two components of the decomposition are orthogonal when one of the associated subsets of variables is included in the other one. Consequently, main effects and interactions are clearly separated and the obtained decomposition is highly sparse, as opposed to other additive models.

\citet{chastaing2012hoeffding} provide a clear presentation of the Hoeffding decomposition, and use it to build variable importance measures. We report Theorem $1$ from \citet{chastaing2012hoeffding} below, which is a reformulation of results from \citet{stone1994use} and \citet{hooker2007generalized}. Notice that this theorem holds for input variables with specific distributions of unbounded support, under the Condition (C2), which is not useful in our case where inputs are bounded cosine similarities. Therefore, we consider that each variable distribution has a compact support to state the Hoeffding decomposition.

\begin{theorem}[Hoeffding decomposition]
    Let $\eta$ be a square integrable function, and $\smash{\bX=(X^{(1)}, \hdots, X^{(p)})}$ be a real random vector, where $p$ is a positive integer and each variable takes values in a compact set. If the distribution of $\bX$ is bounded away from $0$ and infinity on the product support, then there is a unique set of functions $\{ \eta^{(J)} \}_{J \subset \{1, \hdots, p\}}$ such that
    \begin{align*}
        \eta(\bX) = \sum_{J \subset \{1, \hdots, p\}} \eta^{(J)}(\bX^{(J)}),
    \end{align*}
    and for all $J \subset \{1, \hdots, p\}$ and $I \subsetneq J$, $\E[\eta^{(J)}(\bX^{(J)}) \mid \bX^{(I)}] = 0$.
\end{theorem}

The Hoeffding decomposition is a theoretical result, and its computation in practice is a difficult and open problem. When inputs are independent, closed formulas of the decomposition exist, but not when inputs are dependent, as it is the case in standard learning settings. When the input distribution is known, \citet{lengerich2020purifying} build on tree ensembles to estimate the Hoeffding decomposition. However in practice, only a data sample is available, and density estimation is highly challenging even in moderate dimension. The TreeHFD algorithm \citep{benard2025tree} circumvents this limitation and estimates the Hoeffding decomposition using tree ensembles and a data sample.

\subsection{TreeHFD}

The TreeHFD algorithm is introduced in \citet{benard2025tree} to estimate the Hoeffding decomposition (HFD) of tree ensembles. 

The core principle of TreeHFD is to discretize the HFD using the tree partitions of a tree ensemble (e.g., gradient-boosting models, random forests...). 
More precisely, the Cartesian tree partition is built for each tree from the tree splits, and defined as the grid where all splits involving a given variable define the partition along the corresponding direction, regardless of the tree structure.

Hence, the main result is the adaptation of the Hoeffding decomposition for the piecewise-constant functions involved in tree ensembles. If the hierarchical orthogonality constraints are discretized over the Cartesian tree partition of a given tree, there is a unique decomposition of this tree, where components are all piecewise constant on the Cartesian tree partition---see Theorem $2$ in \citet{benard2025tree}. Theorem $6$ shows that the approximation induced by the discretization is small, provided that the tree ensemble is accurate, and the input distribution does not vary too much over the cells of the tree partitions.

Then, the components of the decomposition are parametrized with a coefficient for each constant part, and TreeHFD solves a quadratic problem for each tree, to find the optimal coefficients which satisfy the discretized orthogonality constraints and the decomposition of the tree predictor. In practice, the cost function is built using the input data sample and the tree partition. The final decomposition is obtained by adding the decompositions of all trees for each variable subset.
Theorem $7$ shows that TreeHFD converges towards the Hoeffding decomposition with discretized orthogonality constraints when the size of the training data grows.

Experiments with simulated and real data show that TreeHFD accurately estimates the decomposition of XGBoost models using main effects and second-order interactions: the decomposition residuals are about $1\%$ of the initial boosting model variance, and the orthogonality constraints are well approximated, since correlation coefficients between interactions and the associated main effects are below $5\%$ in all cases.

\section{Additional experimental settings} \label{app:xp_settings}

\subsection{Datasets and concepts}

We provide all details about the tested datasets: CIFAR \citep{krizhevsky2009learning}, EuroSAT \citep{helber2019eurosat}, RESISC \citep{cheng2017remote}, Waterbirds~\citep{Sagawa2020Distributionally}, xView \citep{lam2018xview}, CUB \citep{wah2011caltech}, and ImageNet \citep{russakovsky2015imagenet}.
All concept lists can be found in the GitHub repository of the corresponding articles.

\textbf{CIFAR-10} is a famous dataset used in computer-vision benchmarks, with $60 000$ images and $10$ classes of vehicles and animals, split into training and testing sets ($83 \% / 17 \%$).  We use the original list of concepts from PCBM \citep{yuksekgonul2022post}, built with ConceptNet, which contains $175$ concepts.
\textbf{CIFAR-100} has the same number of images but $100$ classes, grouped in $20$ superclasses of vehicles, animals, objects, and plants, and is also split into training and testing datasets. We use the list of $440$ concepts from PCBM.

\textbf{EuroSAT} is a dataset of $27 000$ satellite images with $10$ classes of land use and land cover elements. We use the list of $100$ concepts generated in \citet{nguyen2025interpretable}. Since we found no existing train/test split, we use $5$-fold cross-validation to assess model performance.

\textbf{RESISC} is a dataset of $31 500$ satellite images with $45$ classes of aircraft, boats, buildings, and land cover elements. We use the train/test split and the list of selected concepts from the established CBM Labo \citep{yang2023language}. This list contains $50$ concepts for each class, for a total number of $2250$ concepts.

\textbf{Waterbirds} is a dataset of $4 795$ training images, $1 199$ validation images, and $5 794$ test images of birds. This dataset is constructed by superimposing bird crops from CUB~\citep{wah2011caltech} onto background scenes from Places~\citep{zhou2017places}. Waterbirds is designed to evaluate a model reliance on spurious correlations: during training, landbirds and waterbirds are predominantly paired with their typical land and water backgrounds respectively. However, these correlations are removed in the test set, creating a distribution shift which penalizes models relying on environmental cues.
We use the concept list from \citet{enouen2026debugging} with 533 concepts.

For \textbf{xView} dataset, most details are provided in the article. Since the validation data of xView has no annotated bounding boxes, we use $5$-fold cross-validation with the training data to assess model performance, similarly to EuroSAT. We generate a list of $73$ concepts---see Appendix \ref{app:xview_concepts}.

\textbf{CUB} dataset contains $11 788$ images of $200$ bird species. We use the train/test split of equal size provided with the data (\url{https://data.caltech.edu/records/65de6-vp158}), and the list of $723$ concepts from VLG-CBM \citep{srivastava2024vlg}.

\textbf{ImageNet} is the famous dataset of $1.2$ million images and $1000$ classes, widely used in computer vision benchmarks. We use the list of $4751$ concepts provided in VLG-CBM \citep{srivastava2024vlg}, and the standard train/test split.

\subsection{Algorithm settings}

\paragraph{VLM.}
For CLIP \citep{radford2021learning}, we use the version ``clip-vit-large-patch14'', downloaded at \url{https://huggingface.co/openai/clip-vit-large-patch14}. \\
For SigLIP \citep{zhai2023sigmoid}, we use the version ``siglip-large-patch16-256'', downloaded at \url{https://huggingface.co/google/siglip-large-patch16-256}.

\paragraph{Linear CBM learning.}
To fit the linear aggregation of CBM, we use the ``saga'' solver provided in \texttt{scikit-learn} \citep{pedregosa2011scikit} to solve the multiclass logistic regression problem, with the elastic-net penalization originally introduced in PCBM \citep{yuksekgonul2022post}, that is an $L1$-ratio of $0.99$, with a penalization strength set to meet the NEC input.
For ImageNet, we use the ``Adam'' stochastic gradient descent from \texttt{PyTorch} to handle the large size of this dataset.

\paragraph{Experiment repetitions.}
As explained in the article, experiment metrics are averaged over $5$ repetitions to make standard deviations small. At each repetition, data splits are drawn independently for cross-validation, as well as any randomness in algorithm training. For information leakage experiments, the noise of injected leakage is also drawn independently at each repetition.

\paragraph{HCBM for ImageNet.}
ImageNet contains $1.2$ million images and $1000$ different classes, as explained above. However, HCBM can be easily applied to ImageNet by downsampling $10 000$ images of the class $0$ (all classes except the targeted one) in Step $4$ of Algorithm \ref{algo_hcbm}. Tables \ref{table:xp_hcbm}, \ref{table:xp_clip_large_classes}, and \ref{table:xp_siglip_large_classes} show that HCBM strongly outperform sparse linear CBM for ImageNet.

\paragraph{Explainable Boosting Machine (EBM).}
EBM is probably the most widely used algorithm to build non-linear additive models \citep{nori2019interpretml}. Hence, EBM can replace XGBoost and TreeHFD in HCBM framework to provide an additive and non-linear aggregation of concept scores. Similarly to XGBoost and TreeHFD, EBM is fit with all default parameters in all experiments of the article, except for RESISC-45. Indeed, EBM uses a very large number of boosting rounds by default of $50 000$, combined with a low learning rate, and  train on one feature at a time in round-robin fashion. In the case of RESISC-45 with $2250$ input concepts, the training cost of EBM becomes prohibitive because of this large number of boosting rounds and number of concepts. More precisely, a single fit of HCBM on RESISC-45 takes about $10$ minutes using the CPU of a machine described in the paragraph below, and around $1$ hour for linear CBM, while it was estimated at $200$ hours for EBM with default settings after $12$ hours of run (for $1$ of the $5$ repetitions). 
Therefore, we reduced the number of boosting rounds to $1000$ and increased the learning rate accordingly ($\eta = 0.75$), and parallelized runs over $32$ cores to avoid jobs exceeding a week of compute. Notice that we picked $\eta = 0.75$ to keep the product of $\eta$ and the number of boosting rounds constant, but we also tried several other values for $\eta$, that is $0.1$, $0.3$, $1.0$ (using CLIP, without leakage, and with one repetition to limit the computational cost), but $0.75$ leads to a better accuracy performance.
We use the same settings for CUB-200, but we do not run HCBM with EBM for CIFAR-100 and ImageNet-1000, because these datasets have larger numbers of classes and images, making the scaling limitations of EBM even worse than in the case of RESISC-45.

\paragraph{Information leakage experiments.}
We provide additional details for the experiments with injected leakage.
As explained in the article, we add irrelevant concepts with injected leakage for each dataset, built from the scores of the original concepts and noise.
More precisely, for each concept of the initial list, we add another concept, where its CLIP score (or SigLIP score) is set as the CLIP score of the original concept, truncated for the highest $20\%$ values (clipped to the quantile $0.8$), and a standard Gaussian noise is added. Hence, such new concepts suffer from strong interconcept leakage by construction, but some relevant information is removed for high concept activations, and noise is added.
Then, linear CBM and HCBM are learned with these additional concepts, and we compute the proportion of concepts selected by each method that contain injected leakage, with the results reported in the bottom part of Tables \ref{table:xp_hcbm}, \ref{table:xp_clip}, and \ref{table:xp_clip_large_classes}, for CLIP. While linear CBM select about half of irrelevant concepts for all datasets, HCBM select few irrelevant concepts.

Tables \ref{table:xp_siglip} and \ref{table:xp_siglip_large_classes} report the results of the same experiments using SigLIP.
We also run the same case with various values of the noise variance using CLIP, with the results displayed in Table \ref{table:clip_leakage_noise_var}.
Finally, we use a square transformation of the original CLIP scores to build the concepts with injected leakage, and also add a standard Gaussian noise. Results are reported in Table \ref{table:clip_leakage_square}.
All these additional experiments lead to the same conclusion: HCBM are much more robust than linear CBM to interconcept leakage.

\paragraph{Validity of the assumptions for the theoretical properties.}
The theoretical properties of HCBM stated in Theorems \ref{thm:hcbm}, \ref{thm:hfd_irrelevant}, and \ref{thm:concept_selection} rely on the assumption that the concept scores have a positive joint density. However, concept scores are linear combinations of vision embeddings through the cosine similarity, and when the number of concepts $K$ is larger than the dimension of the VLM embeddings $d_L$, concept scores have a linear dependence, and thus, cannot have a strictly positive density.
Given that $d_L = 768$ for CLIP, the assumptions of the theoretical properties are violated for RESISC and ImageNet.
However, the proportions of selected concepts with injected leakage are much smaller for HCBM than for the competitors (sparse linear CBM and EBM) for these two datasets. This shows that HCBM (with TreeHFD) are still strongly robust to interconcept leakage in practice, even when the assumptions of the theoretical properties are not satisfied.
Although it is probably possible to weaken the required assumptions of the theoretical properties, it is out of the article scope, since it would require non-trivial extensions of the Hoeffding decomposition for dependent inputs \citep{chastaing2012hoeffding} and of the theoretical analysis of TreeHFD.
Finally, we highlight that TreeHFD is applied to each logit separately. Therefore, the theoretical properties can hold when using several hundreds of concepts for each class, and the concept sets can be different for each class. It is thus possible to use a very large initial concept set, while having the theoretical guarantees hold.

\paragraph{Compute and software resources.}
All experiments were conducted with a standard Slurm HPC cluster, made of machines with $32$ cores at $2.8$ GHz and $384$ GB of RAM.
HCBM are based on XGBoost and TreeHFD. Consequently, \texttt{xgboost} and \texttt{treehfd} software are used in the experiments, in accordance with their Apache License 2.0.

\section{Proofs and additional theorems} \label{app:proof}

\subsection{Proofs of Theorems \ref{thm:hcbm} and \ref{thm:hfd_irrelevant}}

\begin{proof}[Proof of Theorem \ref{thm:hcbm}]
    The core of the proof is to apply Theorem $1$ from \citet{chastaing2012hoeffding} to the function $f_{\ell}$ with respect to the input variables $g^{(1)}(\bX), \hdots, g^{(K)}(\bX)$,  for each label $\ell \in \{1, \hdots, L\}$.

    By assumption, the functions $f_{1}, \hdots, f_{L}$ are square integrable, each concept score takes values in a compact set, and the distribution of $g(\bX)$ is bounded away from $0$ and infinity on its product support. Therefore, Assumption (C.2) of \citet{chastaing2012hoeffding} is satisfied, and we can apply Theorem $1$ from \citet{chastaing2012hoeffding} to $f_{\ell}$, which states the existence of a unique set of functions $\smash{\{f_{\ell}^{(J)}\}_{J \subset \{1, \hdots, K \}}}$ for each label $\ell \in \{1,\hdots,L\}$, such that
    \begin{align*}
        f_{\ell}(g(\bX)) = \sum_{J \subset \{1, \hdots, K \}} f_{\ell}^{(J)}(g^{(J)}(\bX)),
    \end{align*}
    and for $I \subsetneq J \subset \{1, \hdots, K \}$, we have $\smash{\E[f_{\ell}^{(J)}(g^{(J)}(\bX)) \mid g^{(I)}(\bX)] = 0}$.
\end{proof}

\begin{proof}[Proof of Theorem \ref{thm:hfd_irrelevant}]
    Let the assumptions of Theorem \ref{thm:hcbm} be satisfied: the functions $f_{1}, \hdots, f_{L}$ are square integrable, each concept score takes value in a compact set, and the distribution of $g(\bX)$ is bounded away from $0$ and infinity on its product support.
    The subset $J_{\ell}^{\star}$ is defined as the smallest subset such that $\P(Y = \ell \mid g(\bX)) = \P(Y = \ell \mid g^{(J_{\ell}^{\star})}(\bX))$. Since the density of $g(\bX)$ is strictly positive on its product support, concepts are not redundant: no concept score can be written as a deterministic function of other concept scores. This implies that $J_{\ell}^{\star}$ is unique.
    
    The second part of the proof relies on the uniqueness of the Hoeffding functional decomposition.
    By assumption, the Hoeffding decomposition of $f_{\ell}(g(\bX))$ is written as
    \begin{align*}
        f_{\ell}(g(\bX)) = \sum_{J \subset \{1, \hdots, K\}} f_{\ell}^{(J)}(g^{(J)}(\bX)).
    \end{align*}
    By definition of $J_{\ell}^{\star}$, we have $\P(Y = \ell \mid g(\bX)) = \P(Y = \ell \mid g^{(J_{\ell}^{\star})}(\bX))$, and consequently
    \begin{align*}
        f_{\ell}(g(\bX)) = \mathrm{logit}(\P(Y = \ell \mid g^{(J_{\ell}^{\star})}(\bX))).
    \end{align*}
    The Hoeffding decomposition of the function $\mathrm{logit}(\P(Y = \ell \mid g^{(J_{\ell}^{\star})}(\bX)))$ only involves functions with the components of $g^{(J_{\ell}^{\star})}$ by construction.
    Since the Hoeffding decomposition is unique, we conclude that for all $J \not \subset J_{\ell}^{\star}$, $f_{\ell}^{(J)}(g^{(J)}(\bX)) = 0 \quad \mathrm{a.s.}$
\end{proof}

\subsection{Finite-sample theorem for concept selection}

Theorem \ref{thm:concept_selection} below shows that the variances of TreeHFD components are smaller for concept subsets containing irrelevant concepts (even if they suffer from interconcept leakage) than for relevant concept subsets, provided that the theoretical effects of these relevant concepts on a logit are not too small. This ensures that only relevant concepts will be selected by HCBM through Equation (\ref{eq:cimp}), provided that the NEC is not too large. 

More precisely, Theorem \ref{thm:concept_selection}-(a) provides an upper bound on the variance of components with irrelevant concepts (which is null for the theoretical decomposition). This bound is a consequence of TreeHFD properties, derived from Theorem $6$ of TreeHFD article. The bound can be small if the tree ensemble accuracy is high and the tree partitions are fine grained. Then,  Theorem \ref{thm:concept_selection}-(b) is a consequence of Theorem \ref{thm:concept_selection}-(a), and shows that if the theoretical effects of relevant concept subsets on a logit are not too small, and the NEC is not too large (which is the case in our experiments given the impact of NEC on accuracy), then HCBM only select relevant concepts.

\begin{theorem} \label{thm:concept_selection}
    Let $T_n = \sum_b T_{n,b}$ be a tree ensemble fit with the dataset $\Dn$ for the binary classification of label $\ell \in \{1, \hdots, L \}$ versus all other labels with concept scores $g_n(\bX)$ as inputs, and let the assumptions of Theorem \ref{thm:hfd_irrelevant} be satisfied. Let HCBM be fit with TreeHFD loss built with variable subsets of all sizes involved, and with the empirical density of each tree (computed over the Cartesian tree partition with $\Dn$) truncated from above and below by strictly positive constants. Let $\smash{\{f_{n, \ell}^{ (J)} \}_J}$ be the estimates of $T_n$ decomposition output by TreeHFD, and let  $\smash{\{f_{\ell}^{(J)} \}_J}$ be the HFD of the theoretical logit $f_{\ell}(g(\bX))$.
    
    (a) Then, for $J_0 \not \subset J_{\ell}^{\star}$, we have
    \begin{align*}
        \V[f_{n, \ell}^{ (J_0)}(g_n^{(J_0)}(\bX))] \leq & K_1 \E[(f_{\ell}(g(\bX)) - T_n(g_n(\bX)) )^2] \\ &+ K_2 \E\big[\Delta_{T_n}^2 \Big(\sum_{b} \sqrt{\E[T_{n, b}(g_n(\bX))^2]}\Big)^2\big] \overset{\mathrm{(def)}}{=} \Gamma_n,
    \end{align*}
    where $\V[\cdot]$ is the variance operator with respect to all random variables, $\Delta_{T_n}$ is the maximum difference between the theoretical and empirical densities over all trees, and $K_1, K_2$ are fixed constants.

    (b) For $J \subset  J_{\ell}^{\star}$, if $\V[f_{\ell}^{(J)}(g^{(J)}(\bX))] > (1 + \sqrt{3})^2 \Gamma_n$, then we have for all $J_0 \not \subset J_{\ell}^{\star}$,
    \begin{align*}
        \V[f_{n, \ell}^{ (J)}(g_n^{(J)}(\bX))] > \V[f_{n, \ell}^{ (J_0)}(g_n^{(J_0)}(\bX))].
    \end{align*}
\end{theorem}

\begin{proof}[Proof of Theorem \ref{thm:concept_selection}]
    We apply Lemma \ref{lemma:hfd_bound} for $J_0 \not \subset J_{\ell}^{\star}$ to obtain that 
    \begin{align*}
        \E[(f_{n, \ell}^{ (J_0)}(g_n^{(J_0)}(\bX)) - f_{\ell}^{(J_0)}(g^{(J_0)}(\bX)))^2] \leq \Gamma_n.
    \end{align*}
    Next, $f_{\ell}^{(J_0)}(g^{(J_0)}(\bX))$ is null by definition of $J_{\ell}^{\star}$, according to Theorem \ref{thm:hfd_irrelevant}, and by definition of the variance operator, we finally have
    \begin{align*}
        \V[f_{n, \ell}^{ (J_0)}(g_n^{(J_0)}(\bX))] \leq \E[f_{n, \ell}^{ (J_0)}(g_n^{(J_0)}(\bX))^2] \leq \Gamma_n,
    \end{align*}
    which proves (a).

    For (b), we first expand the variance term as follows,
    \begin{align*}
        \V[f_{n, \ell}^{ (J)}(g_n^{(J)}(\bX))] =&  \E[f_{n, \ell}^{ (J)}(g_n^{(J)}(\bX))^2] - \E[f_{n, \ell}^{ (J)}(g_n^{(J)}(\bX))]^2 \\
        =& \E[(f_{\ell}^{(J)}(g^{(J)}(\bX)) + (f_{n, \ell}^{ (J)}(g_n^{(J)}(\bX)) - f_{\ell}^{(J)}(g^{(J)}(\bX))) )^2] \\
         &- \E[f_{n, \ell}^{ (J)}(g_n^{(J)}(\bX))]^2 \\
        =& \V[f_{\ell}^{(J)}(g^{(J)}(\bX))] + 2 \E[f_{\ell}^{(J)}(g^{(J)}(\bX))(f_{n, \ell}^{ (J)}(g_n^{(J)}(\bX)) - f_{\ell}^{(J)}(g^{(J)}(\bX)))] \\ &+ \E[(f_{n, \ell}^{ (J)}(g_n^{(J)}(\bX)) - f_{\ell}^{(J)}(g^{(J)}(\bX)))^2] - \E[f_{n, \ell}^{ (J)}(g_n^{(J)}(\bX))]^2.
    \end{align*}
    We apply Lemma \ref{lemma:hfd_bound} again for $J \subset  J_{\ell}^{\star}$, to get that $\E[(f_{n, \ell}^{ (J)}(g_n^{(J)}(\bX)) - f_{\ell}^{(J)}(g^{(J)}(\bX)))^2] \leq \Gamma_n$.
    For the last term of the above equation, we use Jensen's inequality and $\E[f_{\ell}^{(J)}(g^{(J)}(\bX))] = 0$ to state
    \begin{align*}
        \E[f_{n, \ell}^{ (J)}(g_n^{(J)}(\bX))]^2 = \E[f_{n, \ell}^{ (J)}(g_n^{(J)}(\bX)) -    f_{\ell}^{(J)}(g^{(J)}(\bX))]^2 \leq \Gamma_n.
    \end{align*}
    Then, we use the Cauchy-Schwarz inequality to bound the cross term of the above equation using $\Gamma_n$, and obtain
    \begin{align*}
        \V[f_{n, \ell}^{ (J)}(g_n^{(J)}(\bX))] \geq \V[f_{\ell}^{(J)}(g^{(J)}(\bX))] - 2 \sqrt{\Gamma_n \V[f_{\ell}^{(J)}(g^{(J)}(\bX))]} - \Gamma_n.
    \end{align*}

    From (a), we have that $\V[f_{n, \ell}^{ (J_0)}(g_n^{(J_0)}(\bX))] \leq \Gamma_n$.
    It is clear from the last two inequalities that if $\V[f_{\ell}^{(J)}(g^{(J)}(\bX))] > (1 + \sqrt{3})^2 \Gamma_n$, then $ \V[f_{n, \ell}^{ (J)}(g_n^{(J)}(\bX))] > \Gamma_n \geq \V[f_{n, \ell}^{ (J_0)}(g_n^{(J_0)}(\bX))]$, which proves (b).
\end{proof}

\begin{lemma} \label{lemma:hfd_bound}
    If the assumptions of Theorem \ref{thm:concept_selection} are satisfied, there exist two constants $K_1$ and $K_2$ such that, for all $J \subset \{1, \hdots, K\}$,
    \begin{align*}
        \E[(f_{n, \ell}^{(J)}(g_n^{(J)}(\bX)) - f_{\ell}^{(J)}(g^{(J)}(\bX)))^2] \leq & K_1 \E[(f_{\ell}(g(\bX)) - T_n(g_n(\bX)) )^2] \\ &+ K_2 \E\big[\Delta_{T_n}^2 \Big(\sum_{b} \sqrt{\E[T_{n, b}(g_n(\bX))^2]}\Big)^2\big].
    \end{align*}
\end{lemma}

\begin{proof}[Proof of Lemma \ref{lemma:hfd_bound}]
    Theorem $6$ from \citet{benard2025tree} states the result for the HFD of the fitted tree ensemble $T_n$ with respect to the theoretical input distribution discretized over the Cartesian tree partition, conditional on the training data of the trees $\Dn$.
    
    This lemma extends the result for the empirical estimates output by TreeHFD using the following key argument: TreeHFD minimizes a loss function for each tree in practice, and the minimum of this loss function is exactly the Hoeffding decomposition of each tree with respect to the empirical distribution computed over the Cartesian tree partition. By assumption, these empirical densities are bounded, which enforces that the Hoeffding decomposition of each tree with respect to the empirical density exists and is unique, according to Theorem $1$ from \citet{chastaing2012hoeffding}. 
    
    Then, the proof of Theorem $6$ of TreeHFD article relies on Lemmas $4$ and $5$, which can be directly extended from the key argument above, and using the empirical tree density instead of the theoretical discretized distribution. Such a modification is straightforward since Lemmas $4$ and $5$, and Theorem $6$ are stated conditional on $\Dn$, and the empirical densities are estimated with $\Dn$. For Lemma $5$, we also use the updated $\Delta_{T_n}$, which quantifies the gap between the input distribution of $\bX$ and empirical densities.
    Also notice that we can consider the score inputs $g(\bX)$ to adapt these proofs, since $g_n(\bX)$ is a deterministic affine transform of $g(\bX)$, conditional on $\Dn$, by definition of the empirical min-max normalization of the VLM scores. Consequently, it is strictly equivalent to consider the tree ensemble $T'_n$ with the appropriate empirical scaling of the tree splits, such that $T'_n(g(\bX)) = T_n(g_n(\bX))$. The HFD of a tree with respect to the discretized distribution is also invariant to affine transforms.

    Finally, we take the expectation with respect to the training data $\Dn$ on both sides of the inequality given by this extension of Theorem $6$ to obtain the targeted result.
\end{proof}

\section{Additional experiments}  \label{app:xp_results}
\FloatBarrier

\subsection{Additional datasets and SigLIP}

We provide additional experiment results in Tables \ref{table:xp_clip}, \ref{table:xp_clip_large_classes}, \ref{table:xp_siglip}, and \ref{table:xp_siglip_large_classes}, especially for CIFAR datasets, and using SigLIP instead of CLIP to compute image and text embeddings for all datasets. Notice that Table \ref{table:xp_hcbm} in the article is a subset of Tables \ref{table:xp_clip} and \ref{table:xp_clip_large_classes}, to get a good span of the number of classes fitting in a single table.
These additional results support the conclusions detailed in the article above.

As explained in the article, XGBoost is always run with all default parameters to show the high performance of HCBM without any parameter tuning. However, for some datasets with a large number of classes such as CIFAR-100 and CUB-200, the obtained accuracy is slightly below linear CBM. Therefore, we simply tune tree depth in XGBoost to improve HCBM accuracy, as reported in performance tables. For CIFAR-100, the accuracy is then very close between HCBM and linear CBM, and HCBM are much more robust to interconcept leakage. For CUB-200, HCBM are also strongly robust to leakage. However, the accuracy is slightly lower, which can be explained by the small number of images per class ($30$), a setting where sparse linear models are highly efficient. Notice that for SigLIP the accuracy of HCBM for CUB-200 is better or close to linear CBM depending on NEC values.

\begin{table}
\caption{Using CLIP, classification accuracy of HCBM with TreeHFD versus sparse linear CBM and HCBM with EBM (with standard deviations $\leq 10^{-3}$), in the upper part.
Below, proportion ($\%$) of irrelevant concepts with injected leakage selected by HCBM versus sparse linear CBM and EBM.}
\centering
\setlength{\tabcolsep}{1pt}
\resizebox{\textwidth}{!}{
\begin{tabular}{cccccccccccccccc}
  \hline
  \hline
  NEC & \multicolumn{3}{c}{Waterbirds-2} & \multicolumn{3}{c}{xView-3} & \multicolumn{3}{c}{CIFAR-10} & \multicolumn{3}{c}{EuroSAT-10} & \multicolumn{3}{c}{RESISC-45} \\
   & Linear & \multicolumn{2}{c}{HCBM} & Linear & \multicolumn{2}{c}{HCBM} & Linear & \multicolumn{2}{c}{HCBM} & Linear & \multicolumn{2}{c}{HCBM} & Linear & \multicolumn{2}{c}{HCBM}  \\
   & & EBM & \small{TreeHFD} & & EBM & \small{TreeHFD} & & EBM & \small{TreeHFD} & & EBM & \small{TreeHFD} & & EBM & \small{TreeHFD} \\
  \hline
   $3$ & $0.542$ & $0.691$ & $0.692$ & $0.689$ & $0.925$ & $0.871$ & $0.923$ & $0.889$ & $0.940$ & $0.742$ & $0.783$ & $0.812$ & $0.673$ & $0.123$ & $0.816$ \\
   $5$ & $0.575$ & $0.737$ & $0.677$ & $0.906$ & $0.950$ & $0.942$ & $0.951$ & $0.955$ & $0.956$ & $0.830$ & $0.843$ & $0.864$ & $0.835$ & $0.172$ & $0.845$ \\
   $10$ & $0.678$ & $0.719$ & $0.704$ & $0.959$ & $0.967$ & $0.964$ & $0.964$ & $0.967$ & $0.962$ & $0.897$ & $0.901$ & $0.900$ & $0.883$ & $0.234$ & $0.875$ \\
  \hline 
   $5$ & $78\%$ & $40\%$ & $0.0\%$ & $59\%$ & $3.2\%$ & $0.0\%$ & $35\%$ & $2.0\%$ & $0.4\%$ & $50\%$ & $1.2\%$ & $0.0\%$ & $42\%$ & $64\%$ & $3.5\%$ \\
   $10$ & $86\%$ & $34\%$ & $6.0\%$ & $57\%$ & $5.6\%$ & $1.7\%$ & $50\%$ & $7.0\%$ & $1.4\%$ & $53\%$ & $2.4\%$ & $0.7\%$ & $48\%$ & $62\%$ & $11\%$ \\
   \hline
   \hline
\end{tabular}}
\label{table:xp_clip}
\end{table}

\begin{table}
\caption{Using CLIP, classification accuracy of HCBM with TreeHFD versus sparse linear CBM (and EBM for CUB-200) for datasets with large numbers of classes (with std $\leq 10^{-3}$), in the upper part.
Below, proportion ($\%$) of irrelevant concepts with injected leakage selected by each method.}
\centering
\setlength{\tabcolsep}{8pt}
\begin{tabular}{cccccccc}
  \hline
  \hline
  NEC & \multicolumn{2}{c}{CIFAR-100} & \multicolumn{3}{c}{CUB-200} & \multicolumn{2}{c}{ImageNet-1000} \\
   & Linear & HCBM & Linear & \multicolumn{2}{c}{HCBM} & Linear & HCBM  \\
   & & TreeHFD & & EBM & TreeHFD & & TreeHFD \\
  \hline
   $3$ & $0.642$ & $0.661$ & $0.294$ & $0.077$ & $0.318$ & $0.275$ & $0.554$\\
   $5$ & $0.728$ & $0.722$ & $0.441$ & $0.116$ & $0.404$ & $0.412$ & $0.605$\\
   $10$ & $0.773$  & $0.770$ & $0.645$ & $0.155$ & $0.626$ & $0.621$ & $0.669$\\
  \hline 
   $5$ & $35\%$ & $1.4\%$ & $72\%$ & $65\%$ & $24\%$ & $72\%$ & $3.8\%$ \\
   $10$ & $43\%$  & $4.2\%$ & $74\%$ & $63\%$ & $33\%$ & $69\%$ & $15\%$ \\
   \hline
   \hline
\end{tabular}
\label{table:xp_clip_large_classes}
\end{table}

\begin{table}
\caption{Using SigLIP, classification accuracy of HCBM with TreeHFD versus sparse linear CBM and HCBM with EBM (with std $\leq 10^{-3}$), in the upper part.
Below, proportion ($\%$) of irrelevant concepts with injected leakage selected by HCBM versus sparse linear CBM and EBM.}
\centering
\setlength{\tabcolsep}{1pt}
\resizebox{\textwidth}{!}{
\begin{tabular}{cccccccccccccccc}
  \hline
  \hline
  NEC & \multicolumn{3}{c}{Waterbirds-2} & \multicolumn{3}{c}{xView-3} & \multicolumn{3}{c}{CIFAR-10} & \multicolumn{3}{c}{EuroSAT-10} & \multicolumn{3}{c}{RESISC-45} \\
   & Linear & \multicolumn{2}{c}{HCBM} & Linear & \multicolumn{2}{c}{HCBM} & Linear & \multicolumn{2}{c}{HCBM} & Linear & \multicolumn{2}{c}{HCBM} & Linear & \multicolumn{2}{c}{HCBM}  \\
   & & EBM & \small{TreeHFD} & & EBM & \small{TreeHFD} & & EBM & \small{TreeHFD} & & EBM & \small{TreeHFD} & & EBM & \small{TreeHFD} \\
  \hline
   $3$ & $0.575$ & $0.808$ & $0.742$ & $0.619$ & $0.711$ & $0.708$ & $0.953$ & $0.894$  & $0.951$ & $0.756$ & $0.475$  & $0.678$ & $0.684$ & $0.119$ & $0.818$ \\
   $5$ & $0.589$ & $0.767$ & $0.745$ & $0.808$ & $0.754$ & $0.774$ & $0.965$ & $0.963$ & $0.963$ & $0.797$ & $0.664$ & $0.816$ & $0.840$ & $0.217$ & $0.852$ \\
   $10$ & $0.712$ & $0.810$ & $0.793$ & $0.860$ & $0.891$ & $0.893$ & $0.971$ & $0.971$  & $0.970$ & $0.876$ & $0.883$ & $0.891$ & $0.883$ & $0.189$ & $0.877$ \\
  \hline 
   $5$ & $100\%$ & $32\%$ & $0.0\%$ &  $66\%$ & $0.0\%$  & $0.0\%$ & $51\%$ & $16\%$ & $0.0\%$ & $67\%$ & $0.0\%$ & $0.2\%$ & $43\%$ & $72\%$ & $7.5\%$ \\
   $10$ & $87\%$ & $48\%$ & $18\%$ & $72\%$ & $0.0\%$ & $0.0\%$ & $58\%$ & $21\%$ & $4.6\%$ & $65\%$ & $1.2\%$ & $0.9\%$ & $47\%$ & $69\%$ & $18\%$ \\
   \hline
   \hline
\end{tabular}}
\label{table:xp_siglip}
\end{table}

\begin{table}
\caption{Using SigLIP, classification accuracy of HCBM with TreeHFD versus sparse linear CBM (and EBM for CUB-200) for datasets with large numbers of classes (with std $\leq 10^{-3}$), in the upper part.
Below, proportion ($\%$) of irrelevant concepts with injected leakage selected by each method.}
\centering
\setlength{\tabcolsep}{8pt}
\begin{tabular}{cccccccc}
  \hline
  \hline
  NEC & \multicolumn{2}{c}{CIFAR-100} & \multicolumn{3}{c}{CUB-200} & \multicolumn{2}{c}{ImageNet-1000} \\
   & Linear & HCBM & Linear & \multicolumn{2}{c}{HCBM} & Linear & HCBM  \\
   & & TreeHFD & & EBM & TreeHFD & & TreeHFD \\
  \hline
   $3$ & $0.717$ & $0.732$ & $0.292$ & $0.082$ & $0.425$ & $0.316$ & $0.633$\\
   $5$ & $0.792$ & $0.783$ & $0.438$ & $0.118$ & $0.529$ & $0.546$ & $0.682$ \\
   $10$ & $0.826$  & $0.827$ & $0.706$ & $0.182$ & $0.697$ & $0.704$ & $0.735$\\
  \hline 
   $5$ & $52\%$ & $3.1\%$ & $75\%$ & $70\%$ & $15\%$ & $24\%$ & $5.4\%$ \\
   $10$ & $66\%$ & $7.9\%$ & $74\%$ & $63\%$ & $31\%$ & $43\%$ & $16\%$ \\
   \hline
   \hline
\end{tabular}
\label{table:xp_siglip_large_classes}
\end{table}

\subsection{Information leakage}

We run additional experiments to evaluate the robustness of HCBM to interconcept leakage, as mentioned in Appendix \ref{app:xp_settings}.

First, we rerun the leakage experiments of Table \ref{table:xp_hcbm} where the noise is a standard Gaussian variable, but with a noise variance of $0.5$ and $2$ (set to $1$ for the results of Table \ref{table:xp_hcbm}).
The results are displayed in Table \ref{table:clip_leakage_noise_var}, and show that the performance is quite insensitive to the noise variance, and slightly degrades as noise variance gets smaller, as expected.
Notice that the noise is added before the normalization of the concept scores, and that the raw concept scores output by CLIP take values between $10$ and $30$ ($100 \times$cosine similarity), for example for CIFAR-10 (similar results for other datasets): mean is $15$, $5\%$-quantile is $11$ and $95\%$-quantile is $19$. The score variances for each concept are between $2$ and $10$, which makes the choice of noise variance of $0.5$, $1$, and $2$ a good span of possible signal-to-noise ratios (noise variance is thus between $5\%$ and $100\%$ of signal variance). Notice that SigLIP scores are of the same orders of magnitude.

\begin{table}
\caption{Using CLIP, proportion ($\%$) of irrelevant concepts with injected leakage selected by HCBM with TreeHFD versus linear CBM, where leakage scores are truncated copies of original scores with noise of variance $2$ and $0.5$.}
\centering
\setlength{\tabcolsep}{1.5pt}
\begin{tabular}{cccccccccccc}
  \hline
  \hline
  Noise Variance & NEC & \multicolumn{2}{c}{Waterbirds-2} & \multicolumn{2}{c}{xView-3} & \multicolumn{2}{c}{CIFAR-10} & \multicolumn{2}{c}{EuroSAT-10} & \multicolumn{2}{c}{RESISC-45} \\
   & & Linear & HCBM & Linear & HCBM & Linear & HCBM & Linear & HCBM & Linear & HCBM \\
  \hline
   $0.5$ & $5$ & $86\%$ & $6.0\%$ & $61\%$ & $2.9\%$ & $39\%$ & $4.8\%$ & $61\%$ & $0.5\%$ & $48\%$ & $7.6\%$ \\
   $0.5$ & $10$ & $88\%$ & $13\%$ & $63\%$ & $13\%$ & $63\%$ & $8.4\%$ & $64\%$ & $5.0\%$ & $54\%$ & $14\%$ \\
  \hline
   $2$ & $5$ & $43\%$ & $0.0\%$ & $32\%$ & $0.0\%$ & $12\%$ & $0.0\%$ & $21\%$ & $0.0\%$ & $29\%$ & $1.5\%$ \\
   $2$ & $10$ & $43\%$ & $2.0\%$ & $27\%$ & $0.0\%$ & $11\%$ & $0.0\%$ & $16\%$ & $0.0\%$ & $34\%$ & $6.4\%$ \\
   \hline
\end{tabular}
\label{table:clip_leakage_noise_var}
\end{table}

We also run another case, where the concepts with injected leakage are derived from the original concepts using the square function instead of a truncation (and a standard Gaussian noise as before). The results in Table \ref{table:clip_leakage_square} show that HCBM also strongly outperform linear CBM in this setting.
\begin{table}
\caption{Using CLIP, proportion ($\%$) of irrelevant concepts with injected leakage selected by HCBM with TreeHFD versus linear CBM, where leakage scores are square functions of original scores.}
\centering
\setlength{\tabcolsep}{3pt}
\begin{tabular}{ccccccccccc}
  \hline
  \hline
  NEC & \multicolumn{2}{c}{Waterbirds-2} & \multicolumn{2}{c}{xView-3} & \multicolumn{2}{c}{CIFAR-10} & \multicolumn{2}{c}{EuroSAT-10} & \multicolumn{2}{c}{RESISC-45} \\
   & Linear & HCBM & Linear & HCBM & Linear & HCBM & Linear & HCBM & Linear & HCBM \\
  \hline
   $5$ & $12\%$ & $4.0\%$ & $46\%$ & $2.7\%$ & $34\%$ & $3.2\%$ & $29\%$ & $0.0\%$ & $78\%$ & $20\%$ \\
   $10$ & $17\%$ & $14\%$ & $35\%$ & $6.4\%$ & $26\%$ & $5.0\%$ & $26\%$ & $1.0\%$ & $70\%$ & $27\%$ \\
   \hline
\end{tabular}
\label{table:clip_leakage_square}
\end{table}

\subsection{Prime implicants}

Tables \ref{table:xp_clip_implicants} and \ref{table:xp_clip_implicants_large_classes} provide the average number of concepts in each logit which belong to the prime implicant for HCBM and linear CBM (the average is taken over all logits and images of the testing data). We observe that these mean numbers of concepts are close for HCBM and linear CBM, and about half of the concepts belong to the prime implicant set for each logit on average, which shows the high efficiency of the prime implicant approach to explain predictions.
\begin{table}
\caption{Using CLIP, average number of concepts in prime implicant for each logit for HCBM predictions (top) and linear CBM (bottom), with standard deviations in brackets when $>10^{-3}$.}
\centering
\setlength{\tabcolsep}{10pt}
\begin{tabular}{cccccc}
  \hline
  \hline
  NEC  & Waterbirds-2 &  xView-3 & CIFAR-10 & EuroSAT-10 & RESISC-45 \\ 
   \hline
   $3$ & $2.1$ & $1.9$ \tiny{$(0.01)$} & $1.3$ & $1.4$ \tiny{$(0.004)$} & $0.5$ \\ 
   $5$ & $2.9$ & $2.9$ \tiny{$(0.04)$} & $2.0$ & $2.3$ \tiny{$(0.02)$} & $0.6$ \\
   $10$ & $5.0$ & $4.9$ \tiny{$(0.03)$} & $3.9$ & $4.7$ \tiny{$(0.01)$} & $1.0$ \\
   \hline
   $3$ & $2.9$ & $1.4$ \tiny{$(0.04)$} & $1.3$ & $1.5$ \tiny{$(0.004)$} & $0.6$ \\ 
   $5$ & $4.4$ & $2.1$ \tiny{$(0.06)$} & $1.8$ & $2.3$ \tiny{$(0.01)$} & $0.7$\\
   $10$ & $9.2$ & $4.7$ \tiny{$(0.05)$} & $3.8$ & $4.8$ \tiny{$(0.008)$} & $1.1$  \\
  \hline 
  \hline 
\end{tabular}
\label{table:xp_clip_implicants}
\end{table}

\begin{table}
\caption{Using CLIP, average number of concepts in prime implicant for each logit for HCBM predictions (top) and linear CBM (bottom), with standard deviations in brackets when $>10^{-3}$, for datasets with a large number of classes.}
\centering
\setlength{\tabcolsep}{10pt}
\begin{tabular}{cccc}
  \hline
  \hline
  NEC & CIFAR-100 & CUB-200 & ImageNet-1000 \\
  \hline
   $3$ & $0.9$ & $0.8$ & $1.4$ \\ 
   $5$ & $1.7$ & $1.7$ & $2.2$ \\ 
   $10$ & $4.0$ & $3.7$ & $4.6$ \\ 
  \hline 
   $3$ & $1.3$ & $1.4$ & $1.2$ \\
   $5$ & $2.2$ & $2.1$ & $1.8$ \\ 
   $10$ & $4.3$ & $4.2$ & $3.5$ \\
   \hline
   \hline
\end{tabular}
\label{table:xp_clip_implicants_large_classes}
\end{table}

\subsection{EuroSAT dataset}

Figure \ref{fig:eurosat_components_full} is an augmented version of Figure \ref{fig:eurosat_components}, to display main effects of HCBM for highway class in EuroSAT dataset, with the other concept components in Figure \ref{fig:eurosat_components_full_2}.

\begin{figure}
	\begin{center}
		\includegraphics[scale=0.37]{hcbm_Highway_natural_and_organic_growth_patterns.png} \hspace*{-0.3cm}
        \includegraphics[scale=0.37]{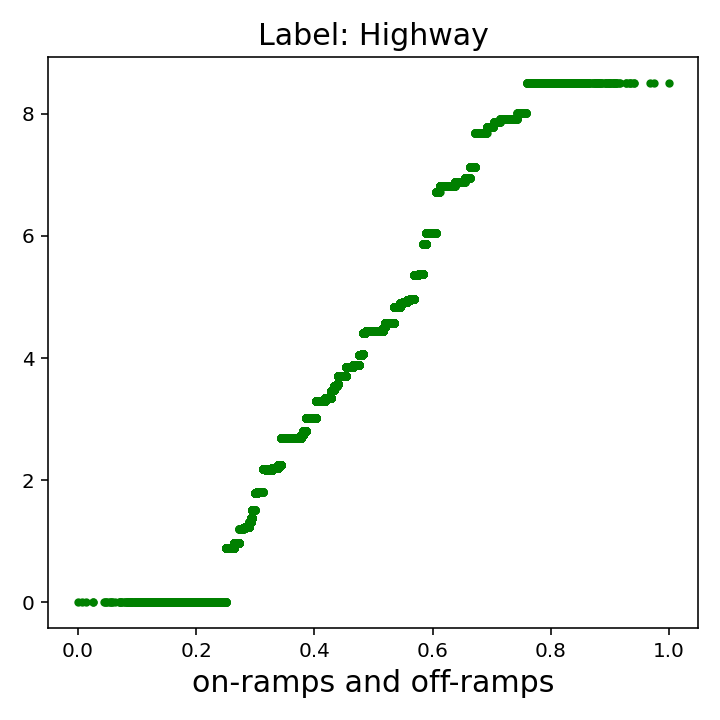} \hspace*{-0.3cm}
        \includegraphics[scale=0.37]{hcbm_Highway_markings_and_lanes_for_traffic.png}
		\caption{For EuroSAT dataset, main effects of concepts in the logit decompositions of ``Highway'' in HCBM model (NEC $= 5$).} \label{fig:eurosat_components_full}
	\end{center}
\end{figure}
\begin{figure}
	\begin{center}
		\includegraphics[scale=0.37]{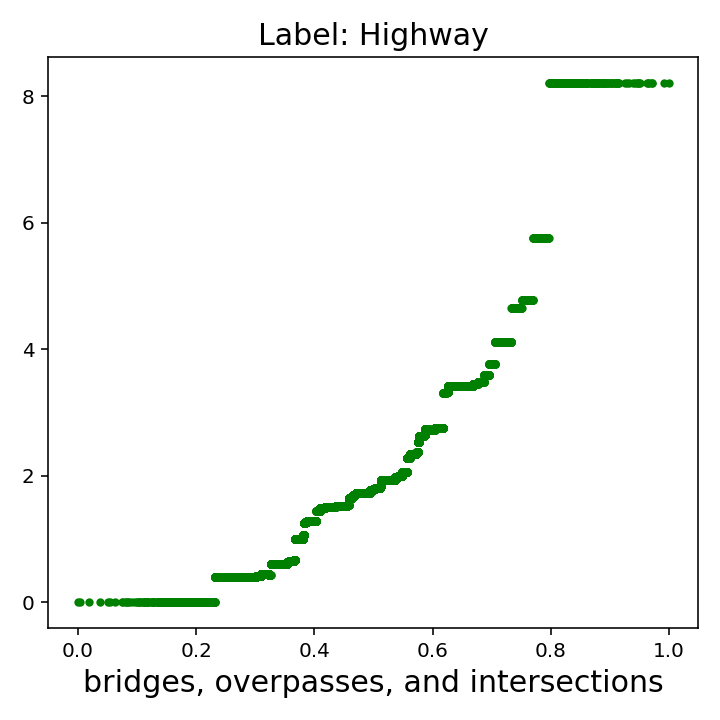} \hspace*{-0.3cm}
        \includegraphics[scale=0.37]{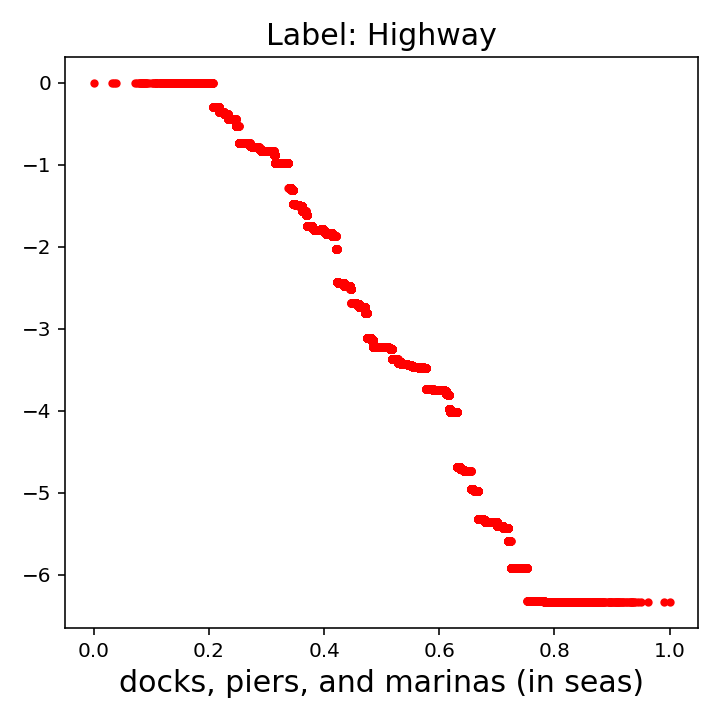} \hspace*{-0.3cm}
        \includegraphics[scale=0.37]{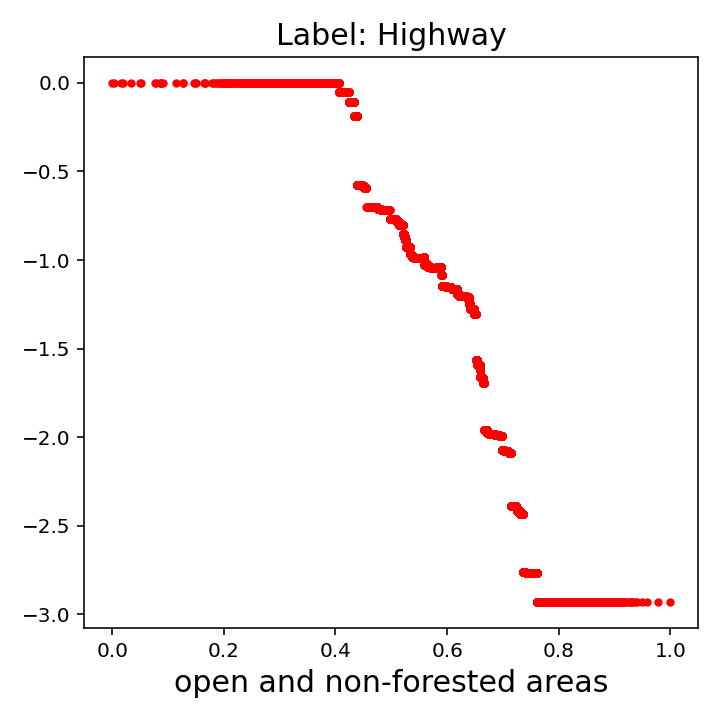}
		\caption{For EuroSAT dataset, main effects of concepts in the logit decompositions of ``Highway'' in HCBM model (NEC $= 5$).} \label{fig:eurosat_components_full_2}
	\end{center}
\end{figure}

For EuroSAT dataset, we display the lists of concepts selected by linear CBM and HCBM with NEC at $5$ in Tables \ref{table:eurosat_concepts_linear} and \ref{table:eurosat_concepts_hcbm}, respectively. The concepts \textcolor{teal}{positively} associated with the considered labels are in \textcolor{teal}{green}, and in \textcolor{red}{red} for \textcolor{red}{negative} ones.
For most classes, selected concepts are meaningful and similar for linear CBM and HCBM, as illustrated with ``Residential'', ``Highway'', and ``River''.  However, for ``Sea \& Lake'' and ``Forest'' labels, several concepts should not be positively associated with these labels, and are highlighted with a bold font in Tables \ref{table:eurosat_concepts_linear} and \ref{table:eurosat_concepts_hcbm}. For linear CBM, we observe the positive concepts of ``streetlights and road signs'' and ``regularly maintained and paved surfaces'' for ``Sea \& Lake'', as mentioned in the article. For both linear CBM and HCBM, we observe concepts related to water surfaces, since the blue-green color of forests may induce a spurious association in the VLM.
Therefore, information leakage is more severe for linear CBM than HCBM, but can also occur for HCBM because of limitations of the VLM.

\begin{table}
\caption{For EuroSAT dataset, list of \textcolor{teal}{positive} and \textcolor{red}{negative} concepts in linear CBM model (NEC $=5$).}
\centering
\setlength{\tabcolsep}{2pt}
\begin{tabular}{c >{\color{teal}}l >{\color{red}}l}
  \hline
  \hline
  Class & Positive Concepts & Negative Concepts \\
  \hline
   \multirow{3}{*}{Residential} & urban and suburban housing areas & \\
    & clustered or spaced residential structures & \\
    & community and neighborhood environments & \\
  \hline
    \multirow{3}{*}{Highway} & on-ramps and off-ramps & natural and organic growth patterns \\
    & markings and lanes for traffic & vast water horizons and shorelines \\
    & bridges, overpasses, and intersections &  \\
  \hline
     \multirow{3}{*}{River} & bridges and crossings over rivers & \begin{tabular}{l}forest floor with undergrowth and \\ fallen leaves\end{tabular} \\
     & boating and water recreational activities  & wildflowers and meadows \\
     & docks, piers, and marinas (in seas) & streetlights and road signs \\
  \hline
  \multirow{4}{*}{Sea \& Lake} & \textbf{streetlights and road signs} & lush and densely wooded areas \\
     & \textbf{regularly maintained and paved surfaces} & surrounding urban or rural landscapes  \\
     & aquatic wildlife and marine habitats & open land used for grazing livestock \\
     & different shades of blue or green water & \\
  \hline
   \multirow{4}{*}{Forest} & forest floor with undergrowth and fallen leaves& \\
   & scenic and serene natural landscapes & \\
   & \textbf{aquatic wildlife and marine habitats} & \\
   & tall trees with canopies & \\
  \hline 
  \hline 
\end{tabular}
\label{table:eurosat_concepts_linear}
\end{table}

\begin{table}
\caption{For EuroSAT dataset, list of \textcolor{teal}{positive} and \textcolor{red}{negative} concepts in HCBM model (NEC $=5$).}
\centering
\setlength{\tabcolsep}{2pt}
\begin{tabular}{c >{\color{teal}}l >{\color{red}}l}
  \hline
  \hline
  Class & Positive Concepts & Negative Concepts \\
  \hline
   \multirow{3}{*}{Residential} & urban and suburban housing areas & \\
    & clustered or spaced residential structures & \\
    & community and neighborhood environments & \\
    \hline
    \multirow{3}{*}{Highway} & on-ramps and off-ramps & natural and organic growth patterns \\
    & markings and lanes for traffic & docks, piers, and marinas (in seas)\\
    & bridges, overpasses, and intersections & open and non-forested areas \\
  \hline
    \multirow{4}{*}{River} & meandering or straight river courses & \\
     & bridges and crossings over rivers \\
     & flora and fauna along riverbanks \\
     & docks, piers, and marinas (in seas) \\
  \hline
  \multirow{3}{*}{Sea \& Lake} & aquatic wildlife and marine habitats & distinctive patterns of planted rows  \\
     & \begin{tabular}{l}expansive bodies of saltwater (sea) or \\ freshwater (lake)\end{tabular} & bridges and crossings over rivers \\
     & often found in grasslands and meadows& \\
  \hline
   \multirow{5}{*}{Forest} & lush and densely wooded areas & \begin{tabular}{l}water sources such as ponds or \\ troughs\end{tabular} \\
   & natural and organic growth patterns & \\
   & forest floor with undergrowth and fallen leaves & \\
   & scenic and serene natural landscapes & \\
   & \textbf{recreational activities (swimming, boating...)} & \\
  \hline 
  \hline 
\end{tabular}
\label{table:eurosat_concepts_hcbm}
\end{table}

\subsection{CIFAR-10 dataset}

Tables \ref{table:cifar_concepts_linear} and \ref{table:cifar_concepts_hcbm} display the concepts respectively selected by linear CBM and HCBM for NEC at $10$. For linear CBM, $13$ selected concepts positively associated with the classes are clearly irrelevant, but only $5$ for HCBM.

\begin{table}
\caption{For CIFAR-10 dataset, list of \textcolor{teal}{positive} and \textcolor{red}{negative} concepts in linear CBM model (NEC $=10$).}
\centering
\setlength{\tabcolsep}{2pt}
\begin{tabular}{c >{\color{teal}}l >{\color{red}}l}
  \hline
  \hline
  Class & Positive Concepts & Negative Concepts \\
  \hline
  Airplane & \begin{tabular}{l} heavier-than-air craft, propellers, fuselage, \\ flying fast, wing, the wing, two wings \end{tabular}  & galley, bilge pump \\
  \hline 
  Car & \begin{tabular}{l} motor vehicle, gear shift, \textbf{mean}, adornment, \\ wheels, hood, the wing, trunk, several wheels, \\ fender  \end{tabular} & \begin{tabular}{l} rudder, withers, vessel, flap, \\ wishbone, galley  \end{tabular} \\
  \hline 
  Bird & \begin{tabular}{l} \textbf{ruminant}, feather, two legs, flap, bird's foot, \\ domestic animal, beak, majestic, vertebrate  \end{tabular} & bulkhead, paws \\
  \hline 
    Cat & whisker, paws, gray, \textbf{radiator, galley}, feline & \\
  \hline 
    Deer & mammal, antler, ruminant, yearling, rising early, \textbf{foal} & oyster, giblet \\
  \hline 
    Dog & giblet, wishbone, \textbf{fender}, canine & vehicle type, syrinx \\
  \hline 
    Frog & \begin{tabular}{l} \textbf{ratline, uropygial gland, mammal}, amphibian, \\ skeleton, \textbf{horse's foot, bird's foot, feline}, \\ vertebrate  \end{tabular} &  \begin{tabular}{l} emergency brakes, gyrostabilizer, \\ porthole, majestic, black  \end{tabular} \\
  \hline 
    Horse & equine, stallion, horseback, horsemeat & \\
  \hline 
    Ship & \begin{tabular}{l} hull, forecastle, navigation light, seanchor, \\ the rudder, galley, bilge pump  \end{tabular} & domestic animal \\
  \hline 
    Truck & \begin{tabular}{l} cargo area, vehicle type, stabilizer bar, vehicle, \\ \textbf{gymnastic apparatus}, many wheels, \\ emergency brakes, tailgate, exhaust pipe, fuel tank  \end{tabular} & \begin{tabular}{l} adornment, the wing, horseback, \\ mute, the tail, lubber's hole, \\ nice friend, riding bitt  \end{tabular}\\
  \hline 
  \hline 
\end{tabular}
\label{table:cifar_concepts_linear}
\end{table}

\begin{table}
\caption{For CIFAR-10 dataset, list of \textcolor{teal}{positive} and \textcolor{red}{negative} concepts in HCBM model (NEC $=10$).}
\centering
\setlength{\tabcolsep}{2pt}
\begin{tabular}{c >{\color{teal}}l >{\color{red}}l}
  \hline
  \hline
  Class & Positive Concepts & Negative Concepts \\
  \hline
  Airplane & \begin{tabular}{l} heavier-than-air craft, propellers, fuselage \end{tabular}  & winch, navigation light, domestic animal \\
  \hline 
  Car & \begin{tabular}{l} motor vehicle, gear shift, adornment, hood, \\ \textbf{brains}, drive train, lubber's hole  \end{tabular} & \begin{tabular}{l} pet, furcula, rudder, withers, vessel, flap, \\ gymnastic apparatus, wishbone  \end{tabular} \\
  \hline 
  Bird & \begin{tabular}{l} pennon, \textbf{ruminant}, feather, flap, bird's foot, \\ drumstick, beak  \end{tabular} & \begin{tabular}{l} bulkhead, paws, antler, \\ two ears, amphibian, horsemeat \end{tabular} \\
  \hline 
    Cat & whisker, paws, gray, feline &  \begin{tabular}{l} amphibian, vehicle type, beak, \\ canine \end{tabular} \\
  \hline 
    Deer & \begin{tabular}{l} mammal, antler, two ears, ruminant, \\ yearling, rising early, \textbf{syrinx} \end{tabular} & oyster, amphibian, cockpit \\
  \hline 
    Dog & paws, giblet, wishbone, canine & \begin{tabular}{l} withers, vehicle type, feline, \\ vertebrate \end{tabular} \\
  \hline 
    Frog & \textbf{uropygial gland}, amphibian, brains, clutch &  emergency brakes, porthole, majestic \\
  \hline 
    Horse & equine, stallion, horse's foot, horseback, white & antler, amphibian \\
  \hline 
    Ship & \begin{tabular}{l} vessel, superstructure, forecastle, \\ navigation light, seanchor, galley, bilge pump \end{tabular} & mammal, flap, domestic animal \\
  \hline 
    Truck & \begin{tabular}{l} pennon, flap, \textbf{gymnastic apparatus}, \\ many wheels, tailgate, gear, wishbone \end{tabular} & \begin{tabular}{l} wing, the wing, the tail, \\ lubber's hole, nice friend, bay, \\ riding bitt  \end{tabular}\\
  \hline 
  \hline 
\end{tabular}
\label{table:cifar_concepts_hcbm}
\end{table}

\section{Interventions} \label{app:interventions}

One of the main features of Concept Bottleneck Models is that users can manually intervene on concept bottlenecks to improve models. In this section, we assess the effectiveness of HCBM with interventions on three datasets: CIFAR-10, MetaShift, and Waterbirds.

\subsection{Interventions on CIFAR-10}
In Tables \ref{table:cifar_concepts_linear} and \ref{table:cifar_concepts_hcbm} of Appendix \ref{app:xp_results}, we show the concepts selected by linear CBM and HCBM for CIFAR-10 with NEC at $10$, and highlight in bold the positive concepts that are clearly irrelevant. Hence, we proceed to the following intervention experiment: we manually set to zero the contributions of these irrelevant concepts in both CBM, and compute predictions on the testing data. Then, we derive the obtained accuracy for both CBM. For HCBM, the accuracy drops from $0.962$ to $0.951$, when the irrelevant concepts are set to zero. On the other hand, the accuracy drops from $0.964$ to $0.851$ for linear CBM.

It is quite expected to get a lower accuracy decrease for HCBM than linear CBM, because we zero out $5$ irrelevant concepts for HCBM and $13$ for linear CBM. However, it is interesting to see the high robustness of HCBM with respect to these interventions, with a small accuracy decrease of $0.01$. Although HCBM select several irrelevant concepts because of concept-task leakage, this intervention experiment shows that HCBM predictions barely rely on these irrelevant concepts, as opposed to linear CBM, which face a large accuracy decrease of $0.11$.

\subsection{Interventions on MetaShift}
In this subsection, we assess the effectiveness of HCBM with interventions on the MetaShift dataset~\citep{liang2022metashift}. To evaluate model behavior under distribution shifts and test robustness against spurious correlations, we reproduce the concept intervention experiments, presented in Post-hoc CBM~\citep{yuksekgonul2022post} and reproduced in~\citet{midavaine2024re}. 

\paragraph{MetaShift.} 
This dataset was constructed by leveraging the scene graphs from Visual Genome~\citep{krishna2017visual} to create controlled distribution shifts based on object co-occurrences. MetaShift is designed to evaluate a model's reliance on spurious correlations. Specifically, we evaluate Task 1, a 5-class classification benchmark where distribution shifts are introduced for one target class at a time (e.g., during training, images of \textit{bed} predominantly contain \textit{cats}, whereas in the test set, these correlations are removed and images contain \textit{dogs}).
Both linear CBM and HCBM are trained using the BRODEN~\citep{fong2018net2vec} concept set under a sparsity constraint of NEC at $10$.

\paragraph{Interventions.} 
We apply \textit{pruning interventions}, where spurious concepts selected for the shifted class are zeroed out at inference time without retraining. Table~\ref{tab:metashift_interventions} reports the overall classification accuracy across six distribution shift settings before and after concept pruning. While both methods achieve comparable initial accuracy ($\sim 0.850$), linear CBM experiences a severe performance degradation after pruning (dropping to an average of $0.781$). In contrast, HCBM demonstrate superior robustness, retaining an average accuracy of $0.831$ ($+0.050$ over linear CBM).

\begin{table}[htbp]
    \centering
    \caption{Classification accuracy on MetaShift (Task 1) under distribution shifts, before (Original) and after (Prune) zeroing out spurious concepts (NEC at $10$). Values in parentheses indicate the change relative to the model before model intervention.}
    \label{tab:metashift_interventions}
    \begin{tabular}{cccccc}
        \toprule
         & & \multicolumn{2}{c}{Original} & \multicolumn{2}{c}{Prune} \\
        \cmidrule(lr){3-4} \cmidrule(lr){5-6}
        Train & Test & Linear & HCBM & Linear & HCBM \\
        \midrule
        bed (cat)     & bed (dog)   & 0.906 & 0.912 & 0.764 ({\color{red}-0.142}) & 0.912 ($=$) \\
        bed (dog)     & bed (cat)   & 0.918 & 0.902 & 0.774 ({\color{red}-0.144}) & 0.880 ({\color{red}-0.022}) \\
        table (books) & table (cat) & 0.824 & 0.792 & 0.824 ($=$)                 & 0.786 ({\color{red}-0.006}) \\
        table (books) & table (dog) & 0.800 & 0.824 & 0.782 ({\color{red}-0.018}) & 0.828 ({\color{green!60!black}+0.004}) \\
        table (cat)   & table (dog) & 0.822 & 0.822 & 0.772 ({\color{red}-0.050}) & 0.802 ({\color{red}-0.020}) \\
        table (dog)   & table (cat) & 0.832 & 0.836 & 0.772 ({\color{red}-0.060}) & 0.780 ({\color{red}-0.056}) \\
        \midrule
        \multicolumn{2}{c}{Average (Task 1)} & 0.850 & 0.848 & 0.781 ({\color{red}-0.069}) & 0.831 ({\color{red}-0.017}) \\
        \bottomrule
    \end{tabular}
\end{table}

To better understand why linear CBM collapses after model intervention, Table~\ref{tab:metashift_breakdown} details the per-class accuracy breakdown for the \textit{bed (cat)} $\rightarrow$ \textit{bed (dog)} shift. When the spurious \textit{cat} concept is pruned, the accuracy of linear CBM on the \textit{bed} class drops to $0.0$, indicating that its linear aggregation had fully relied on the shortcut correlation. In contrast, HCBM maintain $0.820$ accuracy on the shifted class, confirming that its non-linear additive aggregation learns more faithful and independent concept representations.

\begin{table}[htbp]
    \centering
    \caption{Per-class accuracy breakdown for the \textit{bed (cat)} $\rightarrow$ \textit{bed (dog)} distribution shift before and after pruning (NEC at $10$). Values in parentheses indicate the change relative to the model before model intervention.}
    \label{tab:metashift_breakdown}
    \begin{tabular}{ccccc}
        \toprule
        & \multicolumn{2}{c}{Linear CBM} & \multicolumn{2}{c}{HCBM} \\
        \cmidrule(lr){2-3} \cmidrule(lr){4-5}
        Class & Original & Prune & Original & Prune \\
        \midrule
        Airplane & 0.96  & 0.97 ({\color{green!60!black}+0.01}) & 0.97  & 0.97 ($=$) \\
        Bed      & 0.78  & 0.00 ({\color{red}-0.78})                 & 0.87  & 0.82 ({\color{red}-0.05}) \\
        Car      & 0.93  & 0.93 ($=$)                                & 0.90  & 0.90 ($=$) \\
        Cow      & 0.94  & 0.94 ($=$)                                & 0.97  & 0.97 ($=$) \\
        Keyboard & 0.92  & 0.98 ({\color{green!60!black}+0.06}) & 0.85  & 0.90 ({\color{green!60!black}+0.05}) \\
        \midrule
        Average  & 0.906 & 0.764 ({\color{red}-0.142})                & 0.912 & 0.912 ($=$) \\
        \bottomrule
    \end{tabular}
\end{table}

\subsection{Interventions on Waterbirds}
In this subsection, we assess the effectiveness of HCBM with interventions on the Waterbirds dataset~\citep{Sagawa2020Distributionally}. 

\paragraph{Waterbirds.} 
This dataset was constructed by superimposing bird crops from CUB~\citep{wah2011caltech} onto background scenes from Places~\citep{zhou2017places}. Similar to MetaShift, Waterbirds is designed to evaluate a model's reliance on spurious correlations: during training, landbirds and waterbirds are predominantly paired with their typical land and water backgrounds respectively. However, these correlations are removed in the test set, creating a distribution shift which penalizes models relying on environmental cues. 
Following recent works~\citep{rao2024discover,enouen2026debugging}, we investigate how targeted interventions affect worst-group performance. Specifically, we first test whether relying only on bird-specific concepts leads to a performance gain. On the other hand, we also evaluate if ablating bird-related concepts results in a performance drop in worst-group classification. 
To do this, we use \citet{enouen2026debugging} concept set constructed from a synthetic concept set~\citep{wu2023discover} and the attributes from CUB, translated to natural language, resulting in 533 concepts, available at \url{https://github.com/ericenouen/cbdebug/blob/main/pcbm/concepts/Waterbirds.txt}.
\paragraph{Training setup.} We first train our HCBM (NEC $=10$) for the binary classification task of the Waterbirds dataset (i.e., Landbird versus Waterbird), achieving a performance of 70.42\%. 
\paragraph{Interventions.} Following~\citet{enouen2026debugging} setup, we then observe for each class the concepts selected by our HCBM  (NEC $=10$) and manually classify them as bird-related concepts or background concepts. 
For reference, our HCBM (NEC $=10$) relied on ``hooked seabird beak'', ``duck-like body'', ``gull-like body'', ``tree-clinging-like body'' for both classes. These concepts are considered bird-related concepts. Moreover, our model HCBM also relied on ``harbor'', ``lake'', ``sea'', ``tree'', ``lacelike'' for both classes, and ``matted'' for the Landbird class and ``porous'' for the Waterbird class. We consider these concepts as background concepts.
Finally, we apply two types of interventions: \textbf{Remove} and \textbf{Retrain} \citep{enouen2026debugging}. \textbf{Remove} zeroes out the contributions of the concepts to be removed, without retraining the model, while \textbf{Retrain} retrains HCBM with only the selected subset of concepts.
\paragraph{Group-wise Evaluation.}  We report the group-wise results in Table~\ref{tab:intervention}, distinguishing between the correlated pairs seen during training: ``Landbirds on Land'' (L.Bird@L) and ``Waterbirds on Water''(W.Bird@W), and the out-of-distribution ``worst groups'': ``Landbirds on Water'' (L.Bird@W) and ``Waterbirds on Land''(W.Bird@L), which appear only during the evaluation phase.\\ 
\textbf{Bird-related concepts.} For both HCBM and linear CBM, retraining with only bird-related concepts significantly improves worst-group performance, with gains of $+0.101$ and $+0.064$ for HCBM, and $+0.133$ and $+0.131$ for linear CBM, alongside overall improvements of $+0.023$ and $+0.038$ respectively. This confirms that bird-related concepts are sufficient for robust classification under distribution shift. 
In contrast, simply removing background concept weights without retraining leads to a strongly asymmetric behavior: while L.Bird@W improves drastically for HCBM ($+0.397$), W.Bird@L and W.Bird@W collapse ($-0.232$ and $-0.532$).
This means that simply removing concepts without retraining in our context of small NEC values, leads to badly calibrated classifiers. This effect is even more pronounced for linear CBM, which degenerate to a trivial predictor with $0.000$ on both worst groups and $1.000$ on training groups. This suggests that linear CBM have fully exploited the spurious correlation between background concepts and class labels (e.g., aquatic background $\rightarrow$ Waterbird), such that removing background weights leaves the model with no discriminative signal.\\
\textbf{Background concepts.} Restricting the model to background concepts dramatically degrades worst-group performance for both models, with drops of $-0.104$ and $-0.117$ points for HCBM, and $-0.111$ and $-0.075$ points for linear CBM under Retrain, while training group accuracy is largely preserved. This demonstrates that background concepts encode spurious correlations that fail to generalize, and that both HCBM and linear CBM rely on background concepts.  

These results confirm that targeted interventions on the concept set yield the expected impact on group-wise performance. 

\begin{table}[t]
\centering
\caption{Overall and group-wise accuracy of HCBM (NEC $=10$) and linear CBM on Waterbirds before and after concept intervention. Values in parentheses indicate the change relative to the model before intervention. Worst groups refer to out-of-distribution pairs unseen during training.}
\resizebox{\columnwidth}{!}{%
\begin{tabular}{ccccccccc}
\toprule
& & & & \multicolumn{2}{c}{\textbf{Worst Groups}} & \multicolumn{2}{c}{\textbf{Training Groups}} \\
\cmidrule(lr){5-6} \cmidrule(lr){7-8}
\multirow{-2}{*}{\textbf{Model}} & \multirow{-2}{*}{\textbf{Intervention}} & \multirow{-2}{*}{\textbf{Method}} & \multirow{-2}{*}{\textbf{Overall}} & L.Bird@W & W.Bird@L & L.Bird@L & W.Bird@W \\
\midrule
Linear CBM  & Before Intervention & Original & 0.678 & 0.612 & 0.221 & 0.997 & 0.883 \\
HCBM        & Before Intervention & Original & 0.704 & 0.580 & 0.349 & 0.996 & 0.893 \\
\midrule
\multirow{4}{*}{Linear CBM} 
& \multirow{2}{*}{Bird-related Concepts} & Remove  & 0.500 ({\color{red}-0.178})  & 0.000 ({\color{red}-0.612})                & 0.000 ({\color{red}-0.221})   & \bf 1.000 ({\color{green!60!black}+0.003}) & \bf 1.000 ({\color{green!60!black}+0.117}) \\
&                                         & Retrain & \bf 0.716 ({\color{green!60!black}+0.038})  & \bf 0.745 ({\color{green!60!black}+0.133}) & \bf 0.352 ({\color{green!60!black}+0.131}) & 0.996 ({\color{red}-0.001})                & 0.773 ({\color{red}-0.110}) \\
& \multirow{2}{*}{Background Concepts}   & Remove  & 0.629 ({\color{red}-0.049})   & 0.616 ({\color{green!60!black}+0.004})     & 0.073 ({\color{red}-0.148})   & 0.997 ($=$)     & 0.830 ({\color{red}-0.053}) \\
&                                         & Retrain & 0.628 ({\color{red}-0.050})   & 0.501 ({\color{red}-0.111})               & 0.146 ({\color{red}-0.075})   & 0.996 ({\color{red}-0.001})                & 0.869 ({\color{red}-0.014}) \\
\midrule
\multirow{4}{*}{HCBM} 
& \multirow{2}{*}{Bird-related Concepts} & Remove  & 0.614 ({\color{red}-0.090})   & \bf 0.977 ({\color{green!60!black}+0.397}) & 0.117 ({\color{red}-0.232})  & \bf 1.000 ({\color{green!60!black}+0.004}) & 0.361 ({\color{red}-0.532}) \\
&                                         & Retrain & \bf 0.727 ({\color{green!60!black}+0.023})  & 0.681 ({\color{green!60!black}+0.101})     & \bf 0.413 ({\color{green!60!black}+0.064})   & 0.989 ({\color{red}-0.007})                & 0.827 ({\color{red}-0.066}) \\
& \multirow{2}{*}{Background Concepts}   & Remove  & 0.641 ({\color{red}-0.063})   & 0.455 ({\color{red}-0.125})               & 0.221 ({\color{red}-0.128})  & 0.991 ({\color{red}-0.005})                & 0.899 ({\color{green!60!black}+0.006}) \\
&                                         & Retrain & 0.650 ({\color{red}-0.054})   & 0.476 ({\color{red}-0.104})               & 0.232 ({\color{red}-0.117})  & 0.991 ({\color{red}-0.005})                & \bf 0.900 ({\color{green!60!black}+0.007}) \\
\bottomrule
\end{tabular}%
}
\label{tab:intervention}
\end{table}

\section{xView dataset} \label{app:xview_concepts}

\subsection{Class definitions}

We list the label numbers corresponding to the three parent classes of interest: vehicles, aircraft, and maritime vessels.
\begin{itemize}
    \item Vehicles: $4$, $5$, $6$, $17$--$22$
    \item Aircraft: $0$, $1$, $2$
    \item Maritime vessels: $23$--$32$
\end{itemize}

\subsection{Experiment settings}
We provide the details for experiments of object detection with xView dataset.
A CNN detector with RTMDet architecture is first trained for xView, and achieves a high accuracy, computed as follows. We consider that a detection matches a labeled bounding box, if the IoU metric is higher than $0.1$, where the IoU of two bounding boxes is defined as the ratio of the area of their intersection with the area of their union. Otherwise, the true labels of the detections are considered as ``background''. The trained RTMDet detector thus has a detection accuracy of $86\%$ for aircraft, $77\%$ for vehicles, $78\%$ for maritime vessels, and $99\%$ for background. For all classes, most errors come from non-detected objects.
Next, we consider the crops and logits generated by the trained detector, and only keep the crops which have at least one of their logits above $0.3$, for an optimized tradeoff between false positive and false negative detections. Additionally, the number of vehicles is high (above $200 000$) and most small cars are very similar. We thus downsample this class to $6 000$ units to reach a size similar to that of the other classes.
The obtained dataset is used to fit HCBM and linear CBM to provide explainable classification of the detected bounding boxes. Then, we compute the proportion of explained variance obtained by the surrogate CBM for each logit, and the balanced classification accuracy, estimated by $5$-fold cross-validation. Metrics are averaged over several repetitions to make standard deviations negligible.

\subsection{HCBM model}

In Figures \ref{fig:xview_HCBM_1}--\ref{fig:xview_HCBM_5}, we show an example of HCBM output for xView dataset in the classification case with NEC $= 5$, where the training data is directly built from the annotated bounding boxes.
\begin{figure}
	\begin{center}
		\includegraphics[scale=0.37]{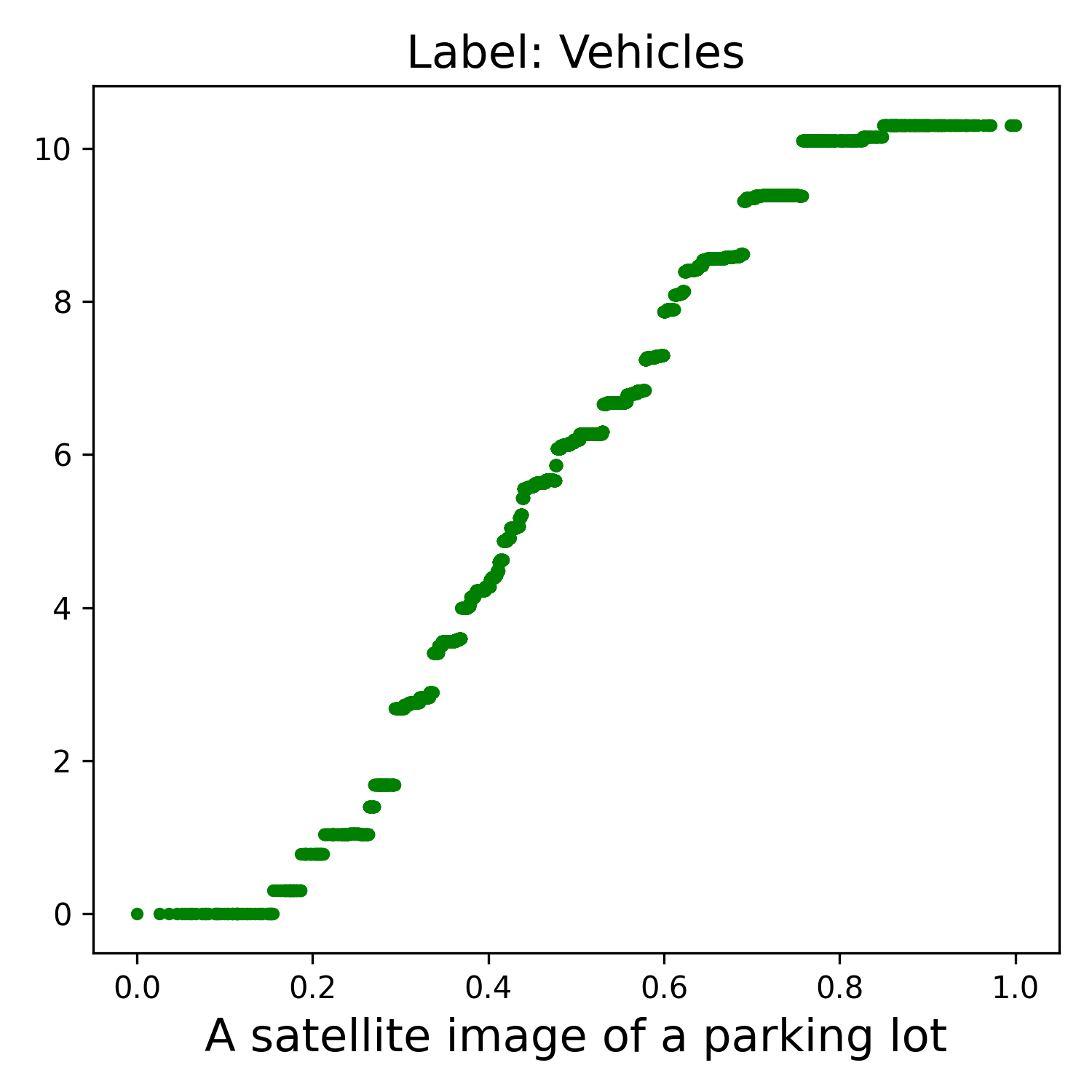} \hspace*{-0.3cm}
        \includegraphics[scale=0.37]{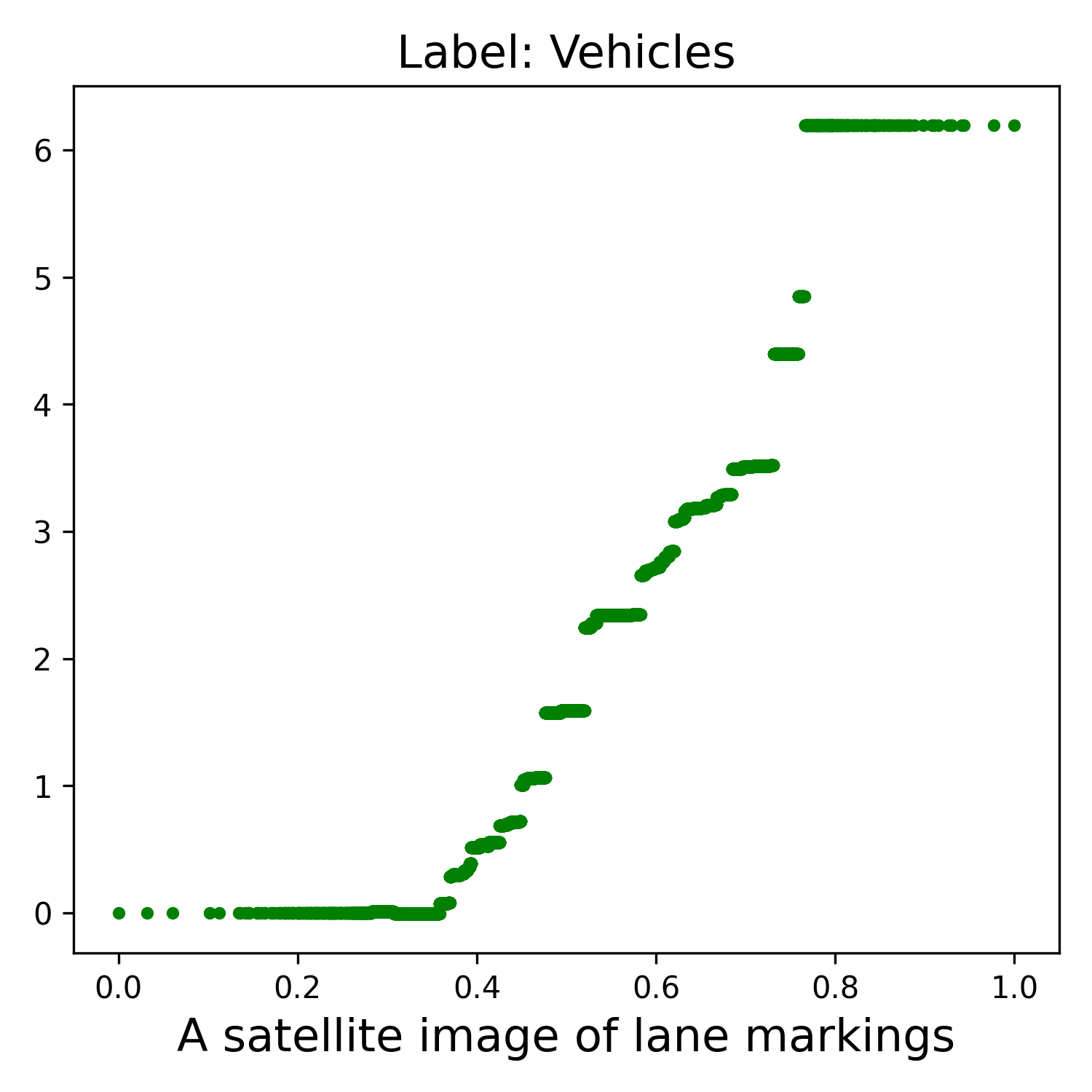} \hspace*{-0.3cm}
        \includegraphics[scale=0.37]{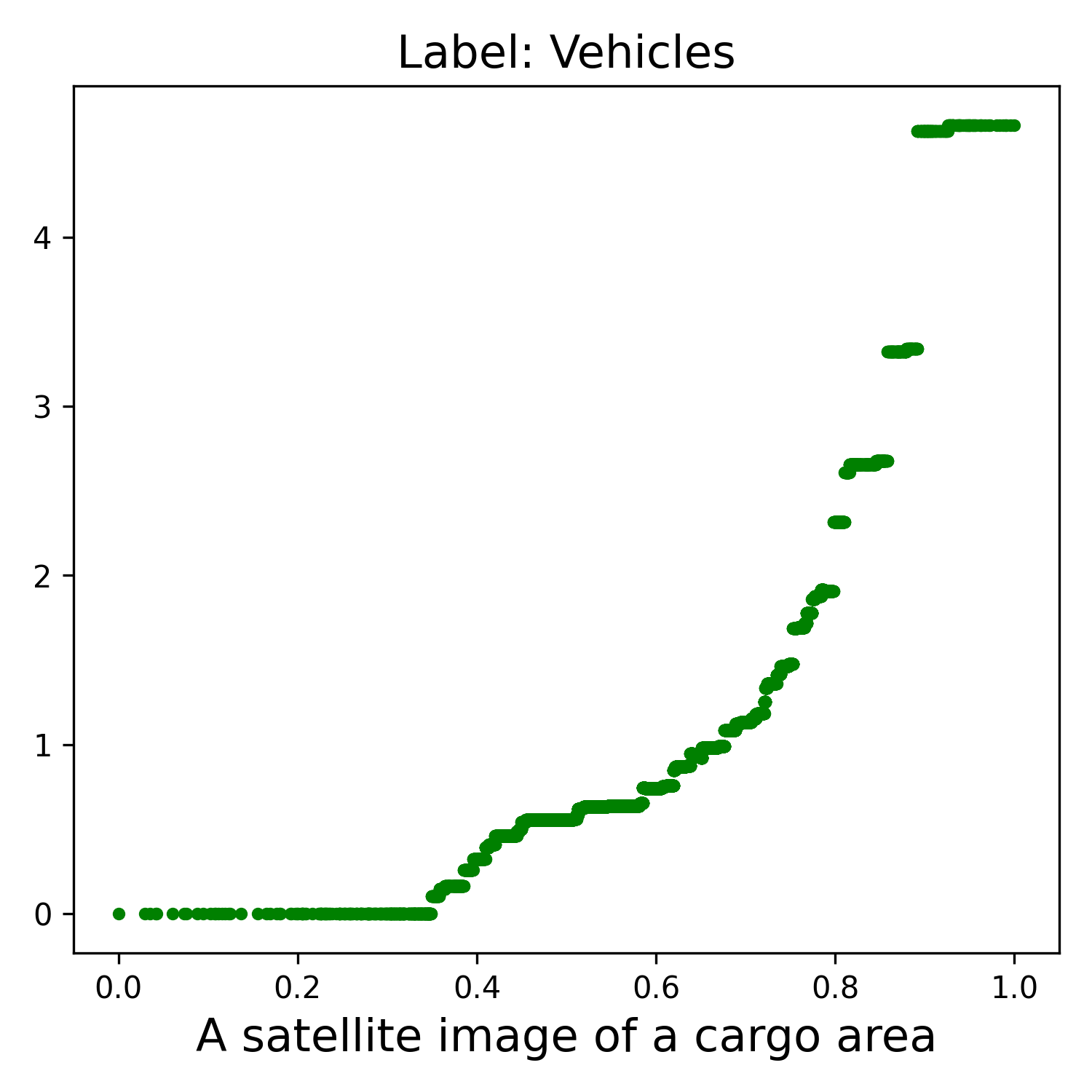}
		\caption{For xView dataset, main effects of positive concepts in the logit decompositions of ``Vehicles'' in HCBM model (NEC $=5$).} \label{fig:xview_HCBM_1}
	\end{center}
\end{figure}
\begin{figure}
	\begin{center}
		\includegraphics[scale=0.37]{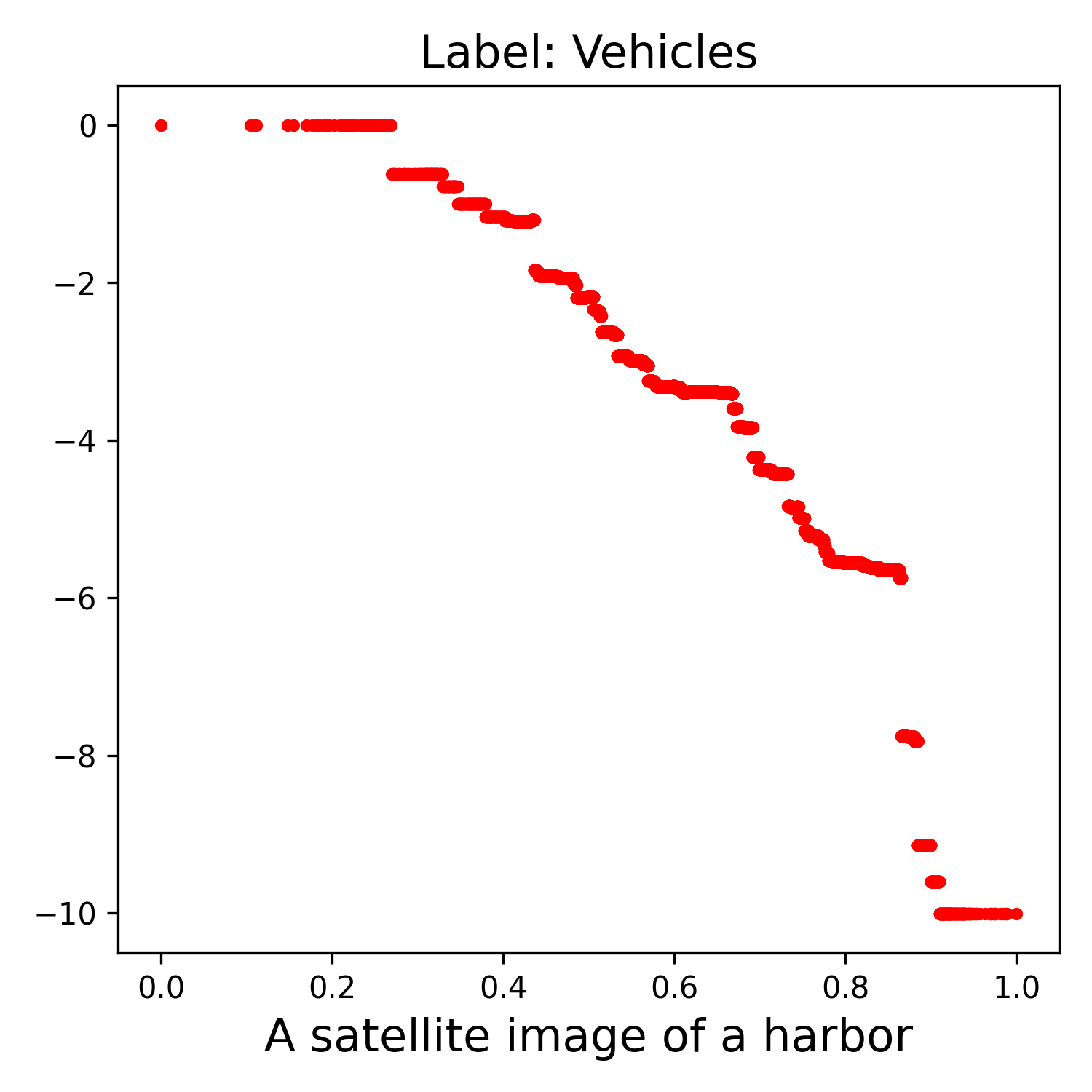} \hspace*{-0.3cm}
        \includegraphics[scale=0.37]{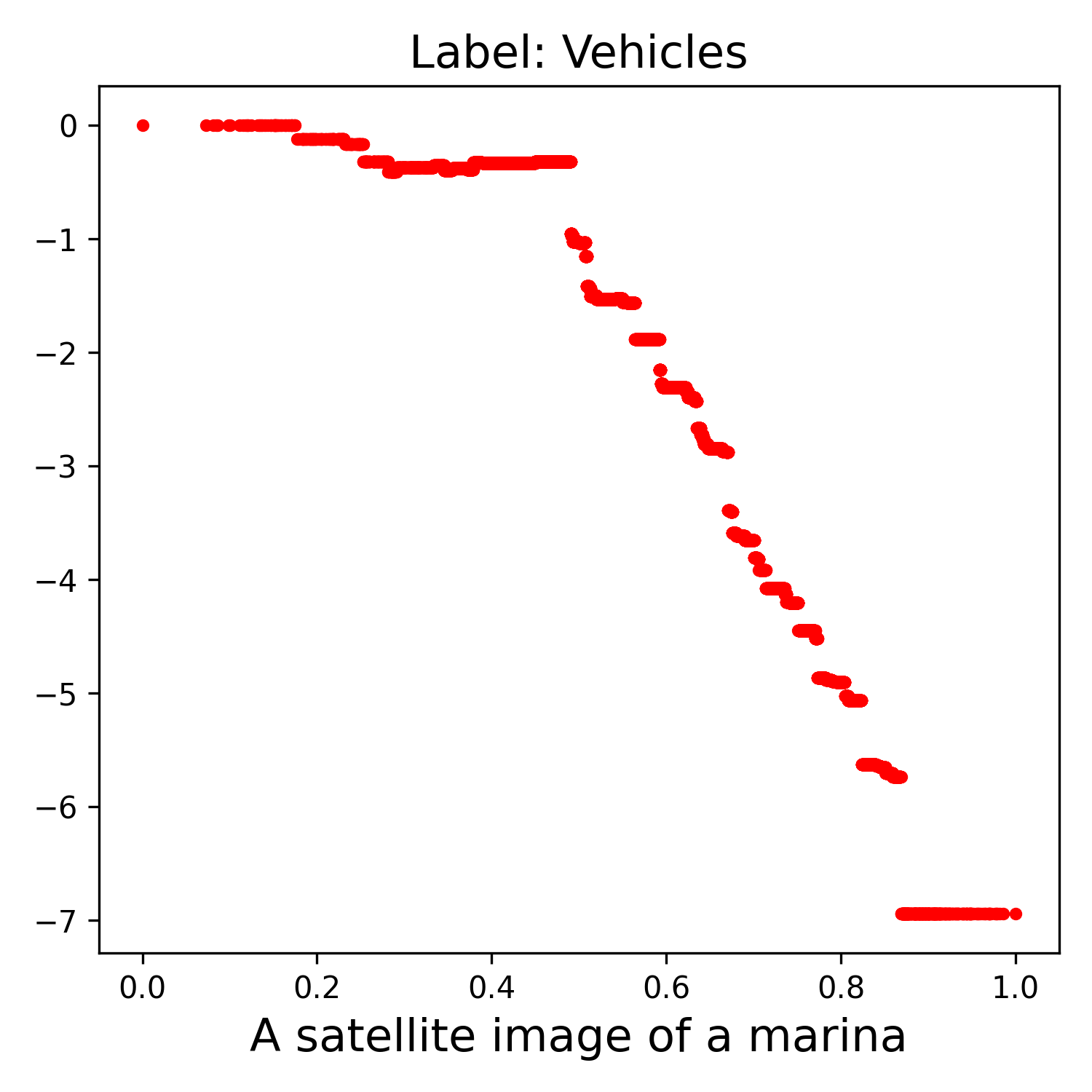} \hspace*{-0.3cm}
        \includegraphics[scale=0.37]{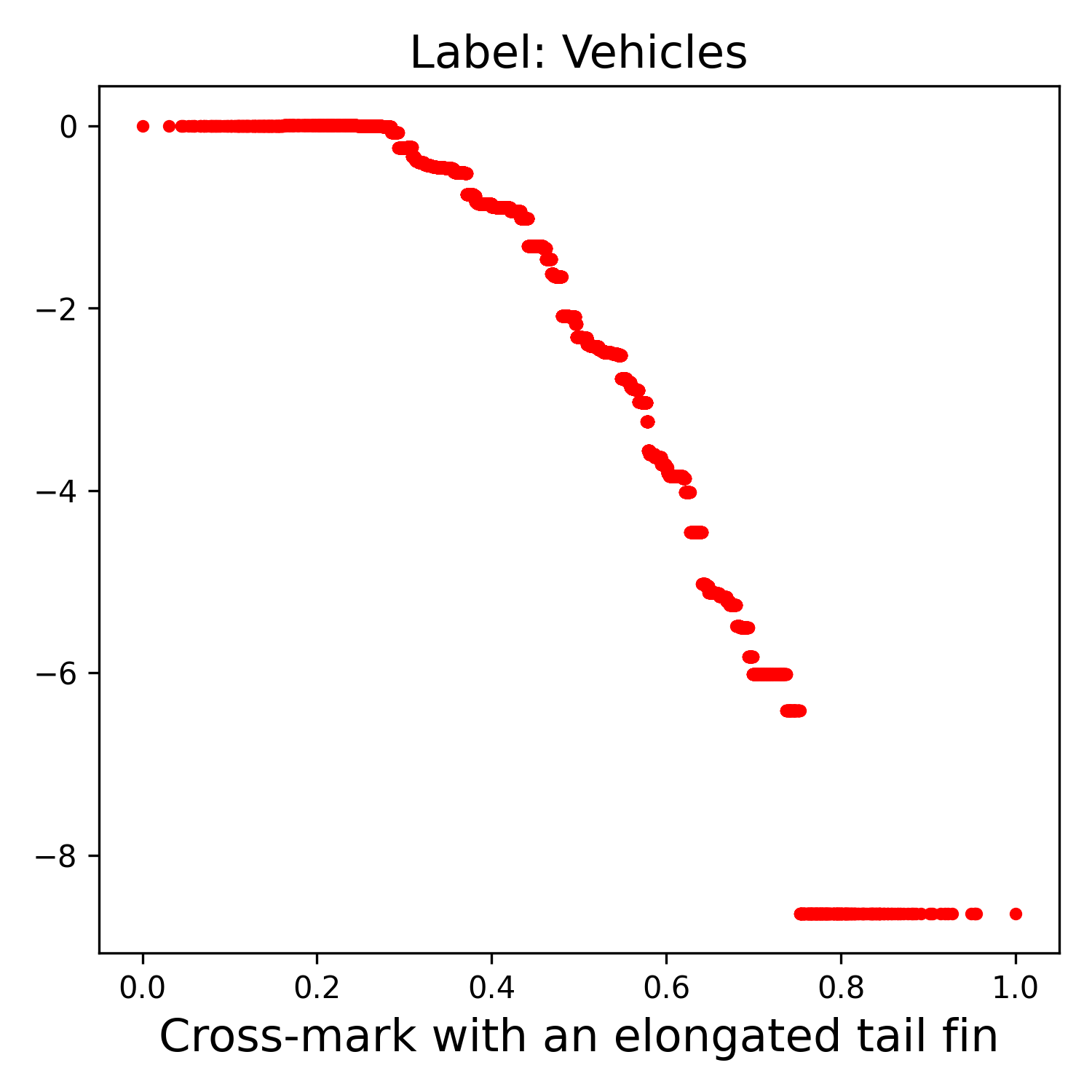}
		\caption{For xView dataset, main effects of negative concepts in the logit decompositions of ``Vehicles'' in HCBM model (NEC $=5$).} \label{fig:xview_HCBM_2}
	\end{center}
\end{figure}
\begin{figure}
	\begin{center}
		\includegraphics[scale=0.37]{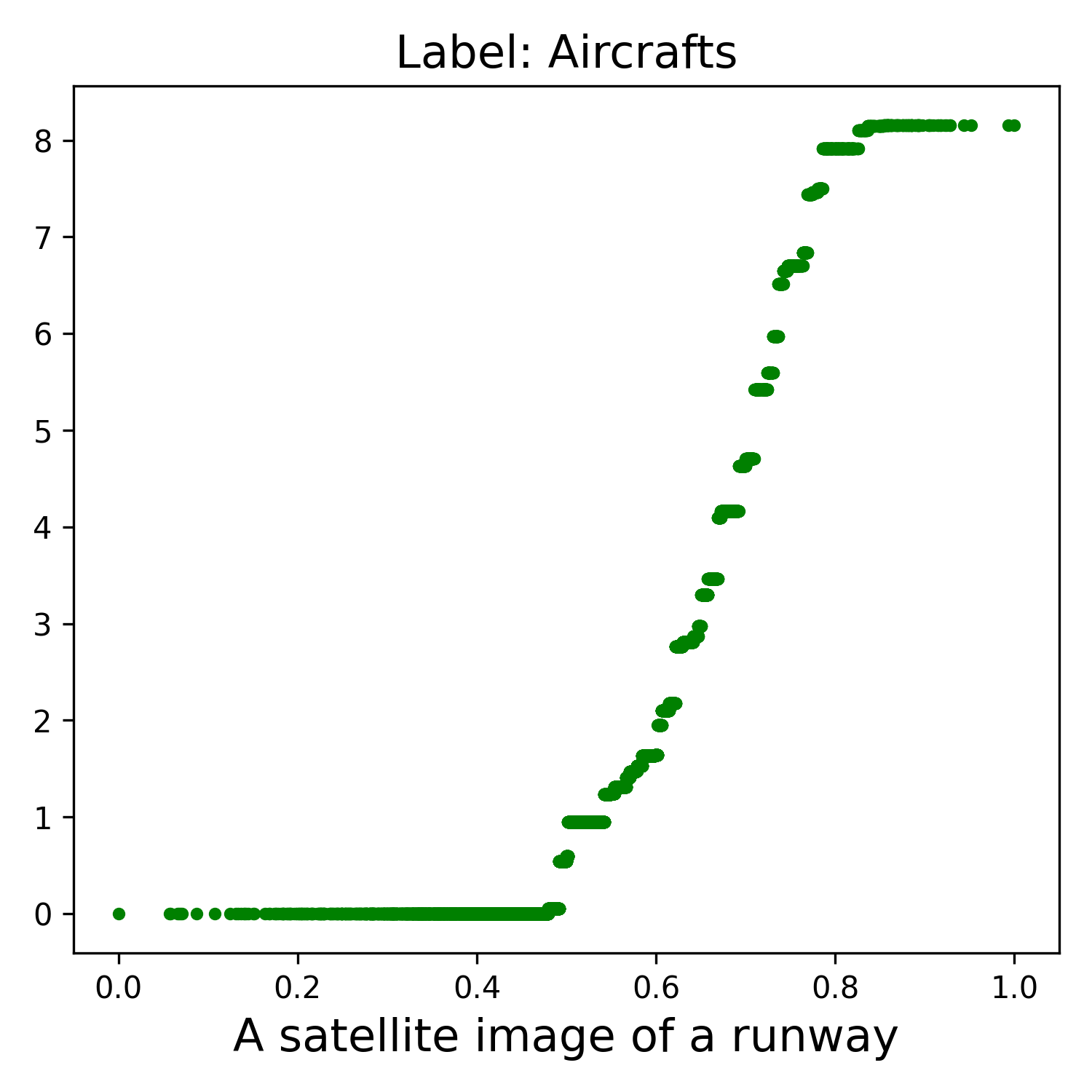} \hspace*{-0.3cm}
        \includegraphics[scale=0.37]{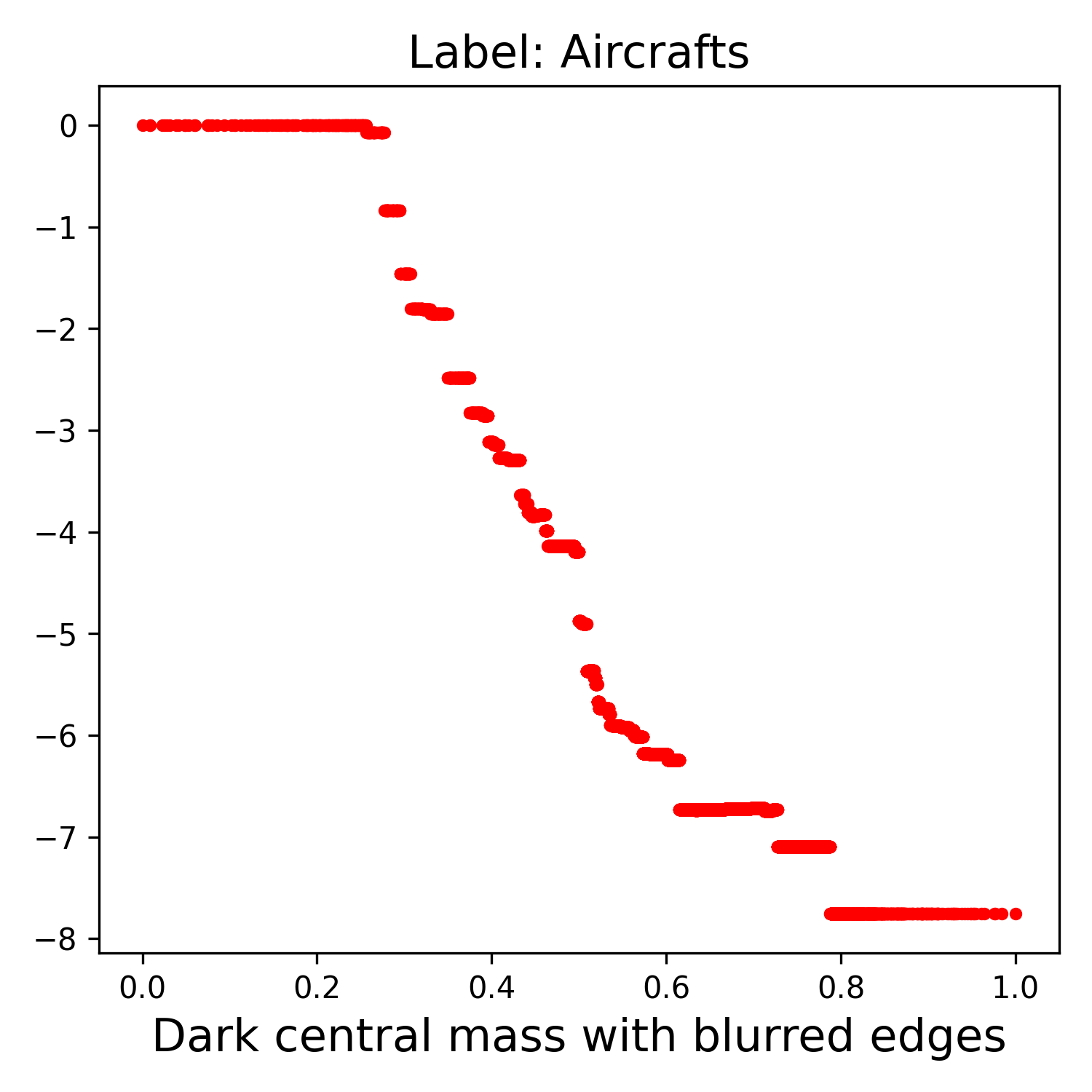} \hspace*{-0.3cm}
        \includegraphics[scale=0.37]{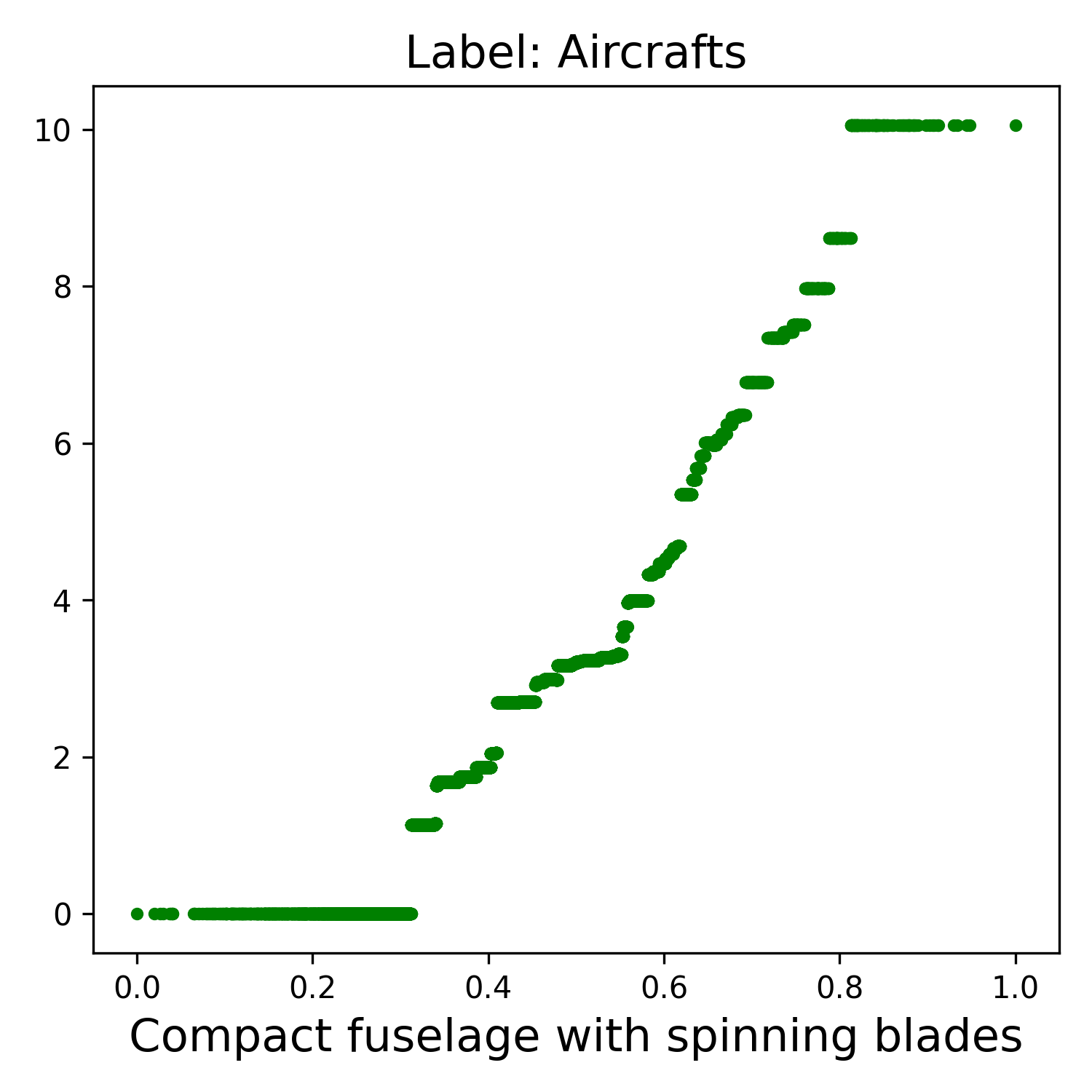}
		\caption{For xView dataset, main effects of concepts in the logit decompositions of ``Aircraft'' in HCBM model (NEC $=5$).} \label{fig:xview_HCBM_3}
	\end{center}
\end{figure}
\begin{figure}
	\begin{center}
		\includegraphics[scale=0.37]{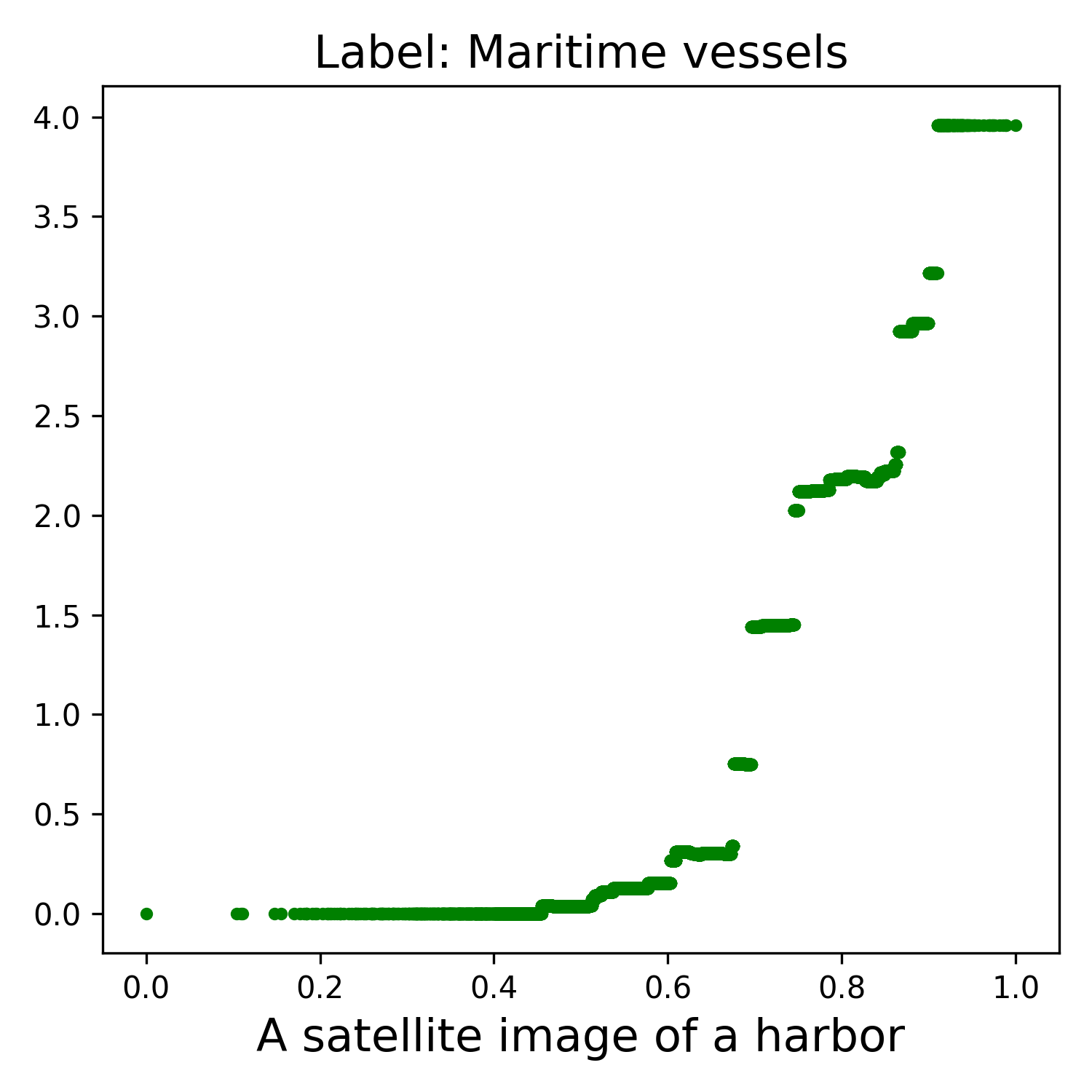} \hspace*{-0.3cm}
        \includegraphics[scale=0.37]{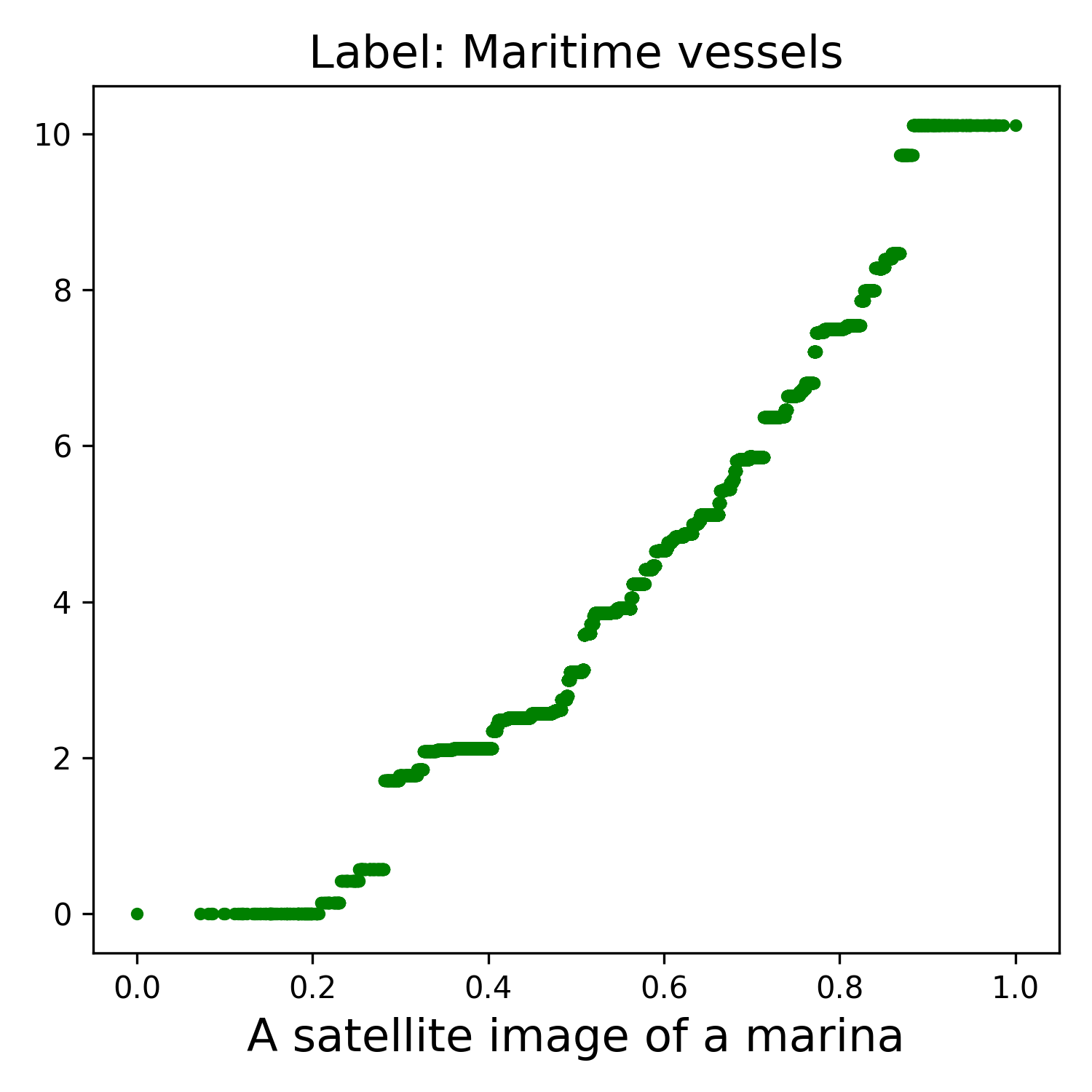} \hspace*{-0.3cm}
        \includegraphics[scale=0.37]{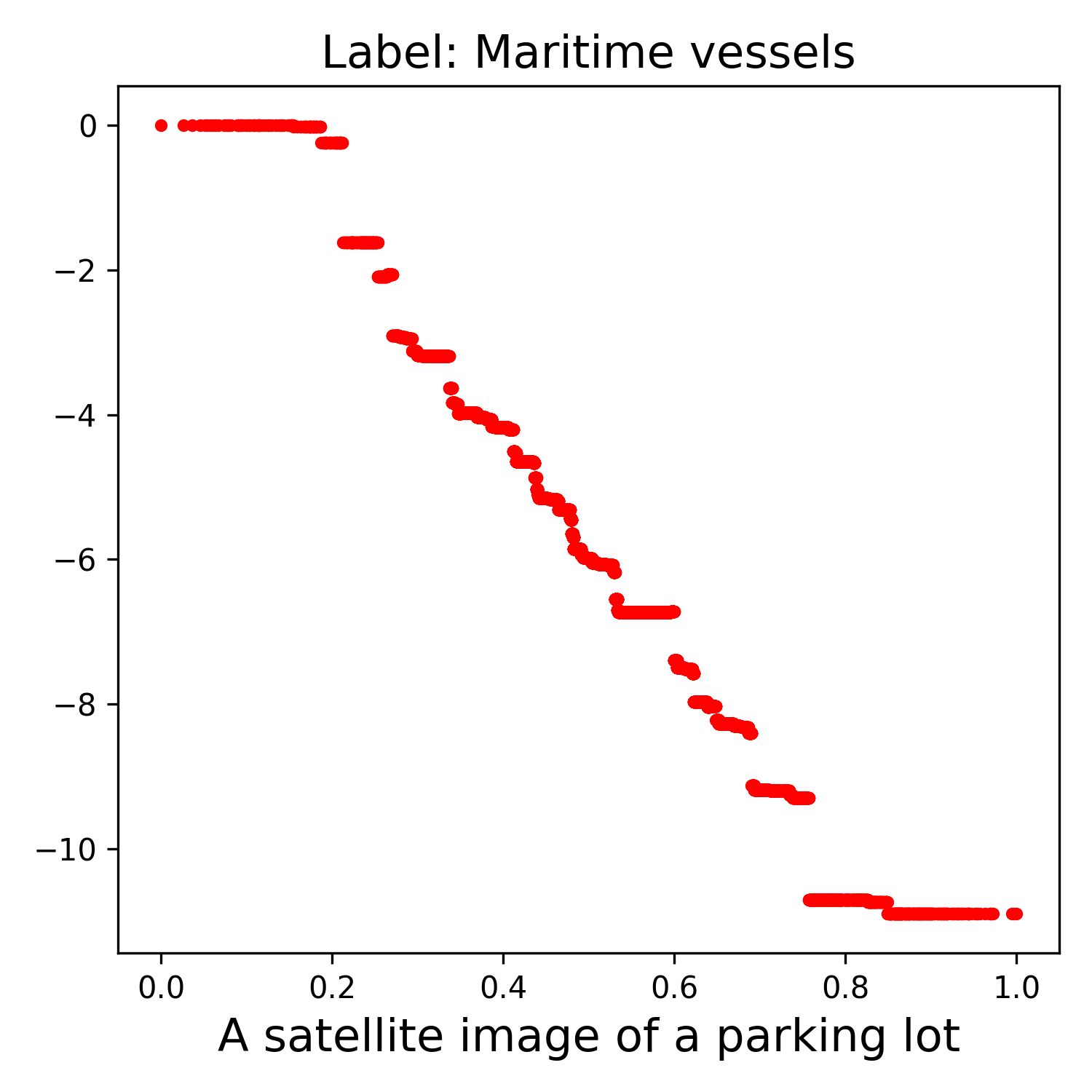}
		\caption{For xView dataset, main effects of concepts in the logit decompositions of ``Maritime Vessels'' in HCBM model (NEC $=5$).} \label{fig:xview_HCBM_4}
	\end{center}
\end{figure}
\begin{figure}
	\begin{center}
		\includegraphics[scale=0.37]{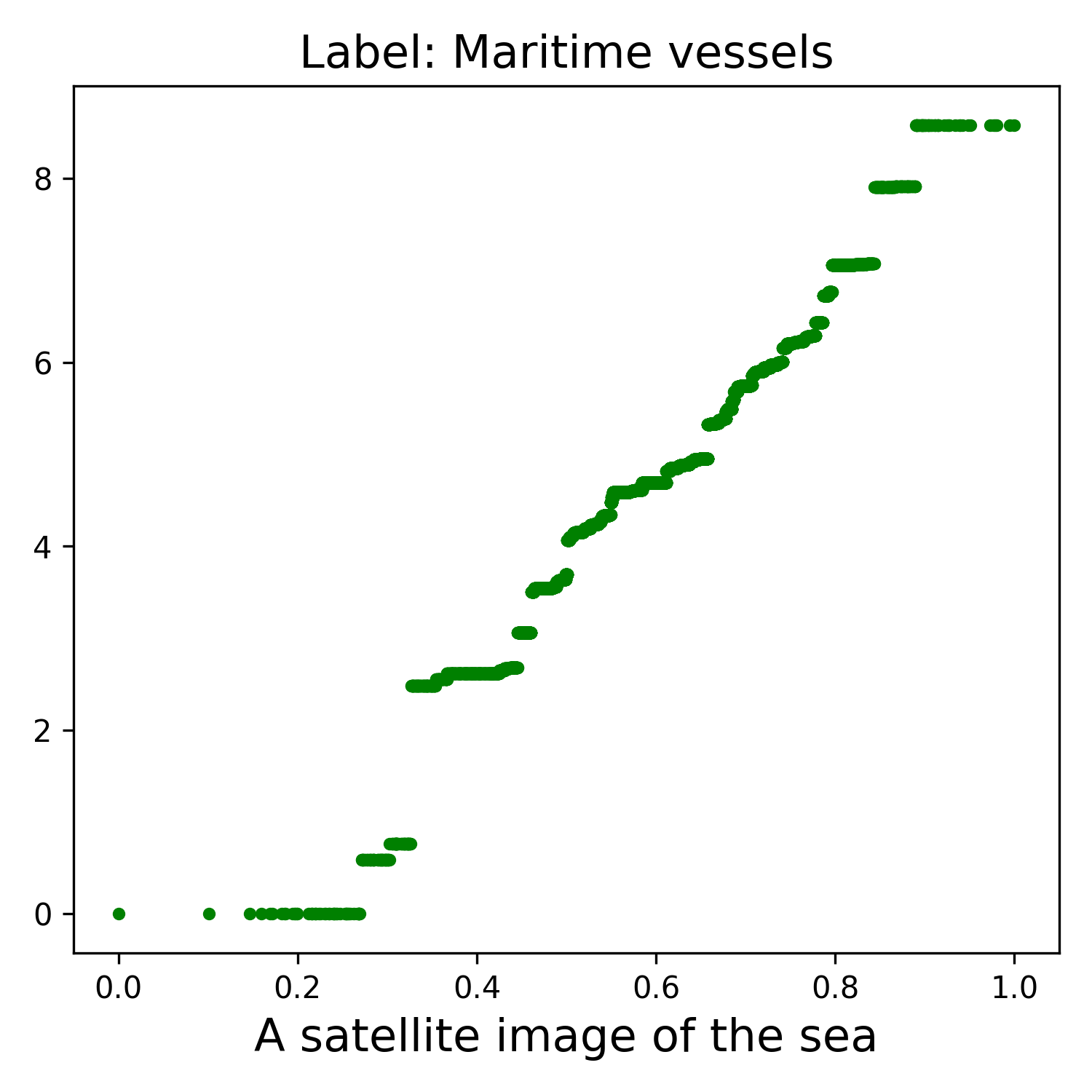} \hspace*{-0.3cm}
        \includegraphics[scale=0.37]{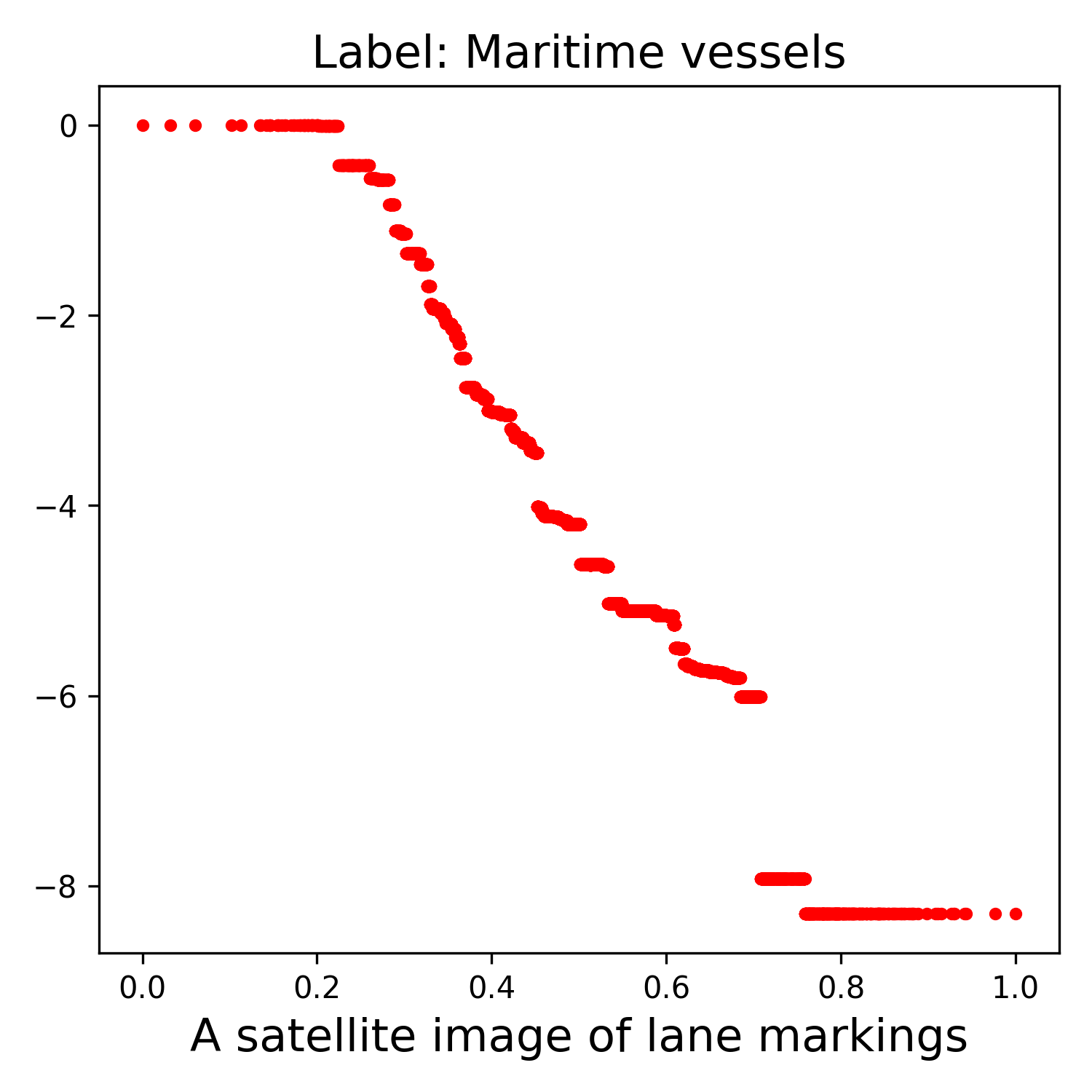} \hspace*{-0.3cm}
        \includegraphics[scale=0.37]{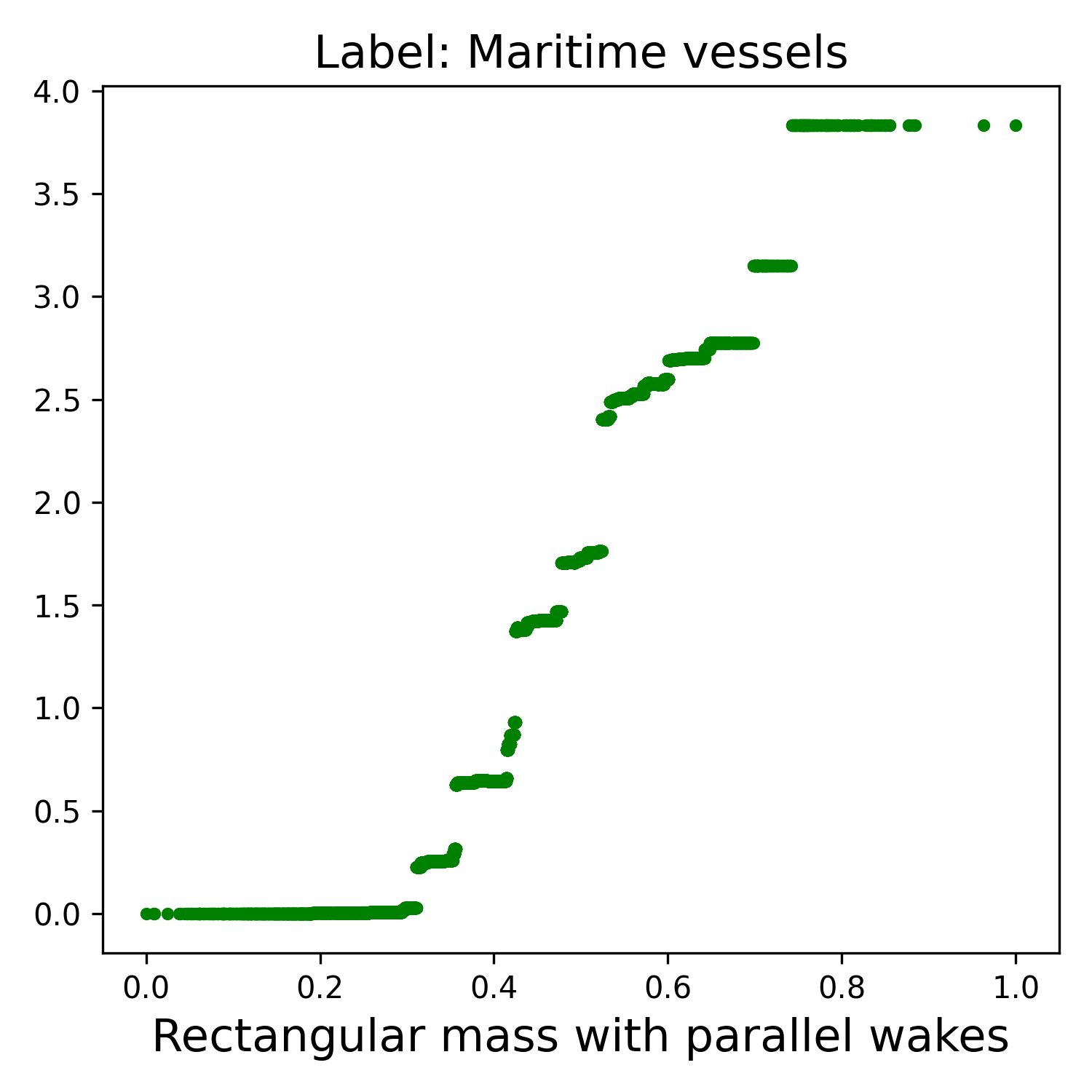}
		\caption{For xView dataset, main effects of concepts in the logit decompositions of ``Maritime Vessels'' in HCBM model (NEC $=5$).} \label{fig:xview_HCBM_5}
	\end{center}
\end{figure}

\subsection{Concepts}

\textbf{Concept list of object attributes:}
\begin{itemize}
    \item Rectangular object
    \item Box-like structure
    \item Elongated shape
    \item Cross-shaped object
    \item Thin elongated body with perpendicular extensions
    \item Bright rectangular shapes contrasting against a gray surface
    \item Long rectangles
    \item An aligned rectangle and square
    \item Several rectangles aligned
    \item Small rectangle with varying color
    \item Black rectangle on lighter background
    \item Rectangle with black and white stripes
    \item Colorful blob with geometric structure
    \item Blurry light spot on a dark background
    \item Shape on dark surface
    \item Streamlined oval shape
    \item Symmetrical structure with a pointed front
    \item White blurry shape on darker background
    \item Cluster of small shapes
    \item White cross-shaped object
    \item Bright-winged silhouette
    \item X-shaped structure with a shadow
    \item Streamlined cross with a pointed tip
    \item Cruciform structure
    \item Small symmetrical figure
    \item Slender main body with perpendicular extensions
    \item Silvery dot
    \item Twin-winged speck on the tarmac
    \item Horizontal bar with a sharp nose
    \item White cross casting a dark outline
    \item Figure with a narrow profile
    \item Black dot surrounded by faint motion trails
    \item Dark central mass with blurred edges
    \item Cross-mark with an elongated tail fin
    \item Small dot with a circular rotor
    \item Compact fuselage with spinning blades
    \item Round central body with extending arcs
    \item Oval-shaped object with a linear extension
    \item Streamlined object leaving a V-shaped wake
    \item Elongated shape cutting through the water
    \item Dark hull with a foamy trail
    \item Narrow shape with a pointed bow
    \item Cigar-shaped silhouette on the waves
    \item White dot trailing a curved water path
    \item Capsule with a tapered end on a dark background
    \item Bright deck with dark water contrast
    \item Oblong object moving across a liquid surface
    \item Rectangular mass with parallel wakes
\end{itemize}

\textbf{Concept list of background elements:}
\begin{itemize}
    \item A satellite image of a runway
    \item A satellite image of a taxiway
    \item A satellite image of taxiway markings
    \item A satellite image of runway markings
    \item A satellite image of an airport
    \item A satellite image of a hangar
    \item A satellite image of a cargo area
    \item A satellite image of a tarmac
    \item A satellite image of a parking lot
    \item A satellite image of airport access roads
    \item A satellite image of a road
    \item A satellite image of urban streets
    \item A satellite image of boulevards
    \item A satellite image of avenues
    \item A satellite image of sidewalks
    \item A satellite image of lane markings
    \item A satellite image of construction zones
    \item A satellite image of motorway exits
    \item A satellite image of maintenance yards
    \item A satellite image of the sea
    \item A satellite image of a harbor
    \item A satellite image of a marina
    \item A satellite image of docks
    \item A satellite image of cargo terminals
    \item A satellite image of a coastal construction zone
\end{itemize}

\end{document}